\documentclass[10pt]{article}

\usepackage[utf8]{inputenc}
\usepackage[T1]{fontenc}

\usepackage{epsf}
\usepackage{amsmath}

\allowdisplaybreaks

\usepackage[showframe=false]{geometry}
\usepackage{changepage}

\usepackage{epsfig}
\usepackage{amssymb}

\usepackage{amsthm}
\usepackage{setspace}
\usepackage{cite}
\usepackage{mcite}

\usepackage{algorithmic}  % This package provides an algorithmic environment fo describing algorithms.
\usepackage{algorithm}

\usepackage{shadow}
\usepackage{fancybox}
\usepackage{fancyhdr}

\usepackage{color}
\usepackage[usenames,dvipsnames,svgnames,table]{xcolor}
\newcommand{\bl}[1]{\textcolor{blue}{#1}}
\newcommand{\red}[1]{\textcolor{red}{#1}}

\definecolor{mypurple}{rgb}{.4,.0,.5}

\usepackage[hyphens]{url}

\usepackage[colorlinks=true,
            linkcolor=black,
            urlcolor=blue,
            citecolor=purple]{hyperref}

\usepackage{breakurl}
\def\y{{\bf y}}

\def\x{{\bf x}}

\def\x{{\mathbf x}}

\def\u{{\bf u}}

\def\x{{\bf x}}
\def\y{{\bf y}}

\def\q{{\bf q}}
\def\m{{\bf m}}

\def\b{{\bf b}}
\def\c{{\bf c}}

\def\h{{\bf h}}

\def\cH{{\mathcal H}}

\def\be{\begin{equation}}
\def\ee{\end{equation}}
\def\ba{\left[\begin{array}}
\def\ea{\end{array}\right]}

\def\u{{\bf u}}

\def\x{{\bf x}}
\def\y{{\bf y}}

\def\q{{\bf q}}

\def\b{{\bf b}}
\def\c{{\bf c}}

\def\p{{\bf p}}

\def\1{{\bf 1}}

\def\0{{\bf 0}}

\def\erfc{\mbox{erfc}}

\def\calX{{\cal X}}
\def\calY{{\cal Y}}

\def\mR{{\mathbb R}}
\def\mN{{\mathbb N}}
\def\mE{{\mathbb E}}
\def\mS{{\mathbb S}}
\def\mP{{\mathbb P}}

\def\lp{\left (}
\def\rp{\right )}

\def\y{{\bf y}}

\def\x{{\bf x}}

\def\x{{\mathbf x}}

\def\u{{\bf u}}

\def\x{{\bf x}}
\def\y{{\bf y}}

\def\q{{\bf q}}

\def\b{{\bf b}}
\def\c{{\bf c}}

\def\h{{\bf h}}

\def\cH{{\cal H}}

\def\be{\begin{equation}}
\def\ee{\end{equation}}
\def\ba{\left[\begin{array}}
\def\ea{\end{array}\right]}

\def\u{{\bf u}}

\def\x{{\bf x}}
\def\y{{\bf y}}

\def\q{{\bf q}}

\def\b{{\bf b}}
\def\c{{\bf c}}

\def\p{{\bf p}}

\def\({\left (}
\def\){\right )}

\def\1{{\bf 1}}
\def\m{{\bf m}}
\def\q{{\bf q}}

\def\0{{\bf 0}}

\def\cX{{\mathcal X}}
\def\cY{{\mathcal Y}}

\usepackage{xcolor}
\usepackage{color}

\definecolor{darkgreen}{rgb}{0, 0.4,0}

\definecolor{purplebrown}{rgb}{0.5,0.1,0.6}

\definecolor{ultclupcol}{rgb}{0.1,0.5,0.5}

\definecolor{mytrycolor}{rgb}{0.5,0.7,0.2}

\definecolor{ultclupcola}{rgb}{.5,0,.5}

\definecolor{shadebrown}{rgb}{0.1,0.1,0.9}
\definecolor{lightblue}{rgb}{0.2,0,1}

\usepackage{fancybox}
\usepackage{graphicx}
\usepackage{epstopdf}
\usepackage{epsfig}
\usepackage{wrapfig}
\usepackage{subfigure}

\usepackage{xcolor}
\usepackage{tcolorbox}
\tcbuselibrary{skins}

\newtcbox{\xmybox}{on line,
arc=7pt,
before upper={\rule[-3pt]{0pt}{10pt}},boxrule=0pt,
boxsep=0pt,left=6pt,right=6pt,top=0pt,bottom=0pt,enhanced, coltext=blue, colback=white!10!yellow}

\newtcbox{\xmyboxa}{on line,
arc=7pt,
before upper={\rule[-3pt]{0pt}{10pt}},boxrule=0pt,
boxsep=0pt,left=6pt,right=6pt,top=0pt,bottom=0pt,enhanced, colback=white!10!yellow}

\newtcbox{\xmyboxb}{on line,
arc=7pt,
before upper={\rule[-3pt]{0pt}{10pt}},boxrule=1pt,colframe=darkgreen!100!blue,
boxsep=0pt,left=6pt,right=6pt,top=0pt,bottom=0pt,enhanced, colback=white!10!yellow}

\newtcbox{\xmyboxc}{on line,
arc=7pt,
before upper={\rule[-3pt]{0pt}{10pt}},boxrule=.7pt,colframe=blue!100!blue,
boxsep=0pt,left=6pt,right=6pt,top=0pt,bottom=0pt,enhanced, coltext=blue, colback=white!10!yellow}

\newtcbox{\xmytboxa}{on line,
arc=7pt,
before upper={\rule[-3pt]{0pt}{10pt}},boxrule=.0pt,colframe=pink!50!yellow,
boxsep=0pt,left=6pt,right=6pt,top=0pt,bottom=0pt,enhanced, coltext=white, colback=blue!40!red}

\newtcbox{\xmytboxb}{on line,
arc=7pt,
before upper={\rule[-3pt]{0pt}{10pt}},boxrule=.0pt,colframe=pink!50!yellow,
boxsep=0pt,left=6pt,right=6pt,top=0pt,bottom=0pt,enhanced, coltext=white, colback=white!40!green}

\makeatletter
\newcommand\subsubsubsection{\@startsection{paragraph}{4}{\z@}{-2.5ex\@plus -1ex \@minus -.25ex}{1.25ex \@plus .25ex}{\normalfont\normalsize\bfseries}}
\newcommand\subsubsubsubsection{\@startsection{subparagraph}{5}{\z@}{-2.5ex\@plus -1ex \@minus -.25ex}{1.25ex \@plus .25ex}{\normalfont\normalsize\bfseries}}
\makeatother

\newtheorem{theorem}{Theorem}

\newtheorem{corollary}{Corollary}

\begin{document}

\begin{singlespace}

\title {Fl RDT based ultimate lowering of the negative spherical perceptron capacity %A tight variant of Gordon's escape through a mesh theorem
%\footnote{ This work was supported in
%part.}
}
\author{
\textsc{Mihailo Stojnic
\footnote{e-mail: {\tt flatoyer@gmail.com}} }}
\date{}
\maketitle

%%%%%%%%%%%%%%%%%%%%%%%%%%%%%%%%%%%%%%%%%%%%%%%%%%%%%%%%%%%%%%%%%%%%%%%%%%%%%%%%
%%%%%%%%%%%%%%%%%%%%%%%%%%%%%%%%%%%%%%%%%%%%%%%%%%%%%%%%%%%%%%%%%%%%%%%%%%%%%%%%
\centerline{{\bf Abstract}} \vspace*{0.1in}
%%%%%%%%%%%%%%%%%%%%%%%%%%%%%%%%%%%%%%%%%%%%%%%%%%%%%%%%%%%%%%%%%%%%%%%%%%%%%%%%
%%%%%%%%%%%%%%%%%%%%%%%%%%%%%%%%%%%%%%%%%%%%%%%%%%%%%%%%%%%%%%%%%%%%%%%%%%%%%%%%

We consider the classical \emph{spherical} perceptrons and study their capacities. The famous zero-threshold case was solved in the sixties of the last century (see, \cite{Wendel62,Winder,Cover65}) through the high-dimensional combinatorial considerations. The general threshold, $\kappa$, case though turned out to be much harder and stayed out of reach for the following several decades. A substantial progress was then  made in \cite{SchTir02} and \cite{StojnicGardGen13} where the \emph{positive} threshold ($\kappa\geq 0$) scenario was finally fully settled. While the negative counterpart ($\kappa\leq 0$) remained out of reach, \cite{StojnicGardGen13} did show that the random duality theory (RDT) is still powerful enough to provide excellent upper bounds. Moreover, in \cite{StojnicGardSphNeg13}, a \emph{partially lifted} RDT variant  was considered and it was shown that the upper bounds of \cite{StojnicGardGen13} can be lowered. After recent breakthroughs in studying bilinearly indexed (bli) random processes in \cite{Stojnicsflgscompyx23,Stojnicnflgscompyx23}, \emph{fully lifted} random duality theory (fl RDT) was developed in \cite{Stojnicflrdt23}. We here first show that the \emph{negative spherical perceptrons} can be fitted into the frame of the fl RDT and then employ the whole fl RDT machinery to characterize the capacity. To be fully practically operational, the fl RDT requires a substantial numerical work. We, however, uncover remarkable closed form analytical relations among key lifting parameters. Such a discovery enables both shedding a new light on the parametric interconnections within the lifting structure and performing the needed numerical calculations to obtain concrete capacity values. After doing all of that, we also observe that an excellent convergence (with the relative improvement $\sim 0.1\%$) is achieved already on the third (second non-trivial) level of the \emph{stationarized} full lifting.

\vspace*{0.25in} \noindent {\bf Index Terms: Negative spherical perceptrons; Fully lifted random duality theory}.

\end{singlespace}

%%%%%%%%%%%%%%%%%%%%%%%%%%%%%%%%%%%%%%%%%%%%%%%%%%%%%%%%%%%%%%%%%
%%%%%%%%%%%%%%%%%%%%%%%%%%%%%%%%%%%%%%%%%%%%%%%%%%%%%%%%%%%%%%%%%
%%%%%%%%%%%%%%%%%%%%%%%%%%%%%%%%%%%%%%%%%%%%%%%%%%%%%%%%%%%%%%%%%
%%%%%%%%%%%%%%%%%%%%%%%%%%%%%%%%%%%%%%%%%%%%%%%%%%%%%%%%%%%%%%%%%
%%%%%%%%%%%%%%%%%%%%%%%%%%%%%%%%%%%%%%%%%%%%%%%%%%%%%%%%%%%%%%%%%
%%%%%%%%%%%%%%%%%%%%%%%%%%%%%%%%%%%%%%%%%%%%%%%%%%%%%%%%%%%%%%%%%
\section{Introduction}
\label{sec:back}
%%%%%%%%%%%%%%%%%%%%%%%%%%%%%%%%%%%%%%%%%%%%%%%%%%%%%%%%%%%%%%%%%
%%%%%%%%%%%%%%%%%%%%%%%%%%%%%%%%%%%%%%%%%%%%%%%%%%%%%%%%%%%%%%%%%
%%%%%%%%%%%%%%%%%%%%%%%%%%%%%%%%%%%%%%%%%%%%%%%%%%%%%%%%%%%%%%%%%
%%%%%%%%%%%%%%%%%%%%%%%%%%%%%%%%%%%%%%%%%%%%%%%%%%%%%%%%%%%%%%%%%
%%%%%%%%%%%%%%%%%%%%%%%%%%%%%%%%%%%%%%%%%%%%%%%%%%%%%%%%%%%%%%%%%

The last two decades have seen a remarkable progress in studying various aspects of neural networks (NN) and machine learning (ML). Development of powerful algorithmic techniques and corresponding performance characterizing  analytical tools  together with persistent widening of the range of potential applications are only a couple of the most important ones. We, here, follow into similar footsteps and continue the analytical progress through a theoretical studying of  \emph{perceptrons} as key NN/ML building blocks.

We are particularly interested in the so-called \emph{spherical} perceptrons which are easily the most popular and quite likely the simplest of all perceptron variants. Despite the simplicity, their full analytical characterizations in many important scenarios are not easy to obtain. For example, one of their most relevant features, the storage or classifying \emph{capacity}, is, in general, very difficult to compute. Moreover, designing practical algorithms that can confirm the capacity achievability is often even harder. Some special cases are a bit easier though and relevant results can be found throughout the literature. For example, the so-called \emph{zero-threshold} capacity was determined through a combinatorial, high-dimensional geometry based, approach in seminal works \cite{Wendel62,Winder,Cover65} (for relevant geometric followup extensions related to  polytopal  neighborliness see, e.g., \cite{Stojnicl1BnBxasymldp,Stojnicl1BnBxfinn}).

While the results of \cite{Wendel62,Winder,Cover65} established a monumental breakthrough at the time of their appearance, they remained an isolated example of extraordinary success for the better part of the following several decades. For example, even the simplest possible extension to general positive thresholds turned out to be a formidable challenge. As it became apparent that the mathematically rigorous treatments might be a bit further away than initially predicated, the emergence of the statistical physics \emph{replica tools} in the seventies of the last century provided a glimmer of hope that at least some (not necessarily mathematically rigorous) analytical characterizations can be obtained. Not long after, in the second half of the eighties, the Gardner's seminal work, \cite{Gar88} appeared and paved the way for many of the very best perceptrons' analytical results. Namely, \cite{Gar88} and a follow-up \cite{GarDer88}, utilized the replica theory and established a generic framework that can be used for the analytical characterizations of, basically, all relevant features of interest in  various perceptrons models. Among others, these certainly included the storage capacities in a host of different scenarios: positive/negative thresholds, correlated/uncorrelated patterns, patterns stored incorrectly and many others. The predictions obtained in \cite{Gar88,GarDer88} were later on (in identical or similar statistical contexts) established as mathematically fully rigorous (see, e.g., \cite{SchTir02,SchTir03,Tal05,Talbook11a,Talbook11b,StojnicGardGen13,StojnicGardSphNeg13,StojnicGardSphErr13}). In particular, \cite{SchTir02,SchTir03} proved the predictions of \cite{Gar88} related to the storage capacity and the volume of the bond strengths that satisfies the dynamics of the \emph{positive} spherical perceptrons (i.e., the perceptrons with spherical constraints and positive thresholds $\kappa\geq 0$). Talagrand, in \cite{Tal05,Talbook11a,Talbook11b}, reconfirmed these predictions through a related but somewhat different approach. On the other hand, \cite{StojnicGardGen13} designed a completely different, random duality theory (RDT) based,  framework  and again confirmed almost all of the predictions from \cite{Gar88}, including many previously not considered in \cite{SchTir02,SchTir03,Tal05,Talbook11a,Talbook11b}.  A substantial help in all of these, mathematically rigorous, treatments, was provided by the underlying \emph{convexity}.

%%%%%%%%%%%%%%%%%%%%%%%%%%%%%%%%%%%%%%%%%%%%%%%%%%%%%%%%%%%%%%%%%%%%%%%%%%%%%%%%%%%%%%%%%%%%%%%%%%%%%%%%%%%%%
%%%%%%%%%%%%%%%%%%%%%%%%%%%%%%%%%%%%%%%%%%%%%%%%%%%%%%%%%%%%%%%%%%%%%%%%%%%%%%%%%%%%%%%%%%%%%%%%%%%%%%%%%%%%%
%%%%%%%%%%%%%%%%%%%%%%%%%%%%%%%%%%%%%%%%%%%%%%%%%%%%%%%%%%%%%%%%%%%%%%%%%%%%%%%%%%%%%%%%%%%%%%%%%%%%%%%%%%%%%
%%%%%%%%%%%%%%%%%%%%%%%%%%%%%%%%%%%%%%%%%%%%%%%%%%%%%%%%%%%%%%%%%%%%%%%%%%%%%%%%%%%%%%%%%%%%%%%%%%%%%%%%%%%%%
%%%%%%%%%%%%%%%%%%%%%%%%%%%%%%%%%%%%%%%%%%%%%%%%%%%%%%%%%%%%%%%%%%%%%%%%%%%%%%%%%%%%%%%%%%%%%%%%%%%%%%%%%%%%%
\subsection{\emph{Negative} spherical perceptron (NSP) --- no convexity help}
\label{sec:devsddconv}
%%%%%%%%%%%%%%%%%%%%%%%%%%%%%%%%%%%%%%%%%%%%%%%%%%%%%%%%%%%%%%%%%%%%%%%%%%%%%%%%%%%%%%%%%%%%%%%%%%%%%%%%%%%%%
%%%%%%%%%%%%%%%%%%%%%%%%%%%%%%%%%%%%%%%%%%%%%%%%%%%%%%%%%%%%%%%%%%%%%%%%%%%%%%%%%%%%%%%%%%%%%%%%%%%%%%%%%%%%%
%%%%%%%%%%%%%%%%%%%%%%%%%%%%%%%%%%%%%%%%%%%%%%%%%%%%%%%%%%%%%%%%%%%%%%%%%%%%%%%%%%%%%%%%%%%%%%%%%%%%%%%%%%%%%
%%%%%%%%%%%%%%%%%%%%%%%%%%%%%%%%%%%%%%%%%%%%%%%%%%%%%%%%%%%%%%%%%%%%%%%%%%%%%%%%%%%%%%%%%%%%%%%%%%%%%%%%%%%%%
%%%%%%%%%%%%%%%%%%%%%%%%%%%%%%%%%%%%%%%%%%%%%%%%%%%%%%%%%%%%%%%%%%%%%%%%%%%%%%%%%%%%%%%%%%%%%%%%%%%%%%%%%%%%%

As recognized in  \cite{StojnicGardGen13,StojnicGardSphNeg13}, the above mentioned convexity help disappears when the spherical perceptrons have a negative threshold (i.e., when $\kappa<0$). The underlying deterministic strong duality is not present anymore and obtaining accurate capacity characterizations becomes notoriously hard. Still, the power of the RDT remains useful. In particular, relying on the fundamental principles of the RDT,  \cite{StojnicGardSphNeg13} proved the Talagrand's conjecture from \cite{Tal05,Talbook11a,Talbook11b} that the capacity predictions of \cite{Gar88} are, at the very least, rigorous upper bounds even when $\kappa<0$. \cite{StojnicGardSphNeg13} went a step further, utilized a \emph{partially lifted} RDT variant and established that, these rigorous bounds can in fact be lowered. This effectively confirmed that the replica symmetry (assumed in \cite{Gar88,GarDer88}) must be broken.

 A series of works based on statistical physics replica approaches then followed (see, e.g., \cite{FPSUZ17,FraPar16,FraSclUrb19,FraSclUrb20,AlaSel20}). \cite{FraPar16} was the first one where the NSP was connected to the recently studied jamming phenomena and hard spheres packing problems. It established a preliminary version of the phase diagram and emphasized the relevance of the distribution laws of ``\emph{gaps}'' and ``\emph{forces}'' and computed their critical exponents. \cite{FPSUZ17} then provided a more complete phase diagram characterization with all predicated types of transitioning in both the so-called SAT and UNSAT phases. Moreover, it hypothesized a potential universality in gaps and forces distribution laws exponents.  \cite{FraSclUrb19} studied similar features in the linear cost NSP variant. Again, the critical exponents of the distribution laws were found to match the ones associated with the jamming of the hard spheres. The corresponding algorithmic confirmations were obtained in  \cite{FraSclUrb20}. Algorithmic considerations of a different type were discussed in \cite{AlaSel20}. Relying on the (access to the) Parisi replica symmetry breaking (rsb) variational functional, an iterative message-passing type of procedure is suggested as an algorithmic way of achieving the capacity.

On the rigorous front though, the results of \cite{StojnicGardSphNeg13} remained untouchable until now. Moreover, \cite{ZhouXiMon21} showed that the upper bounds of \cite{StojnicGardSphNeg13} are actually (up to the leading order terms) tight in $\kappa\rightarrow -\infty$ regime. As mentioned above, \cite{Gar88} also studied many other perceptron properties. It, for example, gave the replica symmetry prediction for the capacities of spherical perceptrons when functioning as erroneous storage memories.  \cite{StojnicGardSphErr13} showed that these predictions of \cite{Gar88} are again rigorous upper bounds  which in certain range of system parameters can be lowered. This proved that, in the erroneous scenarios, the replica symmetry (assumed in \cite{Gar88,GarDer88}) must again be broken.

While our primary interest here is in the simplest and possibly most famous spherical perceptrons, various other perceptron variants are of interest. Moreover, many of them that belong to the class of analytically ``\emph{hard}'' perceptrons have been intensively studied over the last several decades as well. We here single out probably the most well known  \emph{discrete} $\pm 1$ perceptrons. Their \emph{symmetric} realizations are analytically  a bit easier than other variants and the corresponding full  capacity characterizations are known to have very particular relatively simple formulations (see, e.g., \cite{AbbLiSly21b,PerkXu21} as well as \cite{AbbLiSly21a,AlwLiuSaw21,AubPerZde19,GamKizPerXu22}). On the other hand, an initial replica symmetry based treatment of the original \emph{nonsymmetric} ones was already given in the foundational papers  \cite{Gar88,GarDer88}, where the underlying hardness is properly recognized. After the rsb based results were obtained in \cite{KraMez89}, a strong mathematical progress followed first in \cite{DingSun19,NakSun23,BoltNakSunXu22,Tal99a,StojnicDiscPercp13} and then ultimately in \cite{Stojnicbinperflrdt23} as well.

We here follow the path of \cite{Stojnicbinperflrdt23} and utilize the connection between the so-called \emph{random feasibility problems} (rfps) (and the spherical perceptrons as their particular instances) on the one side and the \emph{random duality theory} (RDT) (see, e.g., \cite{StojnicCSetam09,StojnicICASSP10var,StojnicRegRndDlt10,StojnicGardGen13,StojnicICASSP09}) concepts on the other side. We first recognize the connection between the rfps and\emph{bilinearly indexed} (bli) random processes and then utilize a strong recent progress in studying these processes in \cite{Stojnicsflgscompyx23,Stojnicnflgscompyx23}. Namely, relying on \cite{Stojnicsflgscompyx23,Stojnicnflgscompyx23}, in \cite{Stojnicflrdt23} a \emph{fully lifted} random duality theory (fl RDT) was established. Utilizing further the fl RDT and its a particular \emph{stationarized} fl RDT variant (called sfl RDT), we then obtain desired capacity characterizations. As is usually the case, to have the fl RDT become practically operational, underlying numerical evaluations need to be conducted. Doing so is a problem on its own and often requires a rather strong effort. Here, however, we discover remarkable closed form relations between key lifting parameters. These provide a direct view into a rather beautiful structuring of the intrinsic parametric interconnections and ultimately substantially facilitate the underlying numerical work. Moreover, they eventually enable us to uncover that the obtained capacity characterizations, already on the \emph{\textbf{third}} level of the \emph{full} lifting (3-sfl RDT), exhibit an extraordinarily rapid convergence with a relative improvement  $\sim 0.1\%$ for all considered thresholds $\kappa$.

%%%%%%%%%%%%%%%%%%%%%%%%%%%%%%%%%%%%%%%%%%%%%%%%%%%%%%%%%%%%%%%%%%%%%%%%%%%%%%%%%%%%%%%%%%%%%%%%%%%%%%%%%%%%%%%%%%%%%%%
%%%%%%%%%%%%%%%%%%%%%%%%%%%%%%%%%%%%%%%%%%%%%%%%%%%%%%%%%%%%%%%%%%%%%%%%%%%%%%%%%%%%%%%%%%%%%%%%%%%%%%%%%%%%%%%%%%%%%%%
%%%%%%%%%%%%%%%%%%%%%%%%%%%%%%%%%%%%%%%%%%%%%%%%%%%%%%%%%%%%%%%%%%%%%%%%%%%%%%%%%%%%%%%%%%%%%%%%%%%%%%%%%%%%%%%%%%%%%%%
%%%%%%%%%%%%%%%%%%%%%%%%%%%%%%%%%%%%%%%%%%%%%%%%%%%%%%%%%%%%%%%%%%%%%%%%%%%%%%%%%%%%%%%%%%%%%%%%%%%%%%%%%%%%%%%%%%%%%%%
%%%%%%%%%%%%%%%%%%%%%%%%%%%%%%%%%%%%%%%%%%%%%%%%%%%%%%%%%%%%%%%%%%%%%%%%%%%%%%%%%%%%%%%%%%%%%%%%%%%%%%%%%%%%%%%%%%%%%%%
%%%%%%%%%%%%%%%%%%%%%%%%%%%%%%%%%%%%%%%%%%%%%%%%%%%%%%%%%%%%%%%%%%%%%%%%%%%%%%%%%%%%%%%%%%%%%%%%%%%%%%%%%%%%%%%%%%%%%%%
\section{Connecting NSPs to rfps and free energies}
 \label{sec:bprfps}
%%%%%%%%%%%%%%%%%%%%%%%%%%%%%%%%%%%%%%%%%%%%%%%%%%%%%%%%%%%%%%%%%%%%%%%%%%%%%%%%%%%%%%%%%%%%%%%%%%%%%%%%%%%%%%%%%%%%%%%
%%%%%%%%%%%%%%%%%%%%%%%%%%%%%%%%%%%%%%%%%%%%%%%%%%%%%%%%%%%%%%%%%%%%%%%%%%%%%%%%%%%%%%%%%%%%%%%%%%%%%%%%%%%%%%%%%%%%%%%
%%%%%%%%%%%%%%%%%%%%%%%%%%%%%%%%%%%%%%%%%%%%%%%%%%%%%%%%%%%%%%%%%%%%%%%%%%%%%%%%%%%%%%%%%%%%%%%%%%%%%%%%%%%%%%%%%%%%%%%
%%%%%%%%%%%%%%%%%%%%%%%%%%%%%%%%%%%%%%%%%%%%%%%%%%%%%%%%%%%%%%%%%%%%%%%%%%%%%%%%%%%%%%%%%%%%%%%%%%%%%%%%%%%%%%%%%%%%%%%
%%%%%%%%%%%%%%%%%%%%%%%%%%%%%%%%%%%%%%%%%%%%%%%%%%%%%%%%%%%%%%%%%%%%%%%%%%%%%%%%%%%%%%%%%%%%%%%%%%%%%%%%%%%%%%%%%%%%%%%
%%%%%%%%%%%%%%%%%%%%%%%%%%%%%%%%%%%%%%%%%%%%%%%%%%%%%%%%%%%%%%%%%%%%%%%%%%%%%%%%%%%%%%%%%%%%%%%%%%%%%%%%%%%%%%%%%%%%%%%

As suggested above, we will rely on the fact that studying the NSP properties is tightly connected to studying the properties of feasibility problems. Moreover, studying feasibility problems is then tightly connected to studying statistical physics objects called \emph{free energies}. Both of these connections were recognized and utilized in a long line of work \cite{StojnicGardGen13,StojnicGardSphErr13,StojnicGardSphNeg13,StojnicDiscPercp13,Stojnicbinperflrdt23}. To capitalize on the existing results and to make the exposition of the main ideas needed here as smooth as possible, we find it convenient to carefully parallel the presentations from these papers. Along the same lines, to avoid an unnecessary repetition of the already introduced concepts, we adopt the practice to briefly recall on them and then focus on highlighting the main differences, novelties, and other particularities related to the problems of out interest here.

%%%%%%%%%%%%%%%%%%%%%%%%%%%%%%%%%%%%%%%%%%%%%%%%%%%%%%%%%%%%%%%%%%%%%%%%%%%%%%%%%%%%%%%%%%%%%%%%%%%%%%%%%%%%%%%%%%%%%%%
%%%%%%%%%%%%%%%%%%%%%%%%%%%%%%%%%%%%%%%%%%%%%%%%%%%%%%%%%%%%%%%%%%%%%%%%%%%%%%%%%%%%%%%%%%%%%%%%%%%%%%%%%%%%%%%%%%%%%%%
%%%%%%%%%%%%%%%%%%%%%%%%%%%%%%%%%%%%%%%%%%%%%%%%%%%%%%%%%%%%%%%%%%%%%%%%%%%%%%%%%%%%%%%%%%%%%%%%%%%%%%%%%%%%%%%%%%%%%%%
%%%%%%%%%%%%%%%%%%%%%%%%%%%%%%%%%%%%%%%%%%%%%%%%%%%%%%%%%%%%%%%%%%%%%%%%%%%%%%%%%%%%%%%%%%%%%%%%%%%%%%%%%%%%%%%%%%%%%%%
%%%%%%%%%%%%%%%%%%%%%%%%%%%%%%%%%%%%%%%%%%%%%%%%%%%%%%%%%%%%%%%%%%%%%%%%%%%%%%%%%%%%%%%%%%%%%%%%%%%%%%%%%%%%%%%%%%%%%%%
%%%%%%%%%%%%%%%%%%%%%%%%%%%%%%%%%%%%%%%%%%%%%%%%%%%%%%%%%%%%%%%%%%%%%%%%%%%%%%%%%%%%%%%%%%%%%%%%%%%%%%%%%%%%%%%%%%%%%%%
\subsection{NSP $\longleftrightarrow$ rfps connection}
 \label{sec:bprfps}
%%%%%%%%%%%%%%%%%%%%%%%%%%%%%%%%%%%%%%%%%%%%%%%%%%%%%%%%%%%%%%%%%%%%%%%%%%%%%%%%%%%%%%%%%%%%%%%%%%%%%%%%%%%%%%%%%%%%%%%
%%%%%%%%%%%%%%%%%%%%%%%%%%%%%%%%%%%%%%%%%%%%%%%%%%%%%%%%%%%%%%%%%%%%%%%%%%%%%%%%%%%%%%%%%%%%%%%%%%%%%%%%%%%%%%%%%%%%%%%
%%%%%%%%%%%%%%%%%%%%%%%%%%%%%%%%%%%%%%%%%%%%%%%%%%%%%%%%%%%%%%%%%%%%%%%%%%%%%%%%%%%%%%%%%%%%%%%%%%%%%%%%%%%%%%%%%%%%%%%
%%%%%%%%%%%%%%%%%%%%%%%%%%%%%%%%%%%%%%%%%%%%%%%%%%%%%%%%%%%%%%%%%%%%%%%%%%%%%%%%%%%%%%%%%%%%%%%%%%%%%%%%%%%%%%%%%%%%%%%
%%%%%%%%%%%%%%%%%%%%%%%%%%%%%%%%%%%%%%%%%%%%%%%%%%%%%%%%%%%%%%%%%%%%%%%%%%%%%%%%%%%%%%%%%%%%%%%%%%%%%%%%%%%%%%%%%%%%%%%
%%%%%%%%%%%%%%%%%%%%%%%%%%%%%%%%%%%%%%%%%%%%%%%%%%%%%%%%%%%%%%%%%%%%%%%%%%%%%%%%%%%%%%%%%%%%%%%%%%%%%%%%%%%%%%%%%%%%%%%

As is well known, the \emph{feasibility} problems with linear inequalities have the following mathematical form
\begin{eqnarray}
\hspace{-1.5in}\mbox{Feasibility problem $\mathbf{\mathcal F}(G,\b,\cX,\alpha)$:} \hspace{1in}\mbox{find} & & \x\nonumber \\
\mbox{subject to}
& & G\x\geq \b \nonumber \\
& & \x\in\cX. \label{eq:ex1}
\end{eqnarray}
In (\ref{eq:ex1}), $G\in\mR^{n\times n}$, $\b\in\mR^{m\times 1}$, $\cX\in\mR^n$, and $\alpha=\frac{m}{n}$.  \cite{StojnicGardGen13,StojnicGardSphErr13,StojnicGardSphNeg13,StojnicDiscPercp13,Stojnicbinperflrdt23} recognized that the above formulation is directly related to both the main principles of the random duality theory (RDT) and the mathematical description of perceptrons. The perceptron's types, however, can be different and are determined based on matrix $G$, vector $\b$, and set $\cX$. For example, for
$\cX=\{-\frac{1}{\sqrt{n}},\frac{1}{\sqrt{n}} \}^n$ (i.e., for $\cX$ being the corners of the $n$-dimensional unit norm hypercube), one has the so-called binary $\pm 1$ perceptrons, whereas for $\cX=\mS^n$ (i.e., for $\cX$ being the $n$-dimensional unit sphere $\mS^n$), one has the so-called spherical perceptrons. Both of these perceptron variants allow for generic (variable) values of the components of the threshold vector $\b$. When $\b$ is a multiple of $\1$ (column vector of all ones of appropriate dimension), i.e., when $\b=\kappa\1$ (where $\kappa\in\mR$), one further obtains perceptrons with \emph{fixed} thresholds $\kappa$. In particular, for $\cX=\mS^n$ and $\kappa<0$ one obtains that (\ref{eq:ex1}) effectively emulates the so-called \emph{negative spherical} perceptron (NSP). Moreover, if $G$ is generic and deterministic, we have a deterministic perceptron. Correspondingly, if $G$ is random, we have a statistical one. Our main objects of interest in this paper are the random NSPs and, in particular, the Gaussian NSPs, where the components of $G$ are iid standard normal random variables. Taking $G$ to be comprised of the iid standard normal components makes the presentation neater. However, all the key results that we obtain are adaptable so that they relate to other random NSP variants  where the randomness can come from basically any other distribution that can be pushed through the Lindenberg variant of the central limit theorem.

It is easy to see that the feasibility problem from (\ref{eq:ex1}) can be rewritten as the following optimization problem
\begin{eqnarray}
\min_{\x} & & f(\x) \nonumber \\
\mbox{subject to}
& & G\x\geq \b \nonumber \\
& & \x\in\cX, \label{eq:ex1a1}
\end{eqnarray}
where an artificial function $f(\x):\mR^n \rightarrow\mR$ is introduced. As is also well known, for any optimization problem to be solvable, the necessary precondition is that it is actually feasible. Assuming the feasibility, (\ref{eq:ex1a1})  can then be rewritten as
\begin{eqnarray}
\xi_{feas}^{(0)}(f,\cX) = \min_{\x\in\cX} \max_{\y\in\cY_+}  \lp f(\x) -\y^T G\x +\y^T  \b  \rp,
 \label{eq:ex2}
\end{eqnarray}
where $\cY_+$ is basically a set that collects all $\y$ such that $\y_{i}\geq 0,1\leq i\leq m$. Since $f(\x)=0$ is clearly an artificial object, one can also specialize back to $f(\x)=0$ and find
\begin{eqnarray}
\xi_{feas}^{(0)}(0,\cX) = \min_{\x\in\cX} \max_{\y\in\cY_+}  \lp -\y^T G\x +\y^T  \b  \rp.
 \label{eq:ex2a1}
\end{eqnarray}
The main point behind perceptron's functioning and its connection to rfps is contained precisely in (\ref{eq:ex2a1}). To see this, one starts by observing that existence of an $\x$ such that $G\x\geq \b$, i.e., such that (\ref{eq:ex1}) is feasible, ensures that the inner maximization in (\ref{eq:ex2a1}) can do no better than make $\xi_{feas}^{(0)}(0,\cX) =0$. On the other hand, if such an $\x$ does not exist, then at least one of the inequalities in $G\x\geq \b$ is not satisfied and the inner maximization trivially makes $\xi_{feas}^{(0)}(0,\cX) =\infty$. It is also easy to see that, from the feasibility point of view, $\xi_{feas}^{(0)}(0,\cX) =\infty$ and $\xi_{feas}^{(0)}(0,\cX) >0$ are equivalent which implies that, for all practical feasibility purposes, the underlying optimization problem in (\ref{eq:ex2a1}) is structurally insensitive with respect to $\y$ scaling. One can then restrict to $\|\y\|_2=1$  and basically ensure that $\xi_{feas}^{(0)}(0,\cX)$ remains bounded. It is then straightforward to see from (\ref{eq:ex2a1}), that determining
\begin{eqnarray}
\xi_{feas}(0,\cX)
& =  &
\min_{\x\in\cX} \max_{\y\in\cY_+,\|\y\|_2=1}   \lp -\y^TG\x +\y^T\b \rp
                   =
\min_{\x\in\mS^n} \max_{\y\in\mS_+^m}  \lp -\y^TG\x + \kappa \y^T\1 \rp,
 \label{eq:ex3}
\end{eqnarray}
with $\mS_+^m$ being the positive orthant part of the $m$-dimensional unit sphere,
is critically important for the analytical characterization of the rfps from (\ref{eq:ex1}). One then has that the sign of the objective value in (\ref{eq:ex3}) (i.e., of $\xi_{feas}(f,\cX)$) determines the feasibility of (\ref{eq:ex1}). In more concrete terms, (\ref{eq:ex1}) is infeasible if  $\xi_{feas}(f,\cX)>0$ and feasible if   $\xi_{feas}(f,\cX)\leq 0$.

The above reasoning holds generically, i.e., for any $G$ and $\b$. It then automatically applies to the Gaussian NSPs as particular instances of the above formalism obtained for Gaussian $G$ and $\b=\kappa \1,\kappa<0$. Given that the connection between the rfps from (\ref{eq:ex1}) and the corresponding random \emph{optimization} problem counterpart from  (\ref{eq:ex3}) is rather evident, one clearly observes the critically important role of (\ref{eq:ex3}) in characterizing various perceptrons' features. The feature of our particular interest here is the storage/classifying  \emph{capacity}. In a large dimensional statistical context, it is defined as follows
 \begin{eqnarray}
\alpha & = &    \lim_{n\rightarrow \infty} \frac{m}{n}  \nonumber \\
\alpha_c(\kappa) & \triangleq & \max \{\alpha |\hspace{.08in}  \lim_{n\rightarrow\infty}\mP_G\lp\xi_{perc}(0,\cX)\triangleq \xi_{feas}(0,\cX)>0\rp\longrightarrow 1\} \nonumber \\
& = & \max \{\alpha |\hspace{.08in}  \lim_{n\rightarrow\infty}\mP_G\lp{\mathcal F}(G,\b,\cX,\alpha) \hspace{.07in}\mbox{is feasible} \rp\longrightarrow 1\}.
  \label{eq:ex4}
\end{eqnarray}
The above is the so-called \emph{statistical} capacity. The corresponding deterministic  variant is defined in exactly the same way with $\mP_G$ being removed. Throughout the paper, the subscripts next to $\mP$ and $\mE$ denote the randomness with respect to which the statistical evaluation is taken. On occasion, when this is clear from the contexts, these subscripts are left unspecified. Moreover, to shorten writing, we regularly use the term capacity instead of \emph{statistical} capacity.

%%%%%%%%%%%%%%%%%%%%%%%%%%%%%%%%%%%%%%%%%%%%%%%%%%%%%%%%%%%%%%%%%%%%%%%%%%%%%%%%%%%%%%%%%%%%%%%%%%%%%%%%%%%%%%
%%%%%%%%%%%%%%%%%%%%%%%%%%%%%%%%%%%%%%%%%%%%%%%%%%%%%%%%%%%%%%%%%%%%%%%%%%%%%%%%%%%%%%%%%%%%%%%%%%%%%%%%%%%%%%
%%%%%%%%%%%%%%%%%%%%%%%%%%%%%%%%%%%%%%%%%%%%%%%%%%%%%%%%%%%%%%%%%%%%%%%%%%%%%%%%%%%%%%%%%%%%%%%%%%%%%%%%%%%%%%
%%%%%%%%%%%%%%%%%%%%%%%%%%%%%%%%%%%%%%%%%%%%%%%%%%%%%%%%%%%%%%%%%%%%%%%%%%%%%%%%%%%%%%%%%%%%%%%%%%%%%%%%%%%%%%
\subsection{Rfps $\longleftrightarrow$ (partially reciprocal) \emph{free energy} connection}
\label{secrfpsfe}
%%%%%%%%%%%%%%%%%%%%%%%%%%%%%%%%%%%%%%%%%%%%%%%%%%%%%%%%%%%%%%%%%%%%%%%%%%%%%%%%%%%%%%%%%%%%%%%%%%%%%%%%%%%%%%
%%%%%%%%%%%%%%%%%%%%%%%%%%%%%%%%%%%%%%%%%%%%%%%%%%%%%%%%%%%%%%%%%%%%%%%%%%%%%%%%%%%%%%%%%%%%%%%%%%%%%%%%%%%%%%
%%%%%%%%%%%%%%%%%%%%%%%%%%%%%%%%%%%%%%%%%%%%%%%%%%%%%%%%%%%%%%%%%%%%%%%%%%%%%%%%%%%%%%%%%%%%%%%%%%%%%%%%%%%%%%
%%%%%%%%%%%%%%%%%%%%%%%%%%%%%%%%%%%%%%%%%%%%%%%%%%%%%%%%%%%%%%%%%%%%%%%%%%%%%%%%%%%%%%%%%%%%%%%%%%%%%%%%%%%%%%

In the previous section, we have established that studying the random feasibility problems (rfps) is critically important for the NSP's capacity analytical characterization. In this section we extend this connection to studying \emph{free energies}. These object are well known and almost unavoidable in many statistical physics consideration. To introduce them in a mathematically proper way that would be of use here, we start by defining the following, so-called, \emph{bilinear Hamiltonian}
\begin{equation}
\cH_{sq}(G)= \y^TG\x,\label{eq:ham1}
\end{equation}
and its corresponding (so to say, \emph{partially reciprocal}) partition function
\begin{equation}
Z_{sq}(\beta,G)=\sum_{\x\in\cX} \lp \sum_{\y\in\cY}e^{\beta\cH_{sq}(G)}\rp^{-1}.  \label{eq:partfun}
\end{equation}
 To ensure an overall generality of the exposition, we, in (\ref{eq:partfun}), take $\cX$ and $\cY$ as general sets (fairly soon, we make specializations, $\cX=\mS^n$ and $\cY=\mS_+^m$,  necessary for perceptrons' consideration of our interest here). One quickly notes, the reciprocal nature of the inner summation, which makes the partition function given in  (\ref{eq:partfun}) somewhat different from the counterparts typically seen in statistical physics literature. The corresponding  thermodynamic limit of the average ``\emph{partially reciprocal}'' free energy is then given as
\begin{eqnarray}
f_{sq}(\beta) & = & \lim_{n\rightarrow\infty}\frac{\mE_G\log{(Z_{sq}(\beta,G)})}{\beta \sqrt{n}}
=\lim_{n\rightarrow\infty} \frac{\mE_G\log\lp \sum_{\x\in\cX} \lp \sum_{\y\in\cY}e^{\beta\cH_{sq}(G)}\rp^{-1}\rp}{\beta \sqrt{n}} \nonumber \\
& = &\lim_{n\rightarrow\infty} \frac{\mE_G\log\lp \sum_{\x\in\cX} \lp \sum_{\y\in\cY}e^{\beta\y^TG\x)}\rp^{-1}\rp}{\beta \sqrt{n}}.\label{eq:logpartfunsqrt}
\end{eqnarray}
The ground state special case is obtained by considering the so-called ``zero-temperature'' ($T\rightarrow 0$ or  $\beta=\frac{1}{T}\rightarrow\infty$) regime
\begin{eqnarray}
f_{sq}(\infty)   \triangleq    \lim_{\beta\rightarrow\infty}f_{sq}(\beta) & = &
\lim_{\beta,n\rightarrow\infty}\frac{\mE_G\log{(Z_{sq}(\beta,G)})}{\beta \sqrt{n}}
=
 \lim_{n\rightarrow\infty}\frac{\mE_G \max_{\x\in\cX}  -  \max_{\y\in\cY} \y^TG\x}{\sqrt{n}} \nonumber \\
& = & - \lim_{n\rightarrow\infty}\frac{\mE_G \min_{\x\in\cX}  \max_{\y\in\cY} \y^TG\x}{\sqrt{n}}.
  \label{eq:limlogpartfunsqrta0}
\end{eqnarray}
Restricting to $G$'s comprised of iid standard normals allows to utilize their sign symmetry and rewrite the above as
\begin{eqnarray}
-f_{sq}(\infty)
& = &  \lim_{n\rightarrow\infty}\frac{\mE_G \min_{\x\in\cX}  \max_{\y\in\cY} \y^TG\x}{\sqrt{n}}  = \lim_{n\rightarrow\infty}\frac{\mE_G \min_{\x\in\cX}  \max_{\y\in\cY} -\y^TG\x}{\sqrt{n}}.
  \label{eq:limlogpartfunsqrt}
\end{eqnarray}
It is not that difficult to see that (\ref{eq:limlogpartfunsqrt}) is directly related to (\ref{eq:ex3}). This, on the other hand, also implies that $f_{sq}(\infty)$ is very tightly connected to $\xi_{feas}(0,\cX)$ , which  hints that understanding  $f_{sq}(\infty)$ is likely to play critically important role in understanding and ultimately characterizing both  $\xi_{feas}(0,\cX)$ and the NSPs capacity. This is, in fact, exactly what happens in the sections that follow below. Namely, since studying $f_{sq}(\infty)$ directly is not very easy, we rely on studying $f_{sq}(\beta)$. In other words, we study the above introduced partially reciprocal variant of the free energy for a general $\beta$  and then specialize the obtained results  to the ground state, $\beta\rightarrow\infty$, regime. In the interest of easing the exposition, we, however, on occasion neglect some terms which paly no  significant role in the ground state considerations.

%%%%%%%%%%%%%%%%%%%%%%%%%%%%%%%%%%%%%%%%%%%%%%%%%%%%%%%%%%%%%%%%%
%%%%%%%%%%%%%%%%%%%%%%%%%%%%%%%%%%%%%%%%%%%%%%%%%%%%%%%%%%%%%%%%%
%%%%%%%%%%%%%%%%%%%%%%%%%%%%%%%%%%%%%%%%%%%%%%%%%%%%%%%%%%%%%%%%%
%%%%%%%%%%%%%%%%%%%%%%%%%%%%%%%%%%%%%%%%%%%%%%%%%%%%%%%%%%%%%%%%%
%%%%%%%%%%%%%%%%%%%%%%%%%%%%%%%%%%%%%%%%%%%%%%%%%%%%%%%%%%%%%%%%%
%%%%%%%%%%%%%%%%%%%%%%%%%%%%%%%%%%%%%%%%%%%%%%%%%%%%%%%%%%%%%%%%%
\section{Negative spherical perceptrons through the prism of sfl RDT}
\label{sec:randlincons}
%%%%%%%%%%%%%%%%%%%%%%%%%%%%%%%%%%%%%%%%%%%%%%%%%%%%%%%%%%%%%%%%%
%%%%%%%%%%%%%%%%%%%%%%%%%%%%%%%%%%%%%%%%%%%%%%%%%%%%%%%%%%%%%%%%%
%%%%%%%%%%%%%%%%%%%%%%%%%%%%%%%%%%%%%%%%%%%%%%%%%%%%%%%%%%%%%%%%%
%%%%%%%%%%%%%%%%%%%%%%%%%%%%%%%%%%%%%%%%%%%%%%%%%%%%%%%%%%%%%%%%%
%%%%%%%%%%%%%%%%%%%%%%%%%%%%%%%%%%%%%%%%%%%%%%%%%%%%%%%%%%%%%%%%%
%%%%%%%%%%%%%%%%%%%%%%%%%%%%%%%%%%%%%%%%%%%%%%%%%%%%%%%%%%%%%%%%%

We start with one of the key observations that enables pretty much everything that follows. It is precisely the recognition that the free energy from (\ref{eq:logpartfunsqrt}),
\begin{eqnarray}
f_{sq}(\beta) & = &\lim_{n\rightarrow\infty} \frac{\mE_G\log\lp \sum_{\x\in\cX} \lp \sum_{\y\in\cY}e^{\beta\y^TG\x)}\rp^{-1}\rp}{\beta \sqrt{n}},\label{eq:hmsfl1}
\end{eqnarray}
is a function of \emph{bilinearly indexed} (bli) random process $\y^TG\x$. Such a recognition then puts us in position to establish a connection between $f_{sq}(\beta)$ and the bli related results of \cite{Stojnicsflgscompyx23,Stojnicnflgscompyx23,Stojnicflrdt23}. To do so, we closely follow \cite{Stojnichopflrdt23,Stojnicbinperflrdt23} and start with a collection of needed technical definitions. For $r\in\mN$, $k\in\{1,2,\dots,r+1\}$, real scalars $s$, $x$, and $y$  such that $s^2=1$, $x>0$, and $y>0$, sets $\cX\subseteq \mR^n$ and $\cY\subseteq \mR^m$, function $f_S(\cdot):\mR^n\rightarrow R$, vectors $\p=[\p_0,\p_1,\dots,\p_{r+1}]$, $\q=[\q_0,\q_1,\dots,\q_{r+1}]$, and $\c=[\c_0,\c_1,\dots,\c_{r+1}]$ such that
 \begin{eqnarray}\label{eq:hmsfl2}
1=\p_0\geq \p_1\geq \p_2\geq \dots \geq \p_r\geq \p_{r+1} & = & 0 \nonumber \\
1=\q_0\geq \q_1\geq \q_2\geq \dots \geq \q_r\geq \q_{r+1} & = &  0,
 \end{eqnarray}
$\c_0=1$, $\c_{r+1}=0$, and ${\mathcal U}_k\triangleq [u^{(4,k)},\u^{(2,k)},\h^{(k)}]$  such that the components of  $u^{(4,k)}\in\mR$, $\u^{(2,k)}\in\mR^m$, and $\h^{(k)}\in\mR^n$ are i.i.d. standard normals, we set
  \begin{eqnarray}\label{eq:fl4}
\psi_{S,\infty}(f_{S},\calX,\calY,\p,\q,\c,x,y,s)  =
 \mE_{G,{\mathcal U}_{r+1}} \frac{1}{n\c_r} \log
\lp \mE_{{\mathcal U}_{r}} \lp \dots \lp \mE_{{\mathcal U}_3}\lp\lp\mE_{{\mathcal U}_2} \lp \lp Z_{S,\infty}\rp^{\c_2}\rp\rp^{\frac{\c_3}{\c_2}}\rp\rp^{\frac{\c_4}{\c_3}} \dots \rp^{\frac{\c_{r}}{\c_{r-1}}}\rp, \nonumber \\
 \end{eqnarray}
where
\begin{eqnarray}\label{eq:fl5}
Z_{S,\infty} & \triangleq & e^{D_{0,S,\infty}} \nonumber \\
 D_{0,S,\infty} & \triangleq  & \max_{\x\in\cX,\|\x\|_2=x} s \max_{\y\in\cY,\|\y\|_2=y}
 \lp \sqrt{n} f_{S}
+\sqrt{n}  y    \lp\sum_{k=2}^{r+1}c_k\h^{(k)}\rp^T\x
+ \sqrt{n} x \y^T\lp\sum_{k=2}^{r+1}b_k\u^{(2,k)}\rp \rp \nonumber  \\
 b_k & \triangleq & b_k(\p,\q)=\sqrt{\p_{k-1}-\p_k} \nonumber \\
c_k & \triangleq & c_k(\p,\q)=\sqrt{\q_{k-1}-\q_k}.
 \end{eqnarray}
Having all the above definitions set, we are in position to recall on the following theorem -- unquestionably, one of key fundamental components of sfl RDT.
\begin{theorem} \cite{Stojnicflrdt23}
\label{thm:thmsflrdt1}  Consider large $n$ context with  $\alpha=\lim_{n\rightarrow\infty} \frac{m}{n}$, remaining constant as  $n$ grows. Let the elements of  $G\in\mR^{m\times n}$
 be i.i.d. standard normals and let $\cX\subseteq \mR^n$ and $\cY\subseteq \mR^m$ be two given sets. Assume the complete sfl RDT frame from \cite{Stojnicsflgscompyx23} and consider a given function $f(\y):R^m\rightarrow R$. Set
\begin{align}\label{eq:thmsflrdt2eq1}
   \psi_{rp} & \triangleq - \max_{\x\in\cX} s \max_{\y\in\cY} \lp f(\y)+\y^TG\x \rp
   \qquad  \mbox{(\bl{\textbf{random primal}})} \nonumber \\
   \psi_{rd}(\p,\q,\c,x,y,s) & \triangleq    \frac{x^2y^2}{2}    \sum_{k=2}^{r+1}\Bigg(\Bigg.
   \p_{k-1}\q_{k-1}
   -\p_{k}\q_{k}
  \Bigg.\Bigg)
\c_k
  - \psi_{S,\infty}(f(\y),\calX,\calY,\p,\q,\c,x,y,s) \hspace{.03in} \mbox{(\bl{\textbf{fl random dual}})}. \nonumber \\
 \end{align}
Let $\hat{\p_0}\rightarrow 1$, $\hat{\q_0}\rightarrow 1$, and $\hat{\c_0}\rightarrow 1$, $\hat{\p}_{r+1}=\hat{\q}_{r+1}=\hat{\c}_{r+1}=0$, and let the non-fixed parts of $\hat{\p}\triangleq \hat{\p}(x,y)$, $\hat{\q}\triangleq \hat{\q}(x,y)$, and  $\hat{\c}\triangleq \hat{\c}(x,y)$ be the solutions of the following system
\begin{eqnarray}\label{eq:thmsflrdt2eq2}
   \frac{d \psi_{rd}(\p,\q,\c,x,y,s)}{d\p} =  0,\quad
   \frac{d \psi_{rd}(\p,\q,\c,x,y,s)}{d\q} =  0,\quad
   \frac{d \psi_{rd}(\p,\q,\c,x,y,s)}{d\c} =  0.
 \end{eqnarray}
 Then,
\begin{eqnarray}\label{eq:thmsflrdt2eq3}
    \lim_{n\rightarrow\infty} \frac{\mE_G  \psi_{rp}}{\sqrt{n}}
  & = &
\min_{x>0} \max_{y>0} \lim_{n\rightarrow\infty} \psi_{rd}(\hat{\p}(x,y),\hat{\q}(x,y),\hat{\c}(x,y),x,y,s) \qquad \mbox{(\bl{\textbf{strong sfl random duality}})},\nonumber \\
 \end{eqnarray}
where $\psi_{S,\infty}(\cdot)$ is as in (\ref{eq:fl4})-(\ref{eq:fl5}).
 \end{theorem}
\begin{proof}
The $s=-1$ scenario follows directly from the corresponding one proven in \cite{Stojnicflrdt23} after a cosmetic change $f(\x)\rightarrow f(\y)$. On the other hand, the $s=1$ scenario, follows after trivial adjustments and a line-by-line repetition of the arguments of Section 3 of \cite{Stojnicflrdt23} with $s=-1$ replaced by $s=1$ and $f(\x)$ replaced by $f(\y)$.
 \end{proof}

Clearly, the above theorem is very generic and holds for any given sets $\cX$ and $\cY$. The corollary that follows below makes it fully operational for the case of spherical perceptrons which are of our interest here.
\begin{corollary}
\label{cor:cor1}  Assume the setup of Theorem \ref{thm:thmsflrdt1} with $\cX$ and $\cY$ having the unit norm elements. Set
\begin{align}\label{eq:thmsflrdt2eq1a0}
   \psi_{rp} & \triangleq - \max_{\x\in\cX} s \max_{\y\in\cY} \lp \y^TG\x + \kappa \y^T\1 \rp
   \qquad  \mbox{(\bl{\textbf{random primal}})} \nonumber \\
   \psi_{rd}(\p,\q,\c,x,y,s) & \triangleq    \frac{1}{2}    \sum_{k=2}^{r+1}\Bigg(\Bigg.
   \p_{k-1}\q_{k-1}
   -\p_{k}\q_{k}
  \Bigg.\Bigg)
\c_k
  - \psi_{S,\infty}(\kappa\y^T\1,\calX,\calY,\p,\q,\c,1,1,s) \quad \mbox{(\bl{\textbf{fl random dual}})}. \nonumber \\
 \end{align}
Let the non-fixed parts of $\hat{\p}$, $\hat{\q}$, and  $\hat{\c}$ be the solutions of the following system
\begin{eqnarray}\label{eq:thmsflrdt2eq2a0}
   \frac{d \psi_{rd}(\p,\q,\c,1,1,s)}{d\p} =  0,\quad
   \frac{d \psi_{rd}(\p,\q,\c,1,1,s)}{d\q} =  0,\quad
   \frac{d \psi_{rd}(\p,\q,\c,1,1,s)}{d\c} =  0.
 \end{eqnarray}
 Then,
\begin{eqnarray}\label{eq:thmsflrdt2eq3a0}
    \lim_{n\rightarrow\infty} \frac{\mE_G  \psi_{rp}}{\sqrt{n}}
  & = &
 \lim_{n\rightarrow\infty} \psi_{rd}(\hat{\p},\hat{\q},\hat{\c},1,1,s) \qquad \mbox{(\bl{\textbf{strong sfl random duality}})},\nonumber \\
 \end{eqnarray}
where $\psi_{S,\infty}(\cdot)$ is as in (\ref{eq:fl4})-(\ref{eq:fl5}).
 \end{corollary}
\begin{proof}
Follows trivially as a direct consequence of Theorem \ref{thm:thmsflrdt1}, after choosing $f(\y)=\kappa\y^T\1$ and recognizing that all elements in $\cX$ and $\cY$ are of unit norm.
 \end{proof}

As  \cite{Stojnicflrdt23,Stojnichopflrdt23} noted, the above random primal problems' trivial concentrations enable various corresponding probabilistic variants of (\ref{eq:thmsflrdt2eq3}) and (\ref{eq:thmsflrdt2eq3a0}) as well. We, however, skip stating such trivialities.

%%%%%%%%%%%%%%%%%%%%%%%%%%%%%%%%%%%%%%%%%%%%%%%%%%%%%%%%%%%%%%%%%%%%%%%%%%%%%%%%
%%%%%%%%%%%%%%%%%%%%%%%%%%%%%%%%%%%%%%%%%%%%%%%%%%%%%%%%%%%%%%%%%%%%%%%%%%%%%%%%
%%%%%%%%%%%%%%%%%%%%%%%%%%%%%%%%%%%%%%%%%%%%%%%%%%%%%%%%%%%%%%%%%%%%%%%%%%%%%%%%
%%%%%%%%%%%%%%%%%%%%%%%%%%%%%%%%%%%%%%%%%%%%%%%%%%%%%%%%%%%%%%%%%%%%%%%%%%%%%%%%
%%%%%%%%%%%%%%%%%%%%%%%%%%%%%%%%%%%%%%%%%%%%%%%%%%%%%%%%%%%%%%%%%%%%%%%%%%%%%%%%
\section{Practical realization}
\label{sec:prac}
%%%%%%%%%%%%%%%%%%%%%%%%%%%%%%%%%%%%%%%%%%%%%%%%%%%%%%%%%%%%%%%%%%%%%%%%%%%%%%%%
%%%%%%%%%%%%%%%%%%%%%%%%%%%%%%%%%%%%%%%%%%%%%%%%%%%%%%%%%%%%%%%%%%%%%%%%%%%%%%%%
%%%%%%%%%%%%%%%%%%%%%%%%%%%%%%%%%%%%%%%%%%%%%%%%%%%%%%%%%%%%%%%%%%%%%%%%%%%%%%%%
%%%%%%%%%%%%%%%%%%%%%%%%%%%%%%%%%%%%%%%%%%%%%%%%%%%%%%%%%%%%%%%%%%%%%%%%%%%%%%%%
%%%%%%%%%%%%%%%%%%%%%%%%%%%%%%%%%%%%%%%%%%%%%%%%%%%%%%%%%%%%%%%%%%%%%%%%%%%%%%%%

To have the results of Theorem \ref{thm:thmsflrdt1} and Corollary \ref{cor:cor1} become practically useful, one needs to ensure that all the underlying quantities can be valuated. Two key obstacles might pose a problem in that regard: (i) It is a priori not clear what should be the correct value for $r$; and (ii) Sets $\cX$ and $\cY$ do not have a component-wise structure characterization which does not provide any guarantee that the decoupling over both $\x$ and $\y$ is very straightforward. It turns out, however, that neither of these potential obstacles is unsurpassable.

After specialization to $\cX=\mS^n$ and $\cY=\mS_+^m$, we rely on results of Corollary \ref{cor:cor1} and start by observing that the key object of practical interest is the following \emph{random dual}
\begin{align}\label{eq:prac1}
    \psi_{rd}(\p,\q,\c,1,1,s) & \triangleq    \frac{1}{2}    \sum_{k=2}^{r+1}\Bigg(\Bigg.
   \p_{k-1}\q_{k-1}
   -\p_{k}\q_{k}
  \Bigg.\Bigg)
\c_k
  - \psi_{S,\infty}(0,\calX,\calY,\p,\q,\c,1,1,s). \nonumber \\
  & =   \frac{1}{2}    \sum_{k=2}^{r+1}\Bigg(\Bigg.
   \p_{k-1}\q_{k-1}
   -\p_{k}\q_{k}
  \Bigg.\Bigg)
\c_k
  - \frac{1}{n}\varphi(D^{(per)}(s)) - \frac{1}{n}\varphi(D^{(sph)}(s)), \nonumber \\
  \end{align}
where analogously to (\ref{eq:fl4})-(\ref{eq:fl5})
  \begin{eqnarray}\label{eq:prac2}
\varphi(D,\c) & = &
 \mE_{G,{\mathcal U}_{r+1}} \frac{1}{\c_r} \log
\lp \mE_{{\mathcal U}_{r}} \lp \dots \lp \mE_{{\mathcal U}_3}\lp\lp\mE_{{\mathcal U}_2} \lp
\lp    e^{D}   \rp^{\c_2}\rp\rp^{\frac{\c_3}{\c_2}}\rp\rp^{\frac{\c_4}{\c_3}} \dots \rp^{\frac{\c_{r}}{\c_{r-1}}}\rp, \nonumber \\
  \end{eqnarray}
and
\begin{eqnarray}\label{eq:prac3}
D^{(per)}(s) & = & \max_{\x\in\cX} \lp   s\sqrt{n}      \lp\sum_{k=2}^{r+1}c_k\h^{(k)}\rp^T\x  \rp \nonumber \\
  D^{(sph)}(s) & \triangleq  &   s \max_{\y\in\cY}
\lp \sqrt{n} \kappa \y^T\1 + \sqrt{n}  \y^T\lp\sum_{k=2}^{r+1}b_k\u^{(2,k)}\rp \rp.
 \end{eqnarray}
After a simple evaluation, we find
\begin{eqnarray}\label{eq:prac4}
D^{(per)}(s) & = & \max_{\x\in\cX}   \lp s\sqrt{n}      \lp\sum_{k=2}^{r+1}c_k\h^{(k)}\rp^T\x \rp =
\sqrt{n} \max_{\x\in\mS^n}   \lp s     \lp\sum_{k=2}^{r+1}c_k\h^{(k)}\rp^T\x \rp =
\sqrt{n} \left\|       \sum_{k=2}^{r+1}c_k\h^{(k)} \right\|_2. \nonumber \\
 \end{eqnarray}
We now utilize the  \emph{square root trick} introduced on numerous occasions in \cite{StojnicMoreSophHopBnds10,StojnicLiftStrSec13,StojnicGardSphErr13,StojnicGardSphNeg13}
\begin{eqnarray}\label{eq:prac4a00}
D^{(per)}(s) & = &
\sqrt{n} \left\|       \sum_{k=2}^{r+1}c_k\h^{(k)} \right\|_2
=\sqrt{n} \min_{\gamma^{(p)}} \lp \frac{\left\|      \sum_{k=2}^{r+1}c_k\h^{(k)} \right\|_2^2}{4\gamma^{(p)}} +\gamma^{(p)} \rp \nonumber \\
& = & \sqrt{n} \min_{\gamma^{(p)}} \lp \frac{\sum_{i=1}^{n}\lp\sum_{k=2}^{r+1}c_k\h_i^{(k)} \rp^2}{4\gamma^{(p)}} +\gamma^{(p)} \rp.
 \end{eqnarray}
After introducing scaling $\gamma^{(p)}=\gamma^{(p)}_{sq}\sqrt{n}$, we rewrite (\ref{eq:prac4a00}) as
\begin{eqnarray}\label{eq:prac4a01}
D^{(per)}(s) & = &
  \sqrt{n} \min_{\gamma_{sq}^{(p)}} \lp \frac{\sum_{i=1}^{n}\lp\sum_{k=2}^{r+1}c_k\h_i^{(k)} \rp^2}{4\gamma_{sq}^{(p)}\sqrt{n}} +\gamma_{sq}^{(p)}\sqrt{n} \rp
  =
  \min_{\gamma_{sq}^{(p)}} \lp \frac{\sum_{i=1}^{n}\lp\sum_{k=2}^{r+1}c_k\h_i^{(k)} \rp^2}{4\gamma_{sq}^{(p)}} +\gamma_{sq}^{(p)}n \rp \nonumber \\
 & = &
  \min_{\gamma_{sq}^{(p)}} \lp \sum_{i=1}^nD^{(per)}_i(c_k) +\gamma_{sq}^{(p)}n \rp \nonumber \\  .
 \end{eqnarray}
where
\begin{eqnarray}\label{eq:prac5}
D^{(per)}_i(c_k)=\frac{\lp\sum_{k=2}^{r+1}c_k\h_i^{(k)} \rp^2}{4\gamma_{sq}^{(p)}}.
\end{eqnarray}
In a similar fashion (and following \cite{Stojnicbinperflrdt23}), we also have
\begin{eqnarray}\label{eq:prac7}
   D^{(sph)}(s) & \triangleq  &   s \sqrt{n}  \max_{\y\in\cY}
\lp  \kappa \y^T\1 +   \y^T\lp\sum_{k=2}^{r+1}b_k\u^{(2,k)}\rp \rp
=  s  \sqrt{n}   \left \| \max \lp \kappa\1 +\sum_{k=2}^{r+1}b_k\u^{(2,k)},0 \rp  \right \|_2.
 \end{eqnarray}
Utilizing again the  \emph{square root trick}, we obtain
\begin{align}\label{eq:prac8}
   D^{(sph)} (s)
& =   \sqrt{n}  s \left \| \max \lp \kappa\1 + \sqrt{n}\sum_{k=2}^{r+1}b_k\u^{(2,k)},0 \rp  \right \|_2
=  s\sqrt{n}  \min_{\gamma} \lp \frac{\left \| \max \lp \kappa\1 +  \sum_{k=2}^{r+1}b_k\u^{(2,k)},0 \rp  \right \|_2^2}{4\gamma}+\gamma \rp \nonumber \\
 & =   s\sqrt{n}  \min_{\gamma} \lp \frac{\sum_{i=1}^{m}  \max \lp \kappa +  \sum_{k=2}^{r+1}b_k\u_i^{(2,k)},0 \rp ^2}{4\gamma}+\gamma \rp.
 \end{align}
After introducing scaling $\gamma=\gamma_{sq}\sqrt{n}$, (\ref{eq:prac8}) can be rewritten as
\begin{eqnarray}\label{eq:prac9}
   D^{(sph)}(s)
  & =  & s\sqrt{n}  \min_{\gamma_{sq}} \lp \frac{\sum_{i=1}^{m} \max \lp \kappa + \sum_{k=2}^{r+1}b_k\u_i^{(2,k)},0 \rp^2}{4\gamma_{sq}\sqrt{n}}+\gamma_{sq}\sqrt{n} \rp   \nonumber \\
  & = & s \min_{\gamma_{sq}} \lp \frac{\sum_{i=1}^{m} \max \lp \kappa + \sum_{k=2}^{r+1}b_k\u_i^{(2,k)},0  \rp^2}{4\gamma_{sq}}+\gamma_{sq}n \rp \nonumber \\
  & =  &  s \min_{\gamma_{sq}} \lp \sum_{i=1}^{m} D_i^{(sph)}(b_k)+\gamma_{sq}n \rp, \nonumber \\
 \end{eqnarray}
with
\begin{eqnarray}\label{eq:prac10}
   D_i^{(sph)}(b_k)= \frac{\max \lp \kappa + \sum_{k=2}^{r+1}b_k\u_i^{(2,k)},0  \rp^2}{4\gamma_{sq}}.
 \end{eqnarray}

%%%%%%%%%%%%%%%%%%%%%%%%%%%%%%%%%%%%%%%%%%%%%%%%%%%%%%%%%%%%%%%%%%%%%%%%%%%%%%%%%%%%%%%%%%%%%%%%%
%%%%%%%%%%%%%%%%%%%%%%%%%%%%%%%%%%%%%%%%%%%%%%%%%%%%%%%%%%%%%%%%%%%%%%%%%%%%%%%%%%%%%%%%%%%%%%%%%
%%%%%%%%%%%%%%%%%%%%%%%%%%%%%%%%%%%%%%%%%%%%%%%%%%%%%%%%%%%%%%%%%%%%%%%%%%%%%%%%%%%%%%%%%%%%%%%%%
%%%%%%%%%%%%%%%%%%%%%%%%%%%%%%%%%%%%%%%%%%%%%%%%%%%%%%%%%%%%%%%%%%%%%%%%%%%%%%%%%%%%%%%%%%%%%%%%%
%%%%%%%%%%%%%%%%%%%%%%%%%%%%%%%%%%%%%%%%%%%%%%%%%%%%%%%%%%%%%%%%%%%%%%%%%%%%%%%%%%%%%%%%%%%%%%%%%
\subsection{$s=-1$ particularization}
\label{sec:neg}
%%%%%%%%%%%%%%%%%%%%%%%%%%%%%%%%%%%%%%%%%%%%%%%%%%%%%%%%%%%%%%%%%%%%%%%%%%%%%%%%%%%%%%%%%%%%%%%%%
%%%%%%%%%%%%%%%%%%%%%%%%%%%%%%%%%%%%%%%%%%%%%%%%%%%%%%%%%%%%%%%%%%%%%%%%%%%%%%%%%%%%%%%%%%%%%%%%%
%%%%%%%%%%%%%%%%%%%%%%%%%%%%%%%%%%%%%%%%%%%%%%%%%%%%%%%%%%%%%%%%%%%%%%%%%%%%%%%%%%%%%%%%%%%%%%%%%
%%%%%%%%%%%%%%%%%%%%%%%%%%%%%%%%%%%%%%%%%%%%%%%%%%%%%%%%%%%%%%%%%%%%%%%%%%%%%%%%%%%%%%%%%%%%%%%%%
%%%%%%%%%%%%%%%%%%%%%%%%%%%%%%%%%%%%%%%%%%%%%%%%%%%%%%%%%%%%%%%%%%%%%%%%%%%%%%%%%%%%%%%%%%%%%%%%%

Taking $s=-1$ gives us the opportunity to establish a direct connection between the ground state energy, $f_{sq}(\infty)$ given in (\ref{eq:limlogpartfunsqrt}), and the random primal of the above machinery, $\psi_{rp}(\cdot)$, given in Corollary \ref{cor:cor1}. In concrete terms, this basically means the following
 \begin{eqnarray}
-f_{sq}(\infty)
 & = &
- \lim_{n\rightarrow\infty}\frac{\mE_G \max_{\x\in\cX}  -  \max_{\y\in\cY} \y^TG\x}{\sqrt{n}}
 =
    \lim_{n\rightarrow\infty} \frac{\mE_G  \psi_{rp}}{\sqrt{n}}
   =
 \lim_{n\rightarrow\infty} \psi_{rd}(\hat{\p},\hat{\q},\hat{\c},1,1,-1),
  \label{eq:negprac11}
\end{eqnarray}
where the non-fixed parts of $\hat{\p}$, $\hat{\q}$, and  $\hat{\c}$ are the solutions of the following system
\begin{eqnarray}\label{eq:negprac12}
   \frac{d \psi_{rd}(\p,\q,\c,1,1,-1)}{d\p} =  0,\quad
   \frac{d \psi_{rd}(\p,\q,\c,1,1,-1)}{d\q} =  0,\quad
   \frac{d \psi_{rd}(\p,\q,\c,1,1,-1)}{d\c} =  0.
 \end{eqnarray}
Relying on (\ref{eq:prac1})-(\ref{eq:prac10}), we further have
 \begin{eqnarray}
 \lim_{n\rightarrow\infty} \psi_{rd}(\hat{\p},\hat{\q},\hat{\c},1,1,-1) =  \bar{\psi}_{rd}(\hat{\p},\hat{\q},\hat{\c},\hat{\gamma}_{sq},\hat{\gamma}_{sq}^{(p)},1,1,-1),
  \label{eq:negprac12a}
\end{eqnarray}
with
\begin{eqnarray}\label{eq:negprac13}
    \bar{\psi}_{rd}(\p,\q,\c,\gamma_{sq},\gamma_{sq}^{(p)},1,1,-1)   & = &  \frac{1}{2}    \sum_{k=2}^{r+1}\Bigg(\Bigg.
   \p_{k-1}\q_{k-1}
   -\p_{k}\q_{k}
  \Bigg.\Bigg)
\c_k
\nonumber \\
& & -\gamma_{sq}^{(p)} - \varphi(D_1^{(per)}(c_k(\p,\q)),\c) +\gamma_{sq}- \alpha\varphi(-D_1^{(sph)}(b_k(\p,\q)),\c).\nonumber \\
  \end{eqnarray}
Connecting  (\ref{eq:negprac11}), (\ref{eq:negprac12a}), and (\ref{eq:negprac13}), we further find
 \begin{eqnarray}
-f_{sq}(\infty)
& = &  -\lim_{n\rightarrow\infty}\frac{\mE_G \max_{\x\in\cX}  -  \max_{\y\in\cY} \y^TG\x}{\sqrt{n}} \nonumber \\
    &  = &
 \lim_{n\rightarrow\infty} \psi_{rd}(\hat{\p},\hat{\q},\hat{\c},1,1,-1)
 =   \bar{\psi}_{rd}(\hat{\p},\hat{\q},\hat{\c},\hat{\gamma}_{sq},\hat{\gamma}_{sq}^{(p)},1,1,-1) \nonumber \\
 & = &   \frac{1}{2}    \sum_{k=2}^{r+1}\Bigg(\Bigg.
   \hat{\p}_{k-1}\hat{\q}_{k-1}
   -\hat{\p}_{k}\hat{\q}_{k}
  \Bigg.\Bigg)
\hat{\c}_k \nonumber \\
& &
-\hat{\gamma}_{sq}^{(p)}  - \varphi(D_1^{(per)}(c_k(\hat{\p},\hat{\q})),\c) + \hat{\gamma}_{sq} - \alpha\varphi(-D_1^{(sph)}(b_k(\hat{\p},\hat{\q})),\c). \nonumber \\
  \label{eq:negprac18}
\end{eqnarray}
The following theorem summarizes the above mechanism.

\begin{theorem}
  \label{thme:negthmprac1}
  Assume the complete sfl RDT setup of \cite{Stojnicsflgscompyx23}. Consider large $n$ linear regime with $\alpha=\lim_{n\rightarrow\infty} \frac{m}{n}$ and $\varphi(\cdot)$ and $\bar{\psi}(\cdot)$ from (\ref{eq:prac2}) and (\ref{eq:negprac13}). Let the ``fixed'' parts of $\hat{\p}$, $\hat{\q}$, and $\hat{\c}$ satisfy $\hat{\p}_1\rightarrow 1$, $\hat{\q}_1\rightarrow 1$, $\hat{\c}_1\rightarrow 1$, $\hat{\p}_{r+1}=\hat{\q}_{r+1}=\hat{\c}_{r+1}=0$, and let the ``non-fixed'' parts of $\hat{\p}_k$, $\hat{\q}_k$, and $\hat{\c}_k$ ($k\in\{2,3,\dots,r\}$) be the solutions of the following system of equations
  \begin{eqnarray}\label{eq:negthmprac1eq1}
   \frac{d \bar{\psi}_{rd}(\p,\q,\c,\gamma_{sq},\gamma_{sq}^{(p)},1,1,-1)}{d\p} =  0 \nonumber \\
   \frac{d \bar{\psi}_{rd}(\p,\q,\c,\gamma_{sq},\gamma_{sq}^{(p)},1,1,-1)}{d\q} =  0 \nonumber \\
   \frac{d \bar{\psi}_{rd}(\p,\q,\c,\gamma_{sq},\gamma_{sq}^{(p)},1,1,-1)}{d\c} =  0 \nonumber \\
   \frac{d \bar{\psi}_{rd}(\p,\q,\c,\gamma_{sq},\gamma_{sq}^{(p)},1,1,-1)}{d\gamma_{sq}} =  0\nonumber \\
   \frac{d \bar{\psi}_{rd}(\p,\q,\c,\gamma_{sq},\gamma_{sq}^{(p)},1,1,-1)}{d\gamma_{sq}^{(p)}} =  0,
 \end{eqnarray}
 and, consequently, let
\begin{eqnarray}\label{eq:prac17}
c_k(\hat{\p},\hat{\q})  & = & \sqrt{\hat{\q}_{k-1}-\hat{\q}_k} \nonumber \\
b_k(\hat{\p},\hat{\q})  & = & \sqrt{\hat{\p}_{k-1}-\hat{\p}_k}.
 \end{eqnarray}
 Then
 \begin{equation}
-f_{sq}(\infty)
 =      \frac{1}{2}    \sum_{k=2}^{r+1}\Bigg(\Bigg.
   \hat{\p}_{k-1}\hat{\q}_{k-1}
   -\hat{\p}_{k}\hat{\q}_{k}
  \Bigg.\Bigg)
\hat{\c}_k
 -\hat{\gamma}_{sq}^{(p)} - \varphi(D_1^{(per)}(c_k(\hat{\p},\hat{\q})),\hat{\c}) + \hat{\gamma}_{sq} - \alpha\varphi(-D_1^{(sph)}(b_k(\hat{\p},\hat{\q})),\hat{\c}).
  \label{eq:negthmprac1eq2}
\end{equation}
\end{theorem}
\begin{proof}
Follows from the previous discussion, Theorem \ref{thm:thmsflrdt1}, Corollary \ref{cor:cor1}, and the sfl RDT machinery presented in \cite{Stojnicnflgscompyx23,Stojnicsflgscompyx23,Stojnicflrdt23,Stojnichopflrdt23}.
\end{proof}

%%%%%%%%%%%%%%%%%%%%%%%%%%%%%%%%%%%%%%%%%%%%%%%%%%%%%%%%%%%%%%%%%%%%%%%%%%%%%%%%%%%%%%%%%%%%%%%%%%%%%%%%%%%%%%%%%%%%%%%%%
%%%%%%%%%%%%%%%%%%%%%%%%%%%%%%%%%%%%%%%%%%%%%%%%%%%%%%%%%%%%%%%%%%%%%%%%%%%%%%%%%%%%%%%%%%%%%%%%%%%%%%%%%%%%%%%%%%%%%%%%%
%%%%%%%%%%%%%%%%%%%%%%%%%%%%%%%%%%%%%%%%%%%%%%%%%%%%%%%%%%%%%%%%%%%%%%%%%%%%%%%%%%%%%%%%%%%%%%%%%%%%%%%%%%%%%%%%%%%%%%%%%
%%%%%%%%%%%%%%%%%%%%%%%%%%%%%%%%%%%%%%%%%%%%%%%%%%%%%%%%%%%%%%%%%%%%%%%%%%%%%%%%%%%%%%%%%%%%%%%%%%%%%%%%%%%%%%%%%%%%%%%%%
%%%%%%%%%%%%%%%%%%%%%%%%%%%%%%%%%%%%%%%%%%%%%%%%%%%%%%%%%%%%%%%%%%%%%%%%%%%%%%%%%%%%%%%%%%%%%%%%%%%%%%%%%%%%%%%%%%%%%%%%%
\subsection{Numerical evaluations}
\label{sec:nuemrical}
%%%%%%%%%%%%%%%%%%%%%%%%%%%%%%%%%%%%%%%%%%%%%%%%%%%%%%%%%%%%%%%%%%%%%%%%%%%%%%%%%%%%%%%%%%%%%%%%%%%%%%%%%%%%%%%%%%%%%%%%%
%%%%%%%%%%%%%%%%%%%%%%%%%%%%%%%%%%%%%%%%%%%%%%%%%%%%%%%%%%%%%%%%%%%%%%%%%%%%%%%%%%%%%%%%%%%%%%%%%%%%%%%%%%%%%%%%%%%%%%%%%
%%%%%%%%%%%%%%%%%%%%%%%%%%%%%%%%%%%%%%%%%%%%%%%%%%%%%%%%%%%%%%%%%%%%%%%%%%%%%%%%%%%%%%%%%%%%%%%%%%%%%%%%%%%%%%%%%%%%%%%%%
%%%%%%%%%%%%%%%%%%%%%%%%%%%%%%%%%%%%%%%%%%%%%%%%%%%%%%%%%%%%%%%%%%%%%%%%%%%%%%%%%%%%%%%%%%%%%%%%%%%%%%%%%%%%%%%%%%%%%%%%%
%%%%%%%%%%%%%%%%%%%%%%%%%%%%%%%%%%%%%%%%%%%%%%%%%%%%%%%%%%%%%%%%%%%%%%%%%%%%%%%%%%%%%%%%%%%%%%%%%%%%%%%%%%%%%%%%%%%%%%%%%

As stated earlier, the results of Theorem \ref{thme:negthmprac1} become operational if one can conduct the underlying numerical evaluations. All technical ingredients for such evaluations are present in the theorem itself. We start the evaluations with $r=1$ and proceed by incrementally increasing $r$. Proceeding in such a way enables one to systematically follow  progressing of the entire lifting machinery. Moreover, it allows us to connect to some to the known results and show how they can be deduced as special cases of the generic mechanism presented here. To enable concrete numerical values, the evaluations are, on occasion, specialized  to particular values of $\kappa$. Also, several analytical closed form results can be obtained along the way that make the overall evaluation process somewhat easier. We state those below as well.

%%%%%%%%%%%%%%%%%%%%%%%%%%%%%%%%%%%%%%%%%%%%%%%%%%%%%%%%%%%%%%%%%%%%%%%%%%%%%%%%%%%%%%%%%%%%%%%%%%%%%%%%%%%%%%%%%%%%%%%%%
%%%%%%%%%%%%%%%%%%%%%%%%%%%%%%%%%%%%%%%%%%%%%%%%%%%%%%%%%%%%%%%%%%%%%%%%%%%%%%%%%%%%%%%%%%%%%%%%%%%%%%%%%%%%%%%%%%%%%%%%%
%%%%%%%%%%%%%%%%%%%%%%%%%%%%%%%%%%%%%%%%%%%%%%%%%%%%%%%%%%%%%%%%%%%%%%%%%%%%%%%%%%%%%%%%%%%%%%%%%%%%%%%%%%%%%%%%%%%%%%%%%
%%%%%%%%%%%%%%%%%%%%%%%%%%%%%%%%%%%%%%%%%%%%%%%%%%%%%%%%%%%%%%%%%%%%%%%%%%%%%%%%%%%%%%%%%%%%%%%%%%%%%%%%%%%%%%%%%%%%%%%%%
%%%%%%%%%%%%%%%%%%%%%%%%%%%%%%%%%%%%%%%%%%%%%%%%%%%%%%%%%%%%%%%%%%%%%%%%%%%%%%%%%%%%%%%%%%%%%%%%%%%%%%%%%%%%%%%%%%%%%%%%%
\subsubsection{$r=1$ -- first level of lifting}
\label{sec:firstlev}
%%%%%%%%%%%%%%%%%%%%%%%%%%%%%%%%%%%%%%%%%%%%%%%%%%%%%%%%%%%%%%%%%%%%%%%%%%%%%%%%%%%%%%%%%%%%%%%%%%%%%%%%%%%%%%%%%%%%%%%%%
%%%%%%%%%%%%%%%%%%%%%%%%%%%%%%%%%%%%%%%%%%%%%%%%%%%%%%%%%%%%%%%%%%%%%%%%%%%%%%%%%%%%%%%%%%%%%%%%%%%%%%%%%%%%%%%%%%%%%%%%%
%%%%%%%%%%%%%%%%%%%%%%%%%%%%%%%%%%%%%%%%%%%%%%%%%%%%%%%%%%%%%%%%%%%%%%%%%%%%%%%%%%%%%%%%%%%%%%%%%%%%%%%%%%%%%%%%%%%%%%%%%
%%%%%%%%%%%%%%%%%%%%%%%%%%%%%%%%%%%%%%%%%%%%%%%%%%%%%%%%%%%%%%%%%%%%%%%%%%%%%%%%%%%%%%%%%%%%%%%%%%%%%%%%%%%%%%%%%%%%%%%%%
%%%%%%%%%%%%%%%%%%%%%%%%%%%%%%%%%%%%%%%%%%%%%%%%%%%%%%%%%%%%%%%%%%%%%%%%%%%%%%%%%%%%%%%%%%%%%%%%%%%%%%%%%%%%%%%%%%%%%%%%%

For the first level, we have $r=1$ and $\hat{\p}_1\rightarrow 1$ and $\hat{\q}_1\rightarrow 1$ which, together with $\hat{\p}_{r+1}=\hat{\p}_{2}=\hat{\q}_{r+1}=\hat{\q}_{2}=0$, and $\hat{\c}_{2}\rightarrow 0$, gives
\begin{align}\label{eq:negprac19}
    \bar{\psi}_{rd}(\hat{\p},\hat{\q},\hat{\c},\gamma_{sq},\gamma_{sq}^{(p)},1,1,-1)   & =   \frac{1}{2}
\c_2
-\gamma_{sq}^{(p)}  - \frac{1}{\c_2}\log\lp \mE_{{\mathcal U}_2} e^{\c_2\frac{\lp\sqrt{1-0}\h_1^{(2)}\rp^2}{4\gamma_{sq}^{(p)}}}\rp
\nonumber \\
&   +\gamma_{sq}
- \alpha\frac{1}{\c_2}\log\lp \mE_{{\mathcal U}_2} e^{-\c_2\frac{\max(\kappa+\sqrt{1-0}\u_1^{(2,2)},0)^2}{4\gamma_{sq}}}\rp \nonumber \\
& \rightarrow
-\gamma_{sq}^{(p)}   - \frac{1}{\c_2}\log\lp 1+ \mE_{{\mathcal U}_2} \c_2\frac{\lp\sqrt{1-0}\h_1^{(2)}\rp^2}{4\gamma_{sq}^{(p)}}
\rp  \nonumber \\
& \qquad +\gamma_{sq}- \alpha\frac{1}{\c_2}\log\lp 1- \mE_{{\mathcal U}_2} \c_2\frac{\max(\kappa+\sqrt{1-0}\u_1^{(2,2)},0)^2}{4\gamma_{sq}}\rp \nonumber \\
& \rightarrow
 -\gamma_{sq}^{(p)}  - \frac{1}{\c_2}\log\lp 1+ \c_2\frac{1}{4\gamma_{sq}^{(p)}}\rp  \nonumber \\
& \qquad +\gamma_{sq} - \alpha\frac{1}{\c_2}\log\lp 1- \frac{\c_2}{4\gamma_{sq}} \mE_{{\mathcal U}_2} \max(\kappa+\sqrt{1-0}\u_1^{(2,2)},0)^2 \rp \nonumber \\
& \rightarrow
   -\gamma_{sq}^{(p)}-\frac{1}{4\gamma_{sq}^{(p)}} +\gamma_{sq}
+  \frac{\alpha}{4\gamma_{sq}}\mE_{{\mathcal U}_2} \max(\kappa+\sqrt{1-0}\u_1^{(2,2)},0)^2.
  \end{align}
One then easily finds $\gamma_{sq}^{(p)}=\frac{1}{2}$ and $\hat{\gamma}_{sq}=\frac{\sqrt{\alpha}}{2}\sqrt{\mE_{{\mathcal U}_2} \max(\kappa+\sqrt{1-0}\u_1^{(2,2)},0)^2}$ and
\begin{align}\label{eq:negprac20}
 - f_{sq}^{(1)}(\infty)=\bar{\psi}_{rd}(\hat{\p},\hat{\q},\hat{\c},\hat{\gamma}_{sq},\hat{\gamma}_{sq}^{(p)},1,1,-1)   & =
  -1+\sqrt{\alpha}\sqrt{\mE_{{\mathcal U}_2} \max(\kappa+\u_1^{(2,2)},0)^2}.
  \end{align}
To obtain the critical $\alpha_c^{(1)}$ as a function of $\kappa$, we rely on condition $f_{sq}^{(1)}(\infty)=0$, which gives
\begin{equation}\label{eq:negprac20a0}
a_c^{(1)}(\kappa)
=  \frac{1}{\mE_{{\mathcal U}_2} \max(\kappa+\u_1^{(2,2)},0)^2}
=  \frac{1}{\lp \frac{\kappa e^{-\frac{\kappa^2}{2}}}{\sqrt{2\pi}} + \frac{(\kappa^2+1)\erfc\lp -\frac{\kappa}{\sqrt{2}} \rp}{2}  \rp}.
  \end{equation}
To get concrete values we specialize to $\kappa=-1.5$ and find
\begin{equation}\label{eq:negprac21}
\hspace{-2in}(\mbox{\bl{\textbf{first level:}}}) \qquad \qquad  a_c^{(1)}(-1.5) =
  \frac{1}{\mE_{{\mathcal U}_2} \max(-1.5+\u_1^{(2,2)},0)^2}
\rightarrow  \bl{\mathbf{43.77}}.
  \end{equation}

%%%%%%%%%%%%%%%%%%%%%%%%%%%%%%%%%%%%%%%%%%%%%%%%%%%%%%%%%%%%%%%%%%%%%%%%%%%%%%%%%%%%%%%%%%%%%%%%%%%%%%%%%%%%%%%%%%%%%%%%%
%%%%%%%%%%%%%%%%%%%%%%%%%%%%%%%%%%%%%%%%%%%%%%%%%%%%%%%%%%%%%%%%%%%%%%%%%%%%%%%%%%%%%%%%%%%%%%%%%%%%%%%%%%%%%%%%%%%%%%%%%
%%%%%%%%%%%%%%%%%%%%%%%%%%%%%%%%%%%%%%%%%%%%%%%%%%%%%%%%%%%%%%%%%%%%%%%%%%%%%%%%%%%%%%%%%%%%%%%%%%%%%%%%%%%%%%%%%%%%%%%%%
%%%%%%%%%%%%%%%%%%%%%%%%%%%%%%%%%%%%%%%%%%%%%%%%%%%%%%%%%%%%%%%%%%%%%%%%%%%%%%%%%%%%%%%%%%%%%%%%%%%%%%%%%%%%%%%%%%%%%%%%%
%%%%%%%%%%%%%%%%%%%%%%%%%%%%%%%%%%%%%%%%%%%%%%%%%%%%%%%%%%%%%%%%%%%%%%%%%%%%%%%%%%%%%%%%%%%%%%%%%%%%%%%%%%%%%%%%%%%%%%%%%
\subsubsection{$r=2$ -- second level of lifting}
\label{sec:secondlev}
%%%%%%%%%%%%%%%%%%%%%%%%%%%%%%%%%%%%%%%%%%%%%%%%%%%%%%%%%%%%%%%%%%%%%%%%%%%%%%%%%%%%%%%%%%%%%%%%%%%%%%%%%%%%%%%%%%%%%%%%%
%%%%%%%%%%%%%%%%%%%%%%%%%%%%%%%%%%%%%%%%%%%%%%%%%%%%%%%%%%%%%%%%%%%%%%%%%%%%%%%%%%%%%%%%%%%%%%%%%%%%%%%%%%%%%%%%%%%%%%%%%
%%%%%%%%%%%%%%%%%%%%%%%%%%%%%%%%%%%%%%%%%%%%%%%%%%%%%%%%%%%%%%%%%%%%%%%%%%%%%%%%%%%%%%%%%%%%%%%%%%%%%%%%%%%%%%%%%%%%%%%%%
%%%%%%%%%%%%%%%%%%%%%%%%%%%%%%%%%%%%%%%%%%%%%%%%%%%%%%%%%%%%%%%%%%%%%%%%%%%%%%%%%%%%%%%%%%%%%%%%%%%%%%%%%%%%%%%%%%%%%%%%%
%%%%%%%%%%%%%%%%%%%%%%%%%%%%%%%%%%%%%%%%%%%%%%%%%%%%%%%%%%%%%%%%%%%%%%%%%%%%%%%%%%%%%%%%%%%%%%%%%%%%%%%%%%%%%%%%%%%%%%%%%

We split the second level of lifting into two separate parts: (i) \emph{partial} second level of lifting; and (ii) \emph{full} second level of lifting.

%%%%%%%%%%%%%%%%%%%%%%%%%%%%%%%%%%%%%%%%%%%%%%%%%%%%%%%%%%%%%%%%%%%%%%%%%%%%%%%%%%%%%%%%%%%%%%%%%%%%%%%%%%%%%%%%%%%%%%%%%
%%%%%%%%%%%%%%%%%%%%%%%%%%%%%%%%%%%%%%%%%%%%%%%%%%%%%%%%%%%%%%%%%%%%%%%%%%%%%%%%%%%%%%%%%%%%%%%%%%%%%%%%%%%%%%%%%%%%%%%%%
%%%%%%%%%%%%%%%%%%%%%%%%%%%%%%%%%%%%%%%%%%%%%%%%%%%%%%%%%%%%%%%%%%%%%%%%%%%%%%%%%%%%%%%%%%%%%%%%%%%%%%%%%%%%%%%%%%%%%%%%%
%%%%%%%%%%%%%%%%%%%%%%%%%%%%%%%%%%%%%%%%%%%%%%%%%%%%%%%%%%%%%%%%%%%%%%%%%%%%%%%%%%%%%%%%%%%%%%%%%%%%%%%%%%%%%%%%%%%%%%%%%
%%%%%%%%%%%%%%%%%%%%%%%%%%%%%%%%%%%%%%%%%%%%%%%%%%%%%%%%%%%%%%%%%%%%%%%%%%%%%%%%%%%%%%%%%%%%%%%%%%%%%%%%%%%%%%%%%%%%%%%%%
\subsubsubsection{Partial second level of lifting}
\label{sec:secondlevpar}
%%%%%%%%%%%%%%%%%%%%%%%%%%%%%%%%%%%%%%%%%%%%%%%%%%%%%%%%%%%%%%%%%%%%%%%%%%%%%%%%%%%%%%%%%%%%%%%%%%%%%%%%%%%%%%%%%%%%%%%%%
%%%%%%%%%%%%%%%%%%%%%%%%%%%%%%%%%%%%%%%%%%%%%%%%%%%%%%%%%%%%%%%%%%%%%%%%%%%%%%%%%%%%%%%%%%%%%%%%%%%%%%%%%%%%%%%%%%%%%%%%%
%%%%%%%%%%%%%%%%%%%%%%%%%%%%%%%%%%%%%%%%%%%%%%%%%%%%%%%%%%%%%%%%%%%%%%%%%%%%%%%%%%%%%%%%%%%%%%%%%%%%%%%%%%%%%%%%%%%%%%%%%
%%%%%%%%%%%%%%%%%%%%%%%%%%%%%%%%%%%%%%%%%%%%%%%%%%%%%%%%%%%%%%%%%%%%%%%%%%%%%%%%%%%%%%%%%%%%%%%%%%%%%%%%%%%%%%%%%%%%%%%%%
%%%%%%%%%%%%%%%%%%%%%%%%%%%%%%%%%%%%%%%%%%%%%%%%%%%%%%%%%%%%%%%%%%%%%%%%%%%%%%%%%%%%%%%%%%%%%%%%%%%%%%%%%%%%%%%%%%%%%%%%%

When $r=2$ and the partial lifting is considered, then (similarly to the first level)  $\hat{\p}_1\rightarrow 1$ and $\hat{\q}_1\rightarrow 1$, $\hat{\p}_{2}=\hat{\q}_{2}=0$, and $\hat{\p}_{r+1}=\hat{\p}_{3}=\hat{\q}_{r+1}=\hat{\q}_{3}=0$ but in general  $\hat{\c}_{2}\neq 0$. As above, one again has
\begin{align}\label{eq:negprac22}
    \bar{\psi}_{rd}(\hat{\p},\hat{\q},\c,\gamma_{sq},\gamma_{sq}^{(p)},1,1,-1)   & =   \frac{1}{2}
\c_2
 -\gamma_{sq}^{(p)}  - \frac{1}{\c_2}\log\lp \mE_{{\mathcal U}_2} e^{\c_2\frac{\lp\sqrt{1-0}\h_1^{(2)}\rp^2}{4\gamma_{sq}^{(p)}}}\rp \nonumber \\
 &
\quad  + \gamma_{sq}
- \alpha\frac{1}{\c_2}\log\lp \mE_{{\mathcal U}_2} e^{-\c_2\frac{\max(\kappa+\sqrt{1-0}\u_1^{(2,2)},0)^2}{4\gamma_{sq}}}\rp \nonumber \\
& =   \frac{1}{2}
\c_2
      -\gamma_{sq}^{(p)}  +\frac{1}{2\c_2}\log\lp \frac{2\gamma_{sq}^{(p)}-\c_2}{2\gamma_{sq}^{(p)}}\rp \nonumber \\
      &\quad
  + \gamma_{sq}
- \alpha\frac{1}{\c_2}\log\lp \mE_{{\mathcal U}_2} e^{-\c_2\frac{\max(\kappa+\sqrt{1-0}\u_1^{(2,2)},0)^2}{4\gamma_{sq}}}\rp. \nonumber \\
    \end{align}
Solving the integrals gives
\begin{eqnarray}\label{eq:negprac22a0}
\bar{h} & = & -\kappa    \nonumber \\
\bar{B} & = & \frac{\c_2}{4\gamma_{sq}} \nonumber \\
\bar{C} & = & \kappa \nonumber \\
f_{(zd)}& = & \frac{e^{-\frac{\bar{B}\bar{C}^2}{2\bar{B} + 1}}}{2\sqrt{2\bar{B} + 1}}\erfc\lp\frac{\bar{h}}{\sqrt{4\bar{B} + 2}}\rp  \nonumber \\
f_{(zu)}& = & \frac{1}{2}\erfc\lp-\frac{\bar{h}}{\sqrt{2}}\rp,
   \end{eqnarray}
 and
\begin{equation}\label{eq:negprac22a1}
   \mE_{{\mathcal U}_2} e^{-\c_2\frac{\max(\kappa+\sqrt{1-0}\u_1^{(2,2)},0)^2}{4\gamma_{sq}}}=f_{(zd)} + f_{(zu)}.
   \end{equation}
Differentiation  (optimization) with respect to $\gamma_{sq}^{(p)}$, $\gamma_{sq}$, and $\c_2$ brings two different scenarios for concrete optimal parameter values that are distinguished based on the value of $\kappa$.

\noindent (\textbf{\emph{i}}) For $\kappa\geq \kappa_c\approx -0.622$, we find $\hat{\c}_2\rightarrow 0$, $\hat{\gamma}_{sq}^{(p)}=\frac{1}{2}$, and $\hat{\gamma}_{sq}=\frac{\sqrt{\alpha}}{2}\sqrt{\mE_{{\mathcal U}_2} \max(\kappa+\sqrt{1-0}\u_1^{(2,2)},0)^2}$. In other words, when $\kappa\geq \kappa_c\approx -0.622$,  one uncovers the first level of lifting with $a_c^{(2,p)}$ as in (\ref{eq:negprac21}), i.e., with
\begin{equation}\label{eq:negprac22a1a0}
a_c^{(2,p)} =
a_c^{(1)} =
  \frac{1}{\mE_{{\mathcal U}_2} \max(\kappa+\u_1^{(2,2)},0)^2}
  = \frac{1}{\lp \frac{\kappa e^{-\frac{\kappa^2}{2}}}{\sqrt{2\pi}} + \frac{(\kappa^2+1)\erfc\lp -\frac{\kappa}{\sqrt{2}} \rp}{2}  \rp}.
  \end{equation}

\noindent (\textbf{\emph{ii}}) For $\kappa\leq \kappa_c$, after computing the derivatives with respect to $\gamma_{sq}^{(p)}$, $\gamma_{sq}$, and $\c_2$ and equalling them to zero, we obtain for, say, $\kappa=-1.5$ \begin{equation}\label{eq:negprac23}
\hspace{-2in}(\mbox{\bl{\textbf{\emph{partial} second level:}}}) \qquad \qquad  a_c^{(2,p)}(-1.5)  \approx \bl{\mathbf{37.36}}.
  \end{equation}

%%%%%%%%%%%%%%%%%%%%%%%%%%%%%%%%%%%%%%%%%%%%%%%%%%%%%%%%%%%%%%%%%%%%%%%%%%%%%%%%%%%%%%%%%%%%%%%%%%%%%%%%%%%%%%%%%%%%%%%%%
%%%%%%%%%%%%%%%%%%%%%%%%%%%%%%%%%%%%%%%%%%%%%%%%%%%%%%%%%%%%%%%%%%%%%%%%%%%%%%%%%%%%%%%%%%%%%%%%%%%%%%%%%%%%%%%%%%%%%%%%%
%%%%%%%%%%%%%%%%%%%%%%%%%%%%%%%%%%%%%%%%%%%%%%%%%%%%%%%%%%%%%%%%%%%%%%%%%%%%%%%%%%%%%%%%%%%%%%%%%%%%%%%%%%%%%%%%%%%%%%%%%
%%%%%%%%%%%%%%%%%%%%%%%%%%%%%%%%%%%%%%%%%%%%%%%%%%%%%%%%%%%%%%%%%%%%%%%%%%%%%%%%%%%%%%%%%%%%%%%%%%%%%%%%%%%%%%%%%%%%%%%%%
%%%%%%%%%%%%%%%%%%%%%%%%%%%%%%%%%%%%%%%%%%%%%%%%%%%%%%%%%%%%%%%%%%%%%%%%%%%%%%%%%%%%%%%%%%%%%%%%%%%%%%%%%%%%%%%%%%%%%%%%%
\subsubsubsection{Full second level of lifting}
\label{sec:secondlevfull}
%%%%%%%%%%%%%%%%%%%%%%%%%%%%%%%%%%%%%%%%%%%%%%%%%%%%%%%%%%%%%%%%%%%%%%%%%%%%%%%%%%%%%%%%%%%%%%%%%%%%%%%%%%%%%%%%%%%%%%%%%
%%%%%%%%%%%%%%%%%%%%%%%%%%%%%%%%%%%%%%%%%%%%%%%%%%%%%%%%%%%%%%%%%%%%%%%%%%%%%%%%%%%%%%%%%%%%%%%%%%%%%%%%%%%%%%%%%%%%%%%%%
%%%%%%%%%%%%%%%%%%%%%%%%%%%%%%%%%%%%%%%%%%%%%%%%%%%%%%%%%%%%%%%%%%%%%%%%%%%%%%%%%%%%%%%%%%%%%%%%%%%%%%%%%%%%%%%%%%%%%%%%%
%%%%%%%%%%%%%%%%%%%%%%%%%%%%%%%%%%%%%%%%%%%%%%%%%%%%%%%%%%%%%%%%%%%%%%%%%%%%%%%%%%%%%%%%%%%%%%%%%%%%%%%%%%%%%%%%%%%%%%%%%
%%%%%%%%%%%%%%%%%%%%%%%%%%%%%%%%%%%%%%%%%%%%%%%%%%%%%%%%%%%%%%%%%%%%%%%%%%%%%%%%%%%%%%%%%%%%%%%%%%%%%%%%%%%%%%%%%%%%%%%%%

The setup presented above can also be utilized for the full lifting on the second level. However, one has to be careful since now (in addition to $\hat{\c}_{2}\neq 0$) one, in general,  also has $\p_2\neq0$ and $\q_2\neq0$. Analogously to (\ref{eq:negprac22}), we now write
\begin{eqnarray}\label{eq:negprac24}
    \bar{\psi}_{rd}(\p,\q,\c,\gamma_{sq},\gamma_{sq}^{(p)},1,1,-1)   & = &  \frac{1}{2}
(1-\p_2\q_2)\c_2
-  \gamma_{sq}^{(p)}  - \frac{1}{\c_2}\mE_{{\mathcal U}_3}\log\lp \mE_{{\mathcal U}_2} e^{\c_2\frac{\lp\sqrt{1-\q_2}\h_1^{(2)} +\sqrt{\q_2}\h_1^{(3)} \rp^2}{4 \gamma_{sq}^{(p)}}}\rp \nonumber \\
& &   + \gamma_{sq}
 -\alpha\frac{1}{\c_2}\mE_{{\mathcal U}_3} \log\lp \mE_{{\mathcal U}_2} e^{-\c_2\frac{\max(\sqrt{1-\p_2}\u_1^{(2,2)}+\sqrt{\p_2}\u_1^{(2,3)},0)^2}{4\gamma_{sq}}}\rp \nonumber \\
 & = &  \frac{1}{2}
(1-\p_2\q_2)\c_2
 -  \gamma_{sq}^{(p)}
-\Bigg(\Bigg. -\frac{1}{2\c_2} \log \lp \frac{2\gamma_{sq}-\c_2(1-\q_2)}{2\gamma_{sq}} \rp  \nonumber \\
 & & +  \frac{\q_2}{2(2\gamma_{sq}-\c_2(1-\q_2))}   \Bigg.\Bigg)
 \nonumber \\
& &   + \gamma_{sq}
 -\alpha\frac{1}{\c_2}\mE_{{\mathcal U}_3} \log\lp \mE_{{\mathcal U}_2} e^{-\c_2\frac{\max(\sqrt{1-\p_2}\u_1^{(2,2)}+\sqrt{\p_2}\u_1^{(2,3)},0)^2}{4\gamma_{sq}}}\rp.
    \end{eqnarray}
After solving the remaining integrals, we also have
\begin{eqnarray}\label{eq:negprac24a2}
\hat{h} & = &  -\frac{\sqrt{\p_2}\u_1^{(2,3)}+\kappa}{\sqrt{1-\p_2}}    \nonumber \\
\hat{B} & = & \frac{\c_2}{4\gamma_{sq}} %\rightarrow \infty
\nonumber \\
\hat{C} & = & \sqrt{\p_2}\u_1^{(2,3)}+\kappa \nonumber \\
f_{(zd)}^{(2,f)}& = & \frac{e^{-\frac{\hat{B}\hat{C}^2}{2(1-\p_2)\hat{B} + 1}}}{2\sqrt{2(1-\p_2)\hat{B} + 1}}
\erfc\lp\frac{\hat{h}}{\sqrt{4(1-\p_2)\hat{B} + 2}}\rp
%\rightarrow 0
\nonumber \\
f_{(zu)}^{(2,f)}& = & \frac{1}{2}\erfc\lp-\frac{\hat{h}}{\sqrt{2}}\rp,  %\rightarrow \frac{1}{2}\erfc\lp \frac{\kappa}{\sqrt{2}}\rp
\nonumber \\
f_{(zt)}^{(2,f)}& = & f_{(zd)}^{(2,f)}+f_{(zu)}^{(2,f)}.
   \end{eqnarray}
and
\begin{eqnarray}\label{eq:negprac24a3}
   \mE_{{\mathcal U}_3} \log\lp \mE_{{\mathcal U}_2} e^{-\c_2\frac{\max(\sqrt{1-\p_2}\u_1^{(2,2)}+\sqrt{\p_2}\u_1^{(2,3)},0)^2}{4\gamma_{sq}}}\rp
=   \mE_{{\mathcal U}_3} \log\lp f_{(zt)}^{(2,f)} \rp.
    \end{eqnarray}
%    hc=-(b1*zvec+nu)/a1;
%    B=c3/4/gama;
%    C=b1*zvec+nu;
%    fzdown=(exp(-(B *C.^2)./(2 *a1^2*B + 1)).*( erfc((hc)./sqrt(4 *a1^2* B + 2))))./(2* sqrt(2 *a1^2 *B + 1));
%    fzup=.5*erfc(-hc/sqrt(2));
%    fz=fzdown+fzup;
%    Ibin=sum(log(fz).*exp(-zvec.^2/2))*dz/sqrt(2*pi);
One now needs to compute \emph{five} derivatives with respect to $\q_2$, $\p_2$, $\c_2$, $\gamma_{sq}$,, and $\gamma_{sq}^{(p)}$. We systematically compute each of them.

\vspace{.1in}
\noindent \underline{\red{\textbf{(i) $\q_2$ -- derivative:}}} We start by writing
\begin{eqnarray}\label{eq:2levder1}
   \frac{d\bar{\psi}_{rd}(\p,\q,\c,\gamma_{sq},\gamma_{sq}^{(p)},1,1,-1) }{d\q_2}
 & = &  -\frac{1}{2}
\p_2\c_2
-\Bigg(\Bigg.  -\frac{1}{2((2\gamma_{sq}^{(p)}-\c_2(1-\q_2)))}+\frac{1}{2((2\gamma_{sq}^{(p)}-\c_2(1-\q_2)))}\nonumber \\
& & -\frac{\q_2}{2(2\gamma_{sq}^{(p)}-\c_2(1-\q_2))^2}\c_2
  \Bigg.\Bigg)
 \nonumber \\
 & = &  -\frac{1}{2}
\p_2\c_2
+\frac{\q_2}{2(2\gamma_{sq}^{(p)}-\c_2(1-\q_2))^2}\c_2 \nonumber \\
 & = &  \c_2\lp -\frac{1}{2}
\p_2
+\frac{\q_2}{2(2\gamma_{sq}^{(p)}-\c_2(1-\q_2))^2}\rp.
     \end{eqnarray}

\vspace{.1in}
\noindent \underline{\red{\textbf{(ii) $\p_2$ -- derivative:}}} As above, we start by writing
\begin{eqnarray}\label{eq:2levder2}
   \frac{d\bar{\psi}_{rd}(\p,\q,\c,\gamma_{sq},\gamma_{sq}^{(p)},1,1,-1) }{d\p_2}
 & = &  -\frac{1}{2}
\q_2\c_2
- \frac{\alpha}{\c_2}\frac{d\lp \mE_{{\mathcal U}_3} \log\lp f_{(zt)}^{(2,f)} \rp \rp}{d\p_2} \nonumber \\
 & = &  -\frac{1}{2}
\q_2\c_2
- \frac{\alpha}{\c_2}  \mE_{{\mathcal U}_3} \lp \frac{1}{f_{(zt)}^{(2,f)}} \frac{d\lp f_{(zt)}^{(2,f)}\rp}{d\p_2} \rp.
     \end{eqnarray}
From (\ref{eq:negprac24a2}), we then have
\begin{eqnarray}\label{eq:2levder3}
 \frac{df_{(zt)}^{(2,f)}}{d\p_2}& = &  \frac{df_{(zd)}^{(2,f)}}{d\p_2}+ \frac{df_{(zu)}^{(2,f)}}{d\p_2}.
   \end{eqnarray}
Moreover, we also have
\begin{eqnarray}\label{eq:2levder4}
 \frac{df_{(zu)}^{(2,f)}}{d\p_2} =\frac{e^{-\frac{\hat{h}^2}{2}}}{\sqrt{2\pi}}\frac{d\hat{h}}{d\p_2},
   \end{eqnarray}
and
\begin{eqnarray}\label{eq:2levder5}
\frac{d\hat{h}}{d\p_2}
=-\frac{\u_1^{(2,3)}}{2\sqrt{\p_2}\sqrt{1-\p_2}}-\frac{(\sqrt{\p_2}\u_1^{(2,3)}+\kappa)}{2\sqrt{1-\p_2}^3}.
   \end{eqnarray}
A combination of  (\ref{eq:2levder4}) and (\ref{eq:2levder5}) gives
\begin{eqnarray}\label{eq:2levder6}
 \frac{df_{(zu)}^{(2,f)}}{d\p_2}
 =\frac{e^{-\frac{\hat{h}^2}{2}}}{\sqrt{2\pi}}\frac{d\hat{h}}{d\p_2}
=\frac{e^{-\frac{\hat{h}^2}{2}}}{\sqrt{2\pi}}
\lp-\frac{\u_1^{(2,3)}}{2\sqrt{\p_2}\sqrt{1-\p_2}}-\frac{(\sqrt{\p_2}\u_1^{(2,3)}+\kappa)}{2\sqrt{1-\p_2}^3}\rp.
   \end{eqnarray}
   After observing
\begin{eqnarray}\label{eq:2levder7}
 \frac{d\hat{C}}{d\p_2} =\frac{\u_1^{(2,3)}}{2\sqrt{\p_2}},
   \end{eqnarray}
 we can further write
\begin{eqnarray}\label{eq:2levder8}
 \frac{df_{(zd)}^{(2,f)}}{d\p_2} =f_{(d\p)}^{(1)}+f_{(d\p)}^{(2)}+f_{(d\p)}^{(3)},
   \end{eqnarray}
 where
\begin{eqnarray}\label{eq:2levder9}
f_{(d\p)}^{(1)}=\lp-\frac{\hat{B} \hat{C}\u_1^{(2,3)}}{\sqrt{\p_2}(2(1-\p_2)\hat{B} + 1)}-\frac{2\hat{B}^2\hat{C}^2}{(2(1-\p_2)\hat{B} + 1).^2}\rp
e^{-\frac{\hat{B}\hat{C}^2}{2(1-\p_2)\hat{B} + 1}}
\frac{\erfc\lp\frac{\hat{h}}{\sqrt{4(1-\p_2)\hat{B} + 2}}\rp}{2\sqrt{2(1-\p_2)\hat{B} + 1}},
   \end{eqnarray}
   and
\begin{equation}\label{eq:2levder10}
f_{(d\p)}^{(2)}=\frac{e^{-\frac{\hat{B}\hat{C}^2}{2(1-\p_2)\hat{B} + 1}}}{2\sqrt{2(1-\p_2)\hat{B} + 1}}
\lp -\frac{2}{\sqrt{\pi}}\lp
\frac{1}{\sqrt{4(1-\p_2)\hat{B} + 2}}\frac{d\hat{h}}{d\p_2}
+\frac{2\hat{B}\hat{h}}{\sqrt{4(1-\p_2)\hat{B} + 2}^3}\rp
e^{-\lp\frac{\hat{h}}{\sqrt{4(1-\p_2)\hat{B} + 2}}\rp^2}\rp,
   \end{equation}
and
\begin{eqnarray}\label{eq:2levder11}
f_{(d\p)}^{(3)}=\frac{\hat{B}e^{-\frac{\hat{B}\hat{C}^2}{2(1-\p_2)\hat{B} + 1}}}{2\sqrt{2(1-\p_2)\hat{B} + 1}^3}
\erfc\lp\frac{\hat{h}}{\sqrt{4(1-\p_2)\hat{B} + 2}}\rp.
   \end{eqnarray}
A combination of (\ref{eq:negprac24a2}), (\ref{eq:2levder2}), (\ref{eq:2levder3}), (\ref{eq:2levder6}), and (\ref{eq:2levder8})-(\ref{eq:2levder11}) is then sufficient to determine $\p_2$--derivative.

\vspace{.1in}

\noindent \underline{\red{\textbf{(iii) $\c_2$ -- derivative:}}} We again start by writing
\begin{eqnarray}\label{eq:2levder12}
   \frac{d\bar{\psi}_{rd}(\p,\q,\c,\gamma_{sq},\gamma_{sq}^{(p)},1,1,-1) }{d\c_2}
 & = &  \frac{1}{2}(1-
\p_2\q_2)
 -\frac{1}{2\c_2^2} \log \lp \frac{2\gamma_{sq}^{(p)}-\c_2(1-\q_2)}{2\gamma_{sq}^{(p)}} \rp \nonumber \\
 & &
 -\frac{1-\q_2}{2\c_2(2\gamma_{sq}^{(p)}-\c_2(1-\q_2))}    -  \frac{\q_2(1-\q_2)}{2(2\gamma_{sq}^{(p)}-\c_2(1-\q_2))^2}
 \nonumber \\
& &  + \frac{\alpha}{\c_2^2}  \mE_{{\mathcal U}_3} \log\lp f_{(zt)}^{(2,f)} \rp
- \frac{\alpha}{\c_2}  \mE_{{\mathcal U}_3} \lp \frac{1}{f_{(zt)}^{(2,f)}} \frac{d\lp f_{(zt)}^{(2,f)}\rp}{d\c_2} \rp.
      \end{eqnarray}
From (\ref{eq:negprac24a2}), we then have
\begin{eqnarray}\label{eq:2levder13}
 \frac{df_{(zt)}^{(2,f)}}{d\c_2}& = &  \frac{df_{(zd)}^{(2,f)}}{d\c_2}+ \frac{df_{(zu)}^{(2,f)}}{d\c_2}=\frac{df_{(zd)}^{(2,f)}}{d\c_2},
   \end{eqnarray}
where we utilized the fact that
\begin{eqnarray}\label{eq:2levder14}
 \frac{df_{(zu)}^{(2,f)}}{d\c_2} =\frac{e^{-\frac{\hat{h}^2}{2}}}{\sqrt{2\pi}}\frac{d\hat{h}}{d\c_2}=0.
   \end{eqnarray}
    After observing
\begin{eqnarray}\label{eq:2levder17}
 \frac{d\hat{B}}{d\c_2} =\frac{1}{4\gamma_{sq}}\quad \mbox{and} \quad  \frac{d\hat{C}}{d\c_2} =0,
   \end{eqnarray}
 we can further write
\begin{eqnarray}\label{eq:2levder18}
 \frac{df_{(zd)}^{(2,f)}}{d\c_2} =f_{(d\c)}^{(1)}+f_{(d\c)}^{(2)}+f_{(d\c)}^{(3)},
   \end{eqnarray}
 where
\begin{eqnarray}\label{eq:2levder19}
f_{(d\c)}^{(1)}=\lp-\frac{\hat{C}^2}{4\gamma_{sq}(2(1-\p_2)\hat{B} + 1)}+\frac{(1-\p_2)\hat{B}\hat{C}^2}{2\gamma_{sq}(2(1-\p_2)\hat{B} + 1).^2}\rp
e^{-\frac{\hat{B}\hat{C}^2}{2(1-\p_2)\hat{B} + 1}}
\frac{\erfc\lp\frac{\hat{h}}{\sqrt{4(1-\p_2)\hat{B} + 2}}\rp}{2\sqrt{2(1-\p_2)\hat{B} + 1}},
   \end{eqnarray}
   and
\begin{equation}\label{eq:2levder20}
f_{(d\c)}^{(2)}=\frac{e^{-\frac{\hat{B}\hat{C}^2}{2(1-\p_2)\hat{B} + 1}}}{2\sqrt{2(1-\p_2)\hat{B} + 1}}
\lp -\frac{2}{\sqrt{\pi}}\lp
 -\frac{(1-\p_2)\hat{h}}{2\gamma_{sq}\sqrt{4(1-\p_2)\hat{B} + 2}^3}\rp
e^{-\lp\frac{\hat{h}}{\sqrt{4(1-\p_2)\hat{B} + 2}}\rp^2}\rp,
   \end{equation}
and
\begin{eqnarray}\label{eq:2levder21}
f_{(d\c)}^{(3)}=-\frac{(1-\p_2)e^{-\frac{\hat{B}\hat{C}^2}{2(1-\p_2)\hat{B} + 1}}}{8\gamma_{sq}\sqrt{2(1-\p_2)\hat{B} + 1}^3}
\erfc\lp\frac{\hat{h}}{\sqrt{4(1-\p_2)\hat{B} + 2}}\rp,
   \end{eqnarray}
A combination of (\ref{eq:negprac24a2}), (\ref{eq:2levder12}), (\ref{eq:2levder13}), and (\ref{eq:2levder18})-(\ref{eq:2levder21}) is then sufficient to determine $\c_2$--derivative.

\vspace{.1in}

\noindent \underline{\red{\textbf{(iv) $\gamma_{sq}^{(p)}$ -- derivative:}}} We easily find
\begin{eqnarray}\label{eq:2levder21a0}
   \frac{d\bar{\psi}_{rd}(\p,\q,\c,\gamma_{sq},\gamma_{sq}^{(p)},1,1,-1) }{d\gamma_{sq}^{(p)}}
  &  =  &   -1-\lp -\frac{1}{\c_2(2\gamma_{sq}^{(p)}-\c_2(1-\q_2))}+\frac{1}{2\c_2\gamma_{sq}^{(p)}}-\frac{\q_2}{(2\gamma_{sq}^{(p)}-\c_2(1-\q_2))^2}\rp \nonumber \\
   & = &    -1-\lp -\frac{1-\q_2}{2\gamma_{sq}^{(p)}(2\gamma_{sq}^{(p)}-\c_2(1-\q_2))}-\frac{\q_2}{(2\gamma_{sq}^{(p)}-\c_2(1-\q_2))^2}\rp. \nonumber \\
\end{eqnarray}

\vspace{.1in}

\noindent \underline{\red{\textbf{(v) $\gamma_{sq}$ -- derivative:}}} We first write
\begin{equation}\label{eq:2levder22}
   \frac{d\bar{\psi}_{rd}(\p,\q,\c,\gamma_{sq},\gamma_{sq}^{(p)},1,1,-1) }{d\gamma_{sq}}
=
 1
- \frac{\alpha}{\c_2}  \mE_{{\mathcal U}_3} \lp \frac{1}{f_{(zt)}^{(2,f)}} \frac{d\lp f_{(zt)}^{(2,f)}\rp}{d\gamma_{sq}} \rp.
 \end{equation}
Relying on (\ref{eq:negprac24a2}), we also have
\begin{eqnarray}\label{eq:2levder23}
 \frac{df_{(zt)}^{(2,f)}}{d\gamma_{sq}}& = &  \frac{df_{(zd)}^{(2,f)}}{d\gamma_{sq}}+ \frac{df_{(zu)}^{(2,f)}}{d\gamma_{sq}}
 =\frac{df_{(zd)}^{(2,f)}}{d\gamma_{sq}},
   \end{eqnarray}
where we utilized
\begin{eqnarray}\label{eq:2levder24}
 \frac{df_{(zu)}^{(2,f)}}{d\gamma_{sq}} =\frac{e^{-\frac{\hat{h}^2}{2}}}{\sqrt{2\pi}}\frac{d\hat{h}}{d\gamma_{sq}}=0.
   \end{eqnarray}
    After observing
\begin{eqnarray}\label{eq:2levder27}
 \frac{d\hat{h}}{d\gamma_{sq}} =\frac{d\hat{C}}{d\gamma_{sq}} =0 \quad \mbox{and}\quad
  \frac{d\hat{B}}{d\gamma_{sq}} =-\frac{\c_2}{4\gamma_{sq}^2},
   \end{eqnarray}
 we can further write
\begin{eqnarray}\label{eq:2levder28}
 \frac{df_{(zd)}^{(2,f)}}{d\gamma_{sq}} =f_{(d\gamma)}^{(1)}+f_{(d\gamma)}^{(2)}+f_{(d\gamma)}^{(3)},
   \end{eqnarray}
 where
\begin{eqnarray}\label{eq:2levder29}
f_{(d\gamma)}^{(1)}=\lp \frac{\c_2\hat{C}^2}{4\gamma_{sq}^2(2(1-\p_2)\hat{B} + 1)}-\frac{\c_2(1-\p_2)\hat{B}\hat{C}^2}{2\gamma_{sq}^2(2(1-\p_2)\hat{B} + 1).^2}\rp
e^{-\frac{\hat{B}\hat{C}^2}{2(1-\p_2)\hat{B} + 1}}
\frac{\erfc\lp\frac{\hat{h}}{\sqrt{4(1-\p_2)\hat{B} + 2}}\rp}{2\sqrt{2(1-\p_2)\hat{B} + 1}},
   \end{eqnarray}
   and
\begin{equation}\label{eq:2levder30}
f_{(d\gamma)}^{(2)}=\frac{e^{-\frac{\hat{B}\hat{C}^2}{2(1-\p_2)\hat{B} + 1}}}{2\sqrt{2(1-\p_2)\hat{B} + 1}}
\lp -\frac{2}{\sqrt{\pi}}\lp
 \frac{\c_2(1-\p_2)\hat{h}}{2\gamma_{sq}^2\sqrt{4(1-\p_2)\hat{B} + 2}^3}\rp
e^{-\lp\frac{\hat{h}}{\sqrt{4(1-\p_2)\hat{B} + 2}}\rp^2}\rp,
   \end{equation}
and
\begin{eqnarray}\label{eq:2levder31}
f_{(d\gamma)}^{(3)}=\frac{\c_2(1-\p_2)e^{-\frac{\hat{B}\hat{C}^2}{2(1-\p_2)\hat{B} + 1}}}{8\gamma_{sq}^2\sqrt{2(1-\p_2)\hat{B} + 1}^3}
\erfc\lp\frac{\hat{h}}{\sqrt{4(1-\p_2)\hat{B} + 2}}\rp,
   \end{eqnarray}
An easy combination of (\ref{eq:negprac24a2}), (\ref{eq:2levder22}), (\ref{eq:2levder23}), and (\ref{eq:2levder28})-(\ref{eq:2levder31}) ensures that all the ingredients needed to determine $\gamma_{sq}$--derivative are obtained. After solving the following system
\begin{eqnarray}\label{eq:2levder32}
  \frac{d\bar{\psi}_{rd}(\p,\q,\c,\gamma_{sq},\gamma_{sq}^{(p)},1,1,-1) }{d\q_2}
 & = &  0\nonumber \\
 \frac{d\bar{\psi}_{rd}(\p,\q,\c,\gamma_{sq},\gamma_{sq}^{(p)},1,1,-1) }{d\p_2}
 & = &  0 \nonumber \\
 \frac{d\bar{\psi}_{rd}(\p,\q,\c,\gamma_{sq},\gamma_{sq}^{(p)},1,1,-1) }{d\c_2}
 & = &  0\nonumber \\
 \frac{d\bar{\psi}_{rd}(\p,\q,\c,\gamma_{sq},\gamma_{sq}^{(p)},1,1,-1) }{d\gamma_{sq}^{(p)}}
 & = &  0\nonumber \\
 \frac{d\bar{\psi}_{rd}(\p,\q,\c,\gamma_{sq},\gamma_{sq}^{(p)},1,1,-1) }{d\gamma_{sq}}
 & = &  0,
      \end{eqnarray}
and denoting by $\hat{\q}_2,\hat{\p}_2,\hat{\c}_2,\hat{\gamma}_{sq}^{(p)},\hat{\gamma}_{sq}$  the obtained solution, we utilize
\begin{align}\label{eq:2levder33}
 - f_{sq}^{(2)}(\infty)=\bar{\psi}_{rd}(\hat{\p},\hat{\q},\hat{\c},\hat{\gamma}_{sq},\hat{\gamma}_{sq}^{(p)},1,1,-1)   & =
  0,
  \end{align}
to determine the critical $\alpha_c(\kappa)$, for any given $\kappa$. For example, taking $\kappa=-1.5$, we find
\begin{equation}\label{eq:2levder34}
\hspace{-2.5in}(\mbox{\bl{\textbf{\emph{full} second level:}}}) \qquad \qquad  a_c^{(2,f)}(-1.5)
\approx  \bl{\mathbf{36.57}}.
  \end{equation}

\vspace{.1in}

\noindent \underline{\textbf{\emph{Closed form relations:}}} To handle the above system we found as useful to utilize the following helpful, closed form, relations. First from (\ref{eq:2levder1}), we find
\begin{eqnarray}\label{eq:helprel1}
 \p_2=\frac{\q_2}{(2\gamma_{sq}^{(p)}-\c_2(1-\q_2))^2}.
     \end{eqnarray}
From (\ref{eq:2levder21a0}), we further have
\begin{eqnarray}\label{eq:helprel2}
     1=\frac{1-\q_2}{2\gamma_{sq}^{(p)}(2\gamma_{sq}^{(p)}-\c_2(1-\q_2))}+\frac{\q_2}{(2\gamma_{sq}^{(p)}-\c_2(1-\q_2))^2}.
\end{eqnarray}
Combining (\ref{eq:helprel1}) and (\ref{eq:helprel2}), we obtain
\begin{eqnarray}\label{eq:helprel3}
     1=\frac{1-\q_2}{2\gamma_{sq}^{(p)}(2\gamma_{sq}^{(p)}-\c_2(1-\q_2))}+\frac{\q_2}{(2\gamma_{sq}^{(p)}-\c_2(1-\q_2))^2}
     =\frac{1-\q_2}{2\gamma_{sq}^{(p)}(2\gamma_{sq}^{(p)}-\c_2(1-\q_2))}+\p_2.
\end{eqnarray}
A further combination of (\ref{eq:helprel1}) and (\ref{eq:helprel3}) gives
\begin{eqnarray}\label{eq:helprel4}
      \gamma_{sq}^{(p)} =\frac{1-\q_2}{2(1-\p_2)(2\gamma_{sq}^{(p)}-\c_2(1-\q_2))}
      =\frac{1}{2}\frac{1-\q_2}{1-\p_2}\sqrt{\frac{\p_2}{\q_2}}.
\end{eqnarray}
Also, from (\ref{eq:helprel1}) we have
\begin{eqnarray}\label{eq:helprel5}
\c_2(1-\q_2) =2\gamma_{sq}^{(p)}-\sqrt{\frac{\q_2}{\p_2}}.
     \end{eqnarray}
A combination of (\ref{eq:helprel4}) and (\ref{eq:helprel5}) gives
\begin{eqnarray}\label{eq:helprel6}
\c_2(1-\q_2) =2\gamma_{sq}^{(p)}-\sqrt{\frac{\q_2}{\p_2}}=\frac{1-\q_2}{1-\p_2}\sqrt{\frac{\p_2}{\q_2}}-\sqrt{\frac{\q_2}{\p_2}},
     \end{eqnarray}
and
\begin{eqnarray}\label{eq:helprel7}
\c_2 =\frac{1}{1-\p_2}\sqrt{\frac{\p_2}{\q_2}}-\frac{1}{1-\q_2}\sqrt{\frac{\q_2}{\p_2}}.
     \end{eqnarray}

\noindent \underline{\textbf{\emph{Concrete numerical values:}}}  In  Table \ref{tab:tab1}, we complement $a_c^{(2,f)}(-1.5)$ with the concrete values of all the relevant quantities related to the second \emph{full} (2-sfl RDT) level of lifting. To enable a systematic view of the lifting progress,  we, in parallel, show the same quantities for the first \emph{full} (1-sfl RDT) and the second \emph{partial} (2-spf RDT) level.
\begin{table}[h]
\caption{$r$-sfl RDT parameters; \emph{negative} spherical perceptron capacity;  $\hat{\c}_1\rightarrow 1$; $\kappa=-1.5$; $n,\beta\rightarrow\infty$}\vspace{.1in}
%\begin{adjustwidth}{-1.4cm}{}
\centering
\def\arraystretch{1.2}
\begin{tabular}{||l||c|c||c|c||c|c||c||c||}\hline\hline
 \hspace{-0in}$r$-sfl RDT                                             & $\hat{\gamma}_{sq}$    & $\hat{\gamma}_{sq}^{(p)}$    &  $\hat{\p}_2$ & $\hat{\p}_1`$     & $\hat{\q}_2$  & $\hat{\q}_1$ &  $\hat{\c}_2$    & $\alpha_c^{(r)}(-1.5)$  \\ \hline\hline
$1$-sfl RDT                                      & $0.5$ & $0.5$ &  $0$  & $\rightarrow 1$   & $0$ & $\rightarrow 1$
 &  $\rightarrow 0$  & \bl{$\mathbf{43.77}$} \\ \hline\hline
 $2$-spl RDT                                      & $0.1737$ & $1.4397$ &  $0$ & $\rightarrow 1$ &  $0$ & $\rightarrow 1$ &   $2.5320$   & \bl{$\mathbf{37.36}$} \\ \hline
  $2$-sfl RDT                                      & $0.1324$  & $1.8884$  & $0.4747$ & $\rightarrow 1$ &  $0.0981$ & $\rightarrow 1$
 &  $3.6835$   & \bl{$\mathbf{36.57}$}  \\ \hline\hline
  \end{tabular}
%\end{adjustwidth}
\label{tab:tab1}
\end{table}
In Table \ref{tab:tab2}, we show the key, second level of lifting, parameters over a range of $\kappa$.
\begin{table}[h]
\caption{$2$-sfl RDT parameters; \emph{negative} spherical perceptron capacity $\alpha=\alpha_c^{(2,f)}(\kappa)$}\vspace{.1in}
%\begin{adjustwidth}{-1.4cm}{}
\centering
\def\arraystretch{1.2}
\begin{tabular}{||l||c|c|c|c|c|c|c|c|c|c||}\hline\hline
 \hspace{-0in}$\kappa$                                             & $\mathbf{-2.7}$    & $\mathbf{-2.3}$ & $\mathbf{-2.0}$ & $\mathbf{-1.7}$        & $\mathbf{-1.5}$  & $\mathbf{-1.3}$   & $\mathbf{-1}$  &  $\mathbf{-0.7}$    & $\mathbf{-0.5}$    & $\mathbf{-0.3}$  \\ \hline\hline
$\hat{\gamma}_{sq}$                                       & $ 0.0739$ & $ 0.0925$ & $0.1086$ & $0.1238$ & $0.1324$ & $  0.1403 $ & $0.1500$ & $ 0.1575 $ & $ 0.1609$ & $0.1631$
 \\ \hline
$\hat{\gamma}_{sq}^{(p)}$                                       & $3.3827$ & $2.7015$ & $2.3033$ & $2.0219$ & $1.8884$ & $ 1.7823 $ & $1.6660$ & $1.5869$ & $1.5542$ & $1.5339$
 \\ \hline
$\hat{\p}_2$                                       & $0.0560$ & $0.1503$ & $0.2664$ & $0.3934$ & $0.4747$ & $ 0.5520 $ & $0.6649$ & $0.7781$ & $ 0.8550$ & $0.9293$
 \\ \hline
$\hat{\q}_{2}$\hspace{-.08in}                                      & $0.0014  $  & $0.0070   $  & $ 0.0223  $  & $ 0.0580  $  & $ 0.0981  $  & $  0.1547  $  & $   0.2780 $  & $  0.4591  $  & $ 0.6170 $  & $   0.7990  $  \\ \hline
$\hat{\c}_2$                                       & $6.6181$ & $5.2234$ & $4.4157$ & $3.8852$ & $3.6835$ & $ 3.5906 $ & $3.7194$ & $4.4476$ & $5.8985$ & $10.650$
 \\ \hline\hline
 $\alpha$                                      & $\bl{\mathbf{942.9}}  $ & $ \bl{\mathbf{ 284.5 }} $ & $  \bl{\mathbf{125.43}} $ & $  \bl{\mathbf{ 58.80 }} $ & $ \bl{\mathbf{ 36.57 }}$ & $  \bl{\mathbf{ 23.29}}  $ & $ \bl{\mathbf{ 12.32}} $ & $  \bl{\mathbf{ 6.817 }} $ & $ \bl{\mathbf{ 4.701 }}$ & $   \bl{\mathbf{3.298 }}$ \\ \hline\hline
 \end{tabular}
%\end{adjustwidth}
\label{tab:tab2}
\end{table}
The progression of the capacity as the level of lifting increases is shown in Table \ref{tab:tab3}.
\begin{table}[h]
\caption{\emph{Negative} spherical  perceptron capacity $\alpha_c(\kappa)$ --- early progression of $r$-sfl RDT mechanism}\vspace{.1in}
%\begin{adjustwidth}{-1.4cm}{}
\centering
\def\arraystretch{1.2}
\begin{tabular}{||l||c|c|c|c||}\hline\hline
 \hspace{-0in}$\kappa$                                               & $\mathbf{-2.0}$ & $\mathbf{-1.5}$ & $\mathbf{-1}$        & $\mathbf{-0.5}$    \\ \hline\hline
$\alpha_c^{(1,f)}(\kappa)$                    & $  173.4 $  & $   43.77  $  & $ 13.27 $  & $   4.770 $
 \\ \hline
$\alpha_c^{(2,p)}(\kappa)$                                      & $  126.2 $  & $   37.36  $  & $ 12.78 $  & $   4.770 $ \\ \hline \hline
 $\alpha_c^{(2,f)}(\kappa)$                                      & $ \bl{\mathbf{ 125.4 }} $ & $  \bl{\mathbf{36.57}} $ & $  \bl{\mathbf{ 12.32 }} $ & $ \bl{\mathbf{ 4.701 }}$   \\ \hline\hline
 \end{tabular}
%\end{adjustwidth}
\label{tab:tab3}
\end{table}

%%%%%%%%%%%%%%%%%%%%%%%%%%%%%%%%%%%%%%%%%%%%%%%%%%%%%%%%%%%%%%%%%%%%%%%%%%%%%%%%%%%%%%%%%%%%%%%%%%%%%%%%%%%%%%%%%%%%%%%%%
%%%%%%%%%%%%%%%%%%%%%%%%%%%%%%%%%%%%%%%%%%%%%%%%%%%%%%%%%%%%%%%%%%%%%%%%%%%%%%%%%%%%%%%%%%%%%%%%%%%%%%%%%%%%%%%%%%%%%%%%%
%%%%%%%%%%%%%%%%%%%%%%%%%%%%%%%%%%%%%%%%%%%%%%%%%%%%%%%%%%%%%%%%%%%%%%%%%%%%%%%%%%%%%%%%%%%%%%%%%%%%%%%%%%%%%%%%%%%%%%%%%
%%%%%%%%%%%%%%%%%%%%%%%%%%%%%%%%%%%%%%%%%%%%%%%%%%%%%%%%%%%%%%%%%%%%%%%%%%%%%%%%%%%%%%%%%%%%%%%%%%%%%%%%%%%%%%%%%%%%%%%%%
%%%%%%%%%%%%%%%%%%%%%%%%%%%%%%%%%%%%%%%%%%%%%%%%%%%%%%%%%%%%%%%%%%%%%%%%%%%%%%%%%%%%%%%%%%%%%%%%%%%%%%%%%%%%%%%%%%%%%%%%%
\subsubsection{$r=3$ -- third level of lifting}
\label{sec:thirdlev}
%%%%%%%%%%%%%%%%%%%%%%%%%%%%%%%%%%%%%%%%%%%%%%%%%%%%%%%%%%%%%%%%%%%%%%%%%%%%%%%%%%%%%%%%%%%%%%%%%%%%%%%%%%%%%%%%%%%%%%%%%
%%%%%%%%%%%%%%%%%%%%%%%%%%%%%%%%%%%%%%%%%%%%%%%%%%%%%%%%%%%%%%%%%%%%%%%%%%%%%%%%%%%%%%%%%%%%%%%%%%%%%%%%%%%%%%%%%%%%%%%%%
%%%%%%%%%%%%%%%%%%%%%%%%%%%%%%%%%%%%%%%%%%%%%%%%%%%%%%%%%%%%%%%%%%%%%%%%%%%%%%%%%%%%%%%%%%%%%%%%%%%%%%%%%%%%%%%%%%%%%%%%%
%%%%%%%%%%%%%%%%%%%%%%%%%%%%%%%%%%%%%%%%%%%%%%%%%%%%%%%%%%%%%%%%%%%%%%%%%%%%%%%%%%%%%%%%%%%%%%%%%%%%%%%%%%%%%%%%%%%%%%%%%
%%%%%%%%%%%%%%%%%%%%%%%%%%%%%%%%%%%%%%%%%%%%%%%%%%%%%%%%%%%%%%%%%%%%%%%%%%%%%%%%%%%%%%%%%%%%%%%%%%%%%%%%%%%%%%%%%%%%%%%%%

Since we have already seen the main idea behind the partial lifting, we here immediately  consider \emph{full} third level of lifting. For $r=3$, we have that $\hat{\p}_1\rightarrow 1$ and $\hat{\q}_1\rightarrow 1$  as well as  $\hat{\p}_{r+1}=\hat{\p}_{4}=\hat{\q}_{r+1}=\hat{\q}_{4}=0$.  Analogously to (\ref{eq:negprac22}), we now write
{\small\begin{align}\label{eq:3negprac24}
    \bar{\psi}_{rd}(\p,\q,\c,\gamma_{sq},\gamma_{sq}^{(p)},1,1,-1)   & =   \frac{1}{2}
(1-\p_2\q_2)\c_2+ \frac{1}{2}
(\p_2\q_2-\p_3\q_3)\c_3 \nonumber \\
& \quad  -  \gamma_{sq}^{(p)}  - \frac{1}{\c_3}\mE_{{\mathcal U}_4}\log\lp \mE_{{\mathcal U}_3} \lp \mE_{{\mathcal U}_2} e^{\c_2\frac{\lp\sqrt{1-\q_2}\h_1^{(2)} +\sqrt{\q_2-\q_3}\h_1^{(3)}+\sqrt{\q_3}\h_1^{(4)} \rp^2}{4 \gamma_{sq}^{(p)}}}\rp^{\frac{\c_3}{\c_2}}\rp \nonumber \\
& \quad   + \gamma_{sq}
 -\frac{\alpha}{\c_3}\mE_{{\mathcal U}_4} \log\lp \mE_{{\mathcal U}_3} \lp \mE_{{\mathcal U}_2} e^{-\c_2\frac{\max(\sqrt{1-\p_2}\u_1^{(2,2)}+\sqrt{\p_2-\p_3}\u_1^{(2,3)}+\sqrt{\p_3}\u_1^{(2,4)},0)^2}{4\gamma_{sq}}}\rp^{\frac{\c_3}{\c_2}}\rp \nonumber \\
 & =   \frac{1}{2}
(1-\p_2\q_2)\c_2+ \frac{1}{2}
(\p_2\q_2-\p_3\q_3)\c_3 -  \gamma_{sq}^{(p)} \nonumber \\
&\quad
-\Bigg(\Bigg. -\frac{1}{2\c_2} \log \lp \frac{2\gamma_{sq}^{(p)}-\c_2(1-\q_2)}{2\gamma_{sq}^{(p)}} \rp  -\frac{1}{2\c_3} \log \lp \frac{2\gamma_{sq}^{(p)}-\c_2(1-\q_2)-\c_3(\q_2-\q_3)}{2\gamma_{sq}^{(p)}-\c_2(1-\q_2)} \rp  \nonumber \\
& \quad +  \frac{\q_3}{2(2\gamma_{sq}^{(p)}-\c_2(1-\q_2)-\c_3(\q_2-\q_3))}   \Bigg.\Bigg)
 \nonumber \\
& \quad   + \gamma_{sq}
 -\frac{\alpha}{\c_3}\mE_{{\mathcal U}_4} \log\lp \mE_{{\mathcal U}_3} \lp \mE_{{\mathcal U}_2} e^{-\c_2\frac{\max(\sqrt{1-\p_2}\u_1^{(2,2)}+\sqrt{\p_2-\p_3}\u_1^{(2,3)}+\sqrt{\p_3}\u_1^{(2,4)},0)^2}{4\gamma_{sq}}}\rp^{\frac{\c_3}{\c_2}}\rp,
    \end{align}}

\noindent where we handled the first sequence of integrals utilizing the closed form solutions obtained in \cite{Stojnichopflrdt23}. After solving the remaining integrals, we also have
\begin{eqnarray}\label{eq:3negprac24a2}
\tilde{h} & = &  -\frac{\sqrt{\p_2-\p_3}\u_1^{(2,3)}+\sqrt{\p_3}\u_1^{(2,4)}+\kappa}{\sqrt{1-\p_2}}    \nonumber \\
\tilde{B} & = & \frac{\c_2}{4\gamma_{sq}} %\rightarrow \infty
\nonumber \\
\tilde{C} & = & \sqrt{\p_2-\p_3}\u_1^{(2,3)}+\sqrt{\p_3}\u_1^{(2,4)}+\kappa \nonumber \\
f_{(zd)}^{(3,f)}& = & \frac{e^{-\frac{\tilde{B}\tilde{C}^2}{2(1-\p_2)\tilde{B} + 1}}}{2\sqrt{2(1-\p_2)\tilde{B} + 1}}
\erfc\lp\frac{\tilde{h}}{\sqrt{4(1-\p_2)\tilde{B} + 2}}\rp
%\rightarrow 0
\nonumber \\
f_{(zu)}^{(3,f)}& = & \frac{1}{2}\erfc\lp-\frac{\tilde{h}}{\sqrt{2}}\rp,  %\rightarrow \frac{1}{2}\erfc\lp \frac{\kappa}{\sqrt{2}}\rp
\nonumber \\
f_{(zt)}^{(3,f)}& = & f_{(zd)}^{(3,f)}+f_{(zu)}^{(3,f)}.
   \end{eqnarray}
and
\begin{eqnarray}\label{eq:3negprac24a3}
 \mE_{{\mathcal U}_4} \log\lp \mE_{{\mathcal U}_3} \lp \mE_{{\mathcal U}_2} e^{-\c_2\frac{\max(\sqrt{1-\p_2}\u_1^{(2,2)}+\sqrt{\p_2-\p_3}\u_1^{(2,3)}+\sqrt{\p_3}\u_1^{(2,4)},0)^2}{4\gamma_{sq}}}\rp^{\frac{\c_3}{\c_2}}\rp
=   \mE_{{\mathcal U}_4} \log\lp \mE_{{\mathcal U}_3} \lp f_{(zt)}^{(3,f)} \rp^{\frac{\c_3}{\c_2}}\rp.
    \end{eqnarray}
Combining (\ref{eq:3negprac24}) and (\ref{eq:3negprac24a3}), we obtain
\begin{eqnarray}\label{eq:3negprac24ver}
    \bar{\psi}_{rd}(\p,\q,\c,\gamma_{sq},\gamma_{sq}^{(p)},1,1,-1)
 & = &   \frac{1}{2}
(1-\p_2\q_2)\c_2+ \frac{1}{2}
(\p_2\q_2-\p_3\q_3)\c_3 -  \gamma_{sq}^{(p)} \nonumber \\
& &
-\Bigg(\Bigg. -\frac{1}{2\c_2} \log \lp \frac{2\gamma_{sq}^{(p)}-\c_2(1-\q_2)}{2\gamma_{sq}^{(p)}} \rp  \nonumber \\
& & -\frac{1}{2\c_3} \log \lp \frac{2\gamma_{sq}^{(p)}-\c_2(1-\q_2)-\c_3(\q_2-\q_3)}{2\gamma_{sq}^{(p)}-\c_2(1-\q_2)} \rp  \nonumber \\
& & +  \frac{\q_3}{2(2\gamma_{sq}^{(p)}-\c_2(1-\q_2)-\c_3(\q_2-\q_3))}   \Bigg.\Bigg)
 \nonumber \\
& &  + \gamma_{sq}
 -\frac{\alpha}{\c_3}   \mE_{{\mathcal U}_4} \log\lp \mE_{{\mathcal U}_3} \lp f_{(zt)}^{(3,f)} \rp^{\frac{\c_3}{\c_2}}\rp.
    \end{eqnarray}

%%%%%%%%%%%%%%%%%%%%%%%%%%%%%%%%%%%%%%%%%%%%%%%%%%%%%%%%%%%%%%%%%%%%%%%%%%%%%%%%%%%%%%%%%%%%%%%%%%%%%%%%%%%%%%%%%%%%%%%%%
%%%%%%%%%%%%%%%%%%%%%%%%%%%%%%%%%%%%%%%%%%%%%%%%%%%%%%%%%%%%%%%%%%%%%%%%%%%%%%%%%%%%%%%%%%%%%%%%%%%%%%%%%%%%%%%%%%%%%%%%%
%%%%%%%%%%%%%%%%%%%%%%%%%%%%%%%%%%%%%%%%%%%%%%%%%%%%%%%%%%%%%%%%%%%%%%%%%%%%%%%%%%%%%%%%%%%%%%%%%%%%%%%%%%%%%%%%%%%%%%%%%
%%%%%%%%%%%%%%%%%%%%%%%%%%%%%%%%%%%%%%%%%%%%%%%%%%%%%%%%%%%%%%%%%%%%%%%%%%%%%%%%%%%%%%%%%%%%%%%%%%%%%%%%%%%%%%%%%%%%%%%%%
%%%%%%%%%%%%%%%%%%%%%%%%%%%%%%%%%%%%%%%%%%%%%%%%%%%%%%%%%%%%%%%%%%%%%%%%%%%%%%%%%%%%%%%%%%%%%%%%%%%%%%%%%%%%%%%%%%%%%%%%%
\subsubsection{Third level derivatives}
\label{sec:thirdlevder}
%%%%%%%%%%%%%%%%%%%%%%%%%%%%%%%%%%%%%%%%%%%%%%%%%%%%%%%%%%%%%%%%%%%%%%%%%%%%%%%%%%%%%%%%%%%%%%%%%%%%%%%%%%%%%%%%%%%%%%%%%
%%%%%%%%%%%%%%%%%%%%%%%%%%%%%%%%%%%%%%%%%%%%%%%%%%%%%%%%%%%%%%%%%%%%%%%%%%%%%%%%%%%%%%%%%%%%%%%%%%%%%%%%%%%%%%%%%%%%%%%%%
%%%%%%%%%%%%%%%%%%%%%%%%%%%%%%%%%%%%%%%%%%%%%%%%%%%%%%%%%%%%%%%%%%%%%%%%%%%%%%%%%%%%%%%%%%%%%%%%%%%%%%%%%%%%%%%%%%%%%%%%%
%%%%%%%%%%%%%%%%%%%%%%%%%%%%%%%%%%%%%%%%%%%%%%%%%%%%%%%%%%%%%%%%%%%%%%%%%%%%%%%%%%%%%%%%%%%%%%%%%%%%%%%%%%%%%%%%%%%%%%%%%
%%%%%%%%%%%%%%%%%%%%%%%%%%%%%%%%%%%%%%%%%%%%%%%%%%%%%%%%%%%%%%%%%%%%%%%%%%%%%%%%%%%%%%%%%%%%%%%%%%%%%%%%%%%%%%%%%%%%%%%%%

One now needs to compute \emph{eight} derivatives with respect to $\q_3$, $\q_2$, $\p_3$, $\p_2$,, $\c_3$, $\c_2$,, $\gamma_{sq}$,, and $\gamma_{sq}^{(p)}$. We again systematically compute each of them.

\vspace{.1in}
\noindent \underline{\red{\textbf{(i) $\q_3$ -- derivative:}}} Utilizing (\ref{eq:3negprac24}) and (\ref{eq:3negprac24ver}), we have
\begin{eqnarray}\label{eq:3levder1}
   \frac{d\bar{\psi}_{rd}(\p,\q,\c,\gamma_{sq},\gamma_{sq}^{(p)},1,1,-1) }{d\q_3}
 & = &  -\frac{1}{2}
\p_2\c_3
-\Bigg(\Bigg.    -\frac{1}{2(2\gamma_{sq}^{(p)}-\c_2(1-\q_2)-\c_3(\q_2-\q_3))}    \nonumber \\
& & +  \frac{1}{2(2\gamma_{sq}^{(p)}-\c_2(1-\q_2)-\c_3(\q_2-\q_3))}   \nonumber \\
& & -  \frac{\c_3\q_3}{2(2\gamma_{sq}^{(p)}-\c_2(1-\q_2)-\c_3(\q_2-\q_3))^2}   \Bigg.\Bigg)
\nonumber \\
 & = &  -\frac{1}{2}
\p_3\c_3
+  \frac{\c_3\q_3}{2(2\gamma_{sq}^{(p)}-\c_2(1-\q_2)-\c_3(\q_2-\q_3))^2}.
      \end{eqnarray}

\vspace{.1in}

\noindent \underline{\red{\textbf{(ii) $\q_2$ -- derivative:}}} Relying further on (\ref{eq:3negprac24}) and (\ref{eq:3negprac24ver}), we also have
\begin{eqnarray}\label{eq:3levder1a1}
   \frac{d\bar{\psi}_{rd}(\p,\q,\c,\gamma_{sq},\gamma_{sq}^{(p)},1,1,-1) }{d\q_2}
 & = &  -\frac{1}{2}
\p_2(\c_2-\c_3)
-\Bigg(\Bigg. -\frac{1}{2(2\gamma_{sq}^{(p)}-\c_2(1-\q_2))}
 \nonumber \\
 & &
 -\frac{\c_2-\c_3}{2\c_3(2\gamma_{sq}^{(p)}-\c_2(1-\q_2)-\c_3(\q_2-\q_3))}
 +\frac{\c_2}{2\c_3(2\gamma_{sq}^{(p)}-\c_2(1-\q_2))}   \nonumber \\
& & -  \frac{\q_3(\c_2-\c_3)}{2(2\gamma_{sq}^{(p)}-\c_2(1-\q_2)-\c_3(\q_2-\q_3))^2}   \Bigg.\Bigg)
\nonumber \\
 & = &  -\frac{1}{2}
\p_2(\c_2-\c_3)
-\Bigg(\Bigg. \frac{\c_2-\c_3}{2\c_3(2\gamma_{sq}^{(p)}-\c_2(1-\q_2))}
 \nonumber \\
 & &
 -\frac{\c_2-\c_3}{2\c_3(2\gamma_{sq}^{(p)}-\c_2(1-\q_2)-\c_3(\q_2-\q_3))}
  \nonumber \\
& & -  \frac{\q_3(\c_2-\c_3)}{2(2\gamma_{sq}^{(p)}-\c_2(1-\q_2)-\c_3(\q_2-\q_3))^2}   \Bigg.\Bigg)
\nonumber \\
& = &  -\frac{1}{2}
\p_2(\c_2-\c_3)
-(\c_2-\c_3)\Bigg(\Bigg. \frac{1}{2\c_3(2\gamma_{sq}^{(p)}-\c_2(1-\q_2))}
 \nonumber \\
 & &
 -\frac{1}{2\c_3(2\gamma_{sq}^{(p)}-\c_2(1-\q_2)-\c_3(\q_2-\q_3))}
  \nonumber \\
& & -  \frac{\q_3}{2(2\gamma_{sq}^{(p)}-\c_2(1-\q_2)-\c_3(\q_2-\q_3))^2}   \Bigg.\Bigg) \nonumber \\
& = &
(\c_2-\c_3)\Bigg(\Bigg. -\frac{\p_2}{2}
 \nonumber \\
 & &
 +\frac{\q_2-\q_3}{2(2\gamma_{sq}^{(p)}-\c_2(1-\q_2))(2\gamma_{sq}^{(p)}-\c_2(1-\q_2)-\c_3(\q_2-\q_3))}
  \nonumber \\
& & +  \frac{\q_3}{2(2\gamma_{sq}^{(p)}-\c_2(1-\q_2)-\c_3(\q_2-\q_3))^2}   \Bigg.\Bigg).
       \end{eqnarray}

\vspace{.1in}

\noindent \underline{\red{\textbf{(iii) $\p_3$ -- derivative:}}} As above, we utilize (\ref{eq:3negprac24}) and (\ref{eq:3negprac24ver}) and start by writing
\begin{eqnarray}\label{eq:3levder2}
   \frac{d\bar{\psi}_{rd}(\p,\q,\c,\gamma_{sq},\gamma_{sq}^{(p)},1,1,-1) }{d\p_3}
 & = &  -\frac{1}{2}
\q_3\c_3
-\frac{\alpha}{\c_3}  \frac{d \lp \mE_{{\mathcal U}_4} \log\lp \mE_{{\mathcal U}_3} \lp f_{(zt)}^{(3,f)} \rp^{\frac{\c_3}{\c_2}}\rp\rp}{d\p_3}\nonumber \\
 & = &  -\frac{1}{2}
\q_3\c_3
-\frac{\alpha}{\c_3}  \mE_{{\mathcal U}_4} \lp \frac{1}{ \mE_{{\mathcal U}_3} \lp f_{(zt)}^{(3,f)} \rp^{\frac{\c_3}{\c_2}} }
\frac{d\lp \mE_{{\mathcal U}_3} \lp f_{(zt)}^{(3,f)} \rp^{\frac{\c_3}{\c_2}} \rp } {d\p_3} \rp
\nonumber \\
 & = &  -\frac{1}{2}
\q_3\c_3
-\frac{\alpha}{\c_3}  \mE_{{\mathcal U}_4} \lp \frac{1}{ \mE_{{\mathcal U}_3} \lp f_{(zt)}^{(3,f)} \rp^{\frac{\c_3}{\c_2}} }
 \mE_{{\mathcal U}_3} \lp \frac{\c_3}{\c_2}\lp f_{(zt)}^{(3,f)} \rp^{\frac{\c_3}{\c_2}-1} \frac{d f_{(zt)}^{(3,f)}}{d\p_3} \rp \rp. \nonumber \\
     \end{eqnarray}
From (\ref{eq:3negprac24a2}), we then have
\begin{eqnarray}\label{eq:3levder3}
 \frac{df_{(zt)}^{(3,f)}}{d\p_3}& = &  \frac{df_{(zd)}^{(3,f)}}{d\p_3}+ \frac{df_{(zu)}^{(3,f)}}{d\p_3}.
   \end{eqnarray}
Moreover, utilizing (\ref{eq:3negprac24a2}) further, we can also write
\begin{eqnarray}\label{eq:3levder4}
 \frac{df_{(zu)}^{(3,f)}}{d\p_3} =\frac{e^{-\frac{\tilde{h}^2}{2}}}{\sqrt{2\pi}}\frac{d\tilde{h}}{d\p_3},
   \end{eqnarray}
and
\begin{eqnarray}\label{eq:3levder5}
\frac{d\tilde{h}}{d\p_3}
=-\frac{-\frac{1}{2\sqrt{\p_2-\p_3}}\u_1^{(2,3)}+\frac{1}{2\sqrt{\p_3}}\u_1^{(2,4)}+\kappa}{\sqrt{1-\p_2}} .
   \end{eqnarray}
A combination of  (\ref{eq:3levder4}) and (\ref{eq:3levder5}) gives
\begin{eqnarray}\label{eq:3levder6}
 \frac{df_{(zu)}^{(3,f)}}{d\p_3}
 =\frac{e^{-\frac{\tilde{h}^2}{2}}}{\sqrt{2\pi}}\frac{d\tilde{h}}{d\p_3}
=\frac{e^{-\frac{\tilde{h}^2}{2}}}{\sqrt{2\pi}}
\lp -\frac{-\frac{1}{2\sqrt{\p_2-\p_3}}\u_1^{(2,3)}+\frac{1}{2\sqrt{\p_3}}\u_1^{(2,4)}+\kappa}{\sqrt{1-\p_2}}   \rp.
   \end{eqnarray}
   After observing
\begin{eqnarray}\label{eq:3levder7}
 \frac{d\tilde{C}}{d\p_3} =-\frac{1}{2\sqrt{\p_2-\p_3}}\u_1^{(2,3)}+\frac{1}{2\sqrt{\p_3}}\u_1^{(2,4)},
   \end{eqnarray}
 we can further write
\begin{eqnarray}\label{eq:3levder8}
 \frac{df_{(zd)}^{(3,f)}}{d\p_3} =f_{(d\p_3)}^{(1)}+f_{(d\p_3)}^{(2)},
   \end{eqnarray}
 where
\begin{eqnarray}\label{eq:3levder9}
f_{(d\p_3)}^{(1)}=\lp-\frac{\tilde{B} \tilde{C} \lp -\frac{1}{\sqrt{\p_2-\p_3}}\u_1^{(2,3)}+\frac{1}{\sqrt{\p_3}}\u_1^{(2,4)}  \rp }{(2(1-\p_2)\tilde{B} + 1)}   \rp
e^{-\frac{\tilde{B}\tilde{C}^2}{2(1-\p_2)\tilde{B} + 1}}
\frac{\erfc\lp\frac{\tilde{h}}{\sqrt{4(1-\p_2)\tilde{B} + 2}}\rp}{2\sqrt{2(1-\p_2)\tilde{B} + 1}},
   \end{eqnarray}
   and
\begin{equation}\label{eq:3levder10}
f_{(d\p_3)}^{(2)}=\frac{e^{-\frac{\tilde{B}\tilde{C}^2}{2(1-\p_2)\tilde{B} + 1}}}{2\sqrt{2(1-\p_2)\tilde{B} + 1}}
\lp -\frac{2}{\sqrt{\pi}}\lp
\frac{1}{\sqrt{4(1-\p_2)\tilde{B} + 2}}\frac{d\tilde{h}}{d\p_3}
 \rp
e^{-\lp\frac{\tilde{h}}{\sqrt{4(1-\p_2)\tilde{B} + 2}}\rp^2}\rp.
   \end{equation}
 A combination of (\ref{eq:3negprac24a2}), (\ref{eq:3levder2}), (\ref{eq:3levder3}), (\ref{eq:3levder6}), and (\ref{eq:3levder8})-(\ref{eq:3levder10}) is then sufficient to determine $\p_3$--derivative.

\vspace{.1in}

\noindent \underline{\red{\textbf{(iv) $\p_2$ -- derivative:}}} Relying again on (\ref{eq:3negprac24}) and (\ref{eq:3negprac24ver}), we have
\begin{eqnarray}\label{eq:3levder12}
   \frac{d\bar{\psi}_{rd}(\p,\q,\c,\gamma_{sq},\gamma_{sq}^{(p)},1,1,-1) }{d\p_2}
 & = &  -\frac{1}{2}
\q_2(\c_2-\c_3)
-\frac{\alpha}{\c_3}  \frac{d \lp \mE_{{\mathcal U}_4} \log\lp \mE_{{\mathcal U}_3} \lp f_{(zt)}^{(3,f)} \rp^{\frac{\c_3}{\c_2}}\rp\rp}{d\p_2}\nonumber \\
 & = &  -\frac{1}{2}
\q_2(\c_2-\c_3)
-\frac{\alpha}{\c_3}  \mE_{{\mathcal U}_4} \lp \frac{1}{ \mE_{{\mathcal U}_3} \lp f_{(zt)}^{(3,f)} \rp^{\frac{\c_3}{\c_2}} }
\frac{d\lp \mE_{{\mathcal U}_3} \lp f_{(zt)}^{(3,f)} \rp^{\frac{\c_3}{\c_2}} \rp } {d\p_2} \rp
\nonumber \\
 & = &  -\frac{1}{2}
\q_2(\c_2-\c_3) \nonumber \\
& & -\frac{\alpha}{\c_3}  \mE_{{\mathcal U}_4} \lp \frac{1}{ \mE_{{\mathcal U}_3} \lp f_{(zt)}^{(3,f)} \rp^{\frac{\c_3}{\c_2}} }
 \mE_{{\mathcal U}_3} \lp \frac{\c_3}{\c_2}\lp f_{(zt)}^{(3,f)} \rp^{\frac{\c_3}{\c_2}-1} \frac{d f_{(zt)}^{(3,f)}}{d\p_2} \rp \rp. \nonumber \\
     \end{eqnarray}
From (\ref{eq:3negprac24a2}), we then find
\begin{eqnarray}\label{eq:3levder13}
 \frac{df_{(zt)}^{(3,f)}}{d\p_2}& = &  \frac{df_{(zd)}^{(3,f)}}{d\p_2}+ \frac{df_{(zu)}^{(3,f)}}{d\p_2}.
   \end{eqnarray}
Utilizing (\ref{eq:3negprac24a2}) further, we also have
\begin{eqnarray}\label{eq:3levder14}
 \frac{df_{(zu)}^{(3,f)}}{d\p_2} =\frac{e^{-\frac{\tilde{h}^2}{2}}}{\sqrt{2\pi}}\frac{d\tilde{h}}{d\p_2},
   \end{eqnarray}
and
\begin{eqnarray}\label{eq:3levder15}
\frac{d\tilde{h}}{d\p_2}
=-\frac{\frac{1}{2\sqrt{\p_2-\p_3}}\u_1^{(2,3)} }{\sqrt{1-\p_2}}
-\frac{\sqrt{\p_2-\p_3}\u_1^{(2,3)}+\sqrt{\p_3}\u_1^{(2,4)}+\kappa}{2\sqrt{1-\p_2}^3}.
   \end{eqnarray}
Combining  (\ref{eq:3levder14}) and (\ref{eq:3levder15}), we obtain
\begin{eqnarray}\label{eq:3levder16}
 \frac{df_{(zu)}^{(3,f)}}{d\p_2}
 =\frac{e^{-\frac{\tilde{h}^2}{2}}}{\sqrt{2\pi}}\frac{d\tilde{h}}{d\p_3}
=\frac{e^{-\frac{\tilde{h}^2}{2}}}{\sqrt{2\pi}}
\lp -\frac{\frac{1}{2\sqrt{\p_2-\p_3}}\u_1^{(2,3)} }{\sqrt{1-\p_2}}
-\frac{\sqrt{\p_2-\p_3}\u_1^{(2,3)}+\sqrt{\p_3}\u_1^{(2,4)}+\kappa}{2\sqrt{1-\p_2}^3}   \rp.
   \end{eqnarray}
   After observing
\begin{eqnarray}\label{eq:3levder17}
 \frac{d\tilde{C}}{d\p_2} =\frac{1}{2\sqrt{\p_2-\p_3}}\u_1^{(2,3)},
   \end{eqnarray}
 we can further write
\begin{eqnarray}\label{eq:3levder18}
 \frac{df_{(zd)}^{(3,f)}}{d\p_2} =f_{(d\p_2)}^{(1)}+f_{(d\p_2)}^{(2)}+f_{(d\p_2)}^{(3)},
   \end{eqnarray}
 where
\begin{eqnarray}\label{eq:3levder19}
f_{(d\p_2)}^{(1)}=\lp-\frac{\tilde{B} \tilde{C}\u_1^{(2,3)}}{\sqrt{\p_2-\p_3}(2(1-\p_2)\tilde{B} + 1)}-\frac{2\tilde{B}^2\tilde{C}^2}{(2(1-\p_2)\tilde{B} + 1).^2}\rp
e^{-\frac{\tilde{B}\tilde{C}^2}{2(1-\p_2)\tilde{B} + 1}}
\frac{\erfc\lp\frac{\tilde{h}}{\sqrt{4(1-\p_2)\tilde{B} + 2}}\rp}{2\sqrt{2(1-\p_2)\tilde{B} + 1}},
   \end{eqnarray}
   and
\begin{equation}\label{eq:3levder20}
f_{(d\p_2)}^{(2)}=\frac{e^{-\frac{\tilde{B}\tilde{C}^2}{2(1-\p_2)\tilde{B} + 1}}}{2\sqrt{2(1-\p_2)\tilde{B} + 1}}
\lp -\frac{2}{\sqrt{\pi}}\lp
\frac{1}{\sqrt{4(1-\p_2)\tilde{B} + 2}}\frac{d\tilde{h}}{d\p_2}
+\frac{2\tilde{B}\tilde{h}}{\sqrt{4(1-\p_2)\tilde{B} + 2}^3}\rp
e^{-\lp\frac{\tilde{h}}{\sqrt{4(1-\p_2)\tilde{B} + 2}}\rp^2}\rp,
   \end{equation}
and
\begin{eqnarray}\label{eq:3levder21}
f_{(d\p_2)}^{(3)}=\frac{\tilde{B}e^{-\frac{\tilde{B}\tilde{C}^2}{2(1-\p_2)\tilde{B} + 1}}}{2\sqrt{2(1-\p_2)\tilde{B} + 1}^3}
\erfc\lp\frac{\tilde{h}}{\sqrt{4(1-\p_2)\tilde{B} + 2}}\rp.
   \end{eqnarray}
 A combination of (\ref{eq:3negprac24a2}), (\ref{eq:3levder12}), (\ref{eq:3levder13}), (\ref{eq:3levder16}), and (\ref{eq:3levder18})-(\ref{eq:3levder21}) is then sufficient to determine $\p_2$--derivative.

\vspace{.1in}

\noindent \underline{\red{\textbf{(v) $\c_3$ -- derivative:}}} As usual, we start by utilizing  (\ref{eq:3negprac24}) and (\ref{eq:3negprac24ver}) and  write
\begin{eqnarray}\label{eq:3levderc312}
   \frac{d\bar{\psi}_{rd}(\p,\q,\c,\gamma_{sq},\gamma_{sq}^{(p)},1,1,-1) }{d\c_3}
 & = &     \frac{1}{2}
(\p_2\q_2-\p_3\q_3)   \nonumber \\
& &
-\Bigg(\Bigg.   \frac{1}{2\c_3^2} \log \lp \frac{2\gamma_{sq}-\c_2(1-\q_2)-\c_3(\q_2-\q_3)}{2\gamma_{sq}-\c_2(1-\q_2)} \rp  \nonumber \\
& &  \frac{\q_2-\q_3}{2\c_3(2\gamma_{sq}-\c_2(1-\q_2)-\c_3(\q_2-\q_3))}   \nonumber \\
& & +  \frac{\q_3(\q_2-\q_3)}{2(2\gamma_{sq}-\c_2(1-\q_2)-\c_3(\q_2-\q_3))^2}   \Bigg.\Bigg)
 \nonumber \\
& &
 +\frac{\alpha}{\c_3^2}   \mE_{{\mathcal U}_4} \log\lp \mE_{{\mathcal U}_3} \lp f_{(zt)}^{(3,f)} \rp^{\frac{\c_3}{\c_2}}\rp \nonumber \\
& &
 -\frac{\alpha}{\c_3}   \mE_{{\mathcal U}_4} \frac{1}{ \mE_{{\mathcal U}_3} \lp f_{(zt)}^{(3,f)} \rp^{\frac{\c_3}{\c_2}}}
  \mE_{{\mathcal U}_3} \lp \lp f_{(zt)}^{(3,f)} \rp^{\frac{\c_3}{\c_2}} \frac{1}{\c_2}\log \lp f_{(zt)}^{(3,f)} \rp  \rp.
      \end{eqnarray}
(\ref{eq:3levderc312}) is then sufficient to determine $\c_3$--derivative.

\vspace{.1in}

\noindent \underline{\red{\textbf{(vi) $\c_2$ -- derivative:}}} We once again start by utilizing  (\ref{eq:3negprac24}) and (\ref{eq:3negprac24ver}) and  write
\begin{eqnarray}\label{eq:3levderc212}
   \frac{d\bar{\psi}_{rd}(\p,\q,\c,\gamma_{sq},\gamma_{sq}^{(p)},1,1,-1) }{d\c_2}
 & = &  \frac{1}{2}
(1-\p_2\q_2)   \nonumber \\
& &
-\Bigg(\Bigg. \frac{1}{2\c_2^2} \log \lp \frac{2\gamma_{sq}-\c_2(1-\q_2)}{2\gamma_{sq}} \rp  +\frac{1-\q_2}{2\c_2(2\gamma_{sq}-\c_2(1-\q_2))}
\nonumber \\
& &
+\frac{1-\q_2}{2\c_3(2\gamma_{sq}-\c_2(1-\q_2)-\c_3(\q_2-\q_3))}  \nonumber \\
& &
-\frac{1-\q_2}{2\c_3(2\gamma_{sq}-\c_2(1-\q_2))}   +  \frac{\q_3(1-\q_2)}{2(2\gamma_{sq}-\c_2(1-\q_2)-\c_3(\q_2-\q_3))^2}   \Bigg.\Bigg)
 \nonumber \\
& &   -\frac{\alpha}{\c_3}   \mE_{{\mathcal U}_4} \frac{1}{ \mE_{{\mathcal U}_3} \lp f_{(zt)}^{(3,f)} \rp^{\frac{\c_3}{\c_2}}} \nonumber \\
&&
\times
  \mE_{{\mathcal U}_3} \lp \lp f_{(zt)}^{(3,f)} \rp^{\frac{\c_3}{\c_2}} \lp  -\frac{\c_3}{\c_2^2}\log \lp f_{(zt)}^{(3,f)} \rp  +\frac{\c_3}{\c_2 f_{(zt)}^{(3,f)}} \frac{df_{(zt)}^{(3,f)}}{d\c_2}\rp  \rp.
      \end{eqnarray}
From (\ref{eq:3negprac24a2}), we find
\begin{eqnarray}\label{eq:3levderc213}
 \frac{df_{(zt)}^{(3,f)}}{d\c_2}& = &  \frac{df_{(zd)}^{(3,f)}}{d\c_2}+ \frac{df_{(zu)}^{(3,f)}}{d\c_2}=\frac{df_{(zd)}^{(3,f)}}{d\c_2},
   \end{eqnarray}
where we utilized the fact that
\begin{eqnarray}\label{eq:3levderc214}
 \frac{df_{(zu)}^{(3,f)}}{d\c_2} =\frac{e^{-\frac{\tilde{h}^2}{2}}}{\sqrt{2\pi}}\frac{d\tilde{h}}{d\c_2}=0.
   \end{eqnarray}
    After observing
\begin{eqnarray}\label{eq:3levderc217}
 \frac{d\tilde{B}}{d\c_2} =\frac{1}{4\gamma_{sq}}\quad \mbox{and} \quad  \frac{d\tilde{C}}{d\c_2} =0,
   \end{eqnarray}
 we further write
\begin{eqnarray}\label{eq:3levderc218}
 \frac{df_{(zd)}^{(3,f)}}{d\c_2} =f_{(d\c_2)}^{(1)}+f_{(d\c_2)}^{(2)}+f_{(d\c_2)}^{(3)},
   \end{eqnarray}
 where
\begin{eqnarray}\label{eq:3levderc219}
f_{(d\c_2)}^{(1)}=\lp-\frac{\tilde{C}^2}{4\gamma_{sq}(2(1-\p_2)\tilde{B} + 1)}+\frac{(1-\p_2)\tilde{B}\tilde{C}^2}{2\gamma_{sq}(2(1-\p_2)\tilde{B} + 1).^2}\rp
e^{-\frac{\tilde{B}\tilde{C}^2}{2(1-\p_2)\tilde{B} + 1}}
\frac{\erfc\lp\frac{\tilde{h}}{\sqrt{4(1-\p_2)\tilde{B} + 2}}\rp}{2\sqrt{2(1-\p_2)\tilde{B} + 1}},
   \end{eqnarray}
   and
\begin{equation}\label{eq:3levderc220}
f_{(d\c_2)}^{(2)}=\frac{e^{-\frac{\tilde{B}\tilde{C}^2}{2(1-\p_2)\tilde{B} + 1}}}{2\sqrt{2(1-\p_2)\tilde{B} + 1}}
\lp -\frac{2}{\sqrt{\pi}}\lp
 -\frac{(1-\p_2)\tilde{h}}{2\gamma_{sq}\sqrt{4(1-\p_2)\tilde{B} + 2}^3}\rp
e^{-\lp\frac{\tilde{h}}{\sqrt{4(1-\p_2)\tilde{B} + 2}}\rp^2}\rp,
   \end{equation}
and
\begin{eqnarray}\label{eq:3levderc221}
f_{(d\c_2)}^{(3)}=-\frac{(1-\p_2)e^{-\frac{\tilde{B}\tilde{C}^2}{2(1-\p_2)\tilde{B} + 1}}}{8\gamma_{sq}\sqrt{2(1-\p_2)\tilde{B} + 1}^3}
\erfc\lp\frac{\tilde{h}}{\sqrt{4(1-\p_2)\tilde{B} + 2}}\rp,
   \end{eqnarray}
A combination of (\ref{eq:3negprac24a2}), (\ref{eq:3levderc212}), (\ref{eq:3levderc213}), and (\ref{eq:3levderc218})-(\ref{eq:3levderc221}) is then sufficient to determine $\c_2$--derivative.

\vspace{.1in}

\noindent \underline{\red{\textbf{(vii) $\gamma_{sq}^{(p)}$ -- derivative:}}} From (\ref{eq:3negprac24}) and (\ref{eq:3negprac24ver}) ,  we easily also find
\begin{eqnarray}\label{eq:3levder21a0}
   \frac{d\bar{\psi}_{rd}(\p,\q,\c,\gamma_{sq},\gamma_{sq}^{(p)},1,1,-1) }{d\gamma_{sq}^{(p)}}
 & =  & -1
-\Bigg(\Bigg. -\frac{1}{\c_2(2\gamma_{sq}^{(p)}-\c_2(1-\q_2))}  + \frac{1}{2\c_2\gamma_{sq}^{(p)}}
\nonumber \\
& &
-\frac{1}{\c_3(2\gamma_{sq}^{(p)}-\c_2(1-\q_2)-\c_3(\q_2-\q_3))}
+\frac{1}{\c_3(2\gamma_{sq}^{(p)}-\c_2(1-\q_2))}   \nonumber \\
& & -  \frac{\q_3}{(2\gamma_{sq}^{(p)}-\c_2(1-\q_2)-\c_3(\q_2-\q_3))^2}   \Bigg.\Bigg) \nonumber \\
 & =  & -1
-\Bigg(\Bigg. -\frac{1-\q_2}{2\gamma_{sq}^{(p)}(2\gamma_{sq}^{(p)}-\c_2(1-\q_2))}
\nonumber \\
& &
-\frac{\q_2-\q_3}{(2\gamma_{sq}^{(p)}-\c_2(1-\q_2))(2\gamma_{sq}^{(p)}-\c_2(1-\q_2)-\c_3(\q_2-\q_3))}
   \nonumber \\
& & -  \frac{\q_3}{(2\gamma_{sq}^{(p)}-\c_2(1-\q_2)-\c_3(\q_2-\q_3))^2}   \Bigg.\Bigg).
\end{eqnarray}

\vspace{.1in}

\noindent \underline{\red{\textbf{(viii) $\gamma_{sq}$ -- derivative:}}} Relying again on (\ref{eq:3negprac24}) and (\ref{eq:3negprac24ver}), we write
\begin{equation}\label{eq:3levder22}
   \frac{d\bar{\psi}_{rd}(\p,\q,\c,\gamma_{sq},\gamma_{sq}^{(p)},1,1,-1) }{d\gamma_{sq}}
=
 1
- \frac{\alpha}{\c_3}  \mE_{{\mathcal U}_4} \lp \frac{1}{ \mE_{{\mathcal U}_3} \lp f_{(zt)}^{(3,f)} \rp^{\frac{\c_3}{\c_2}} }
 \mE_{{\mathcal U}_3} \lp \frac{\c_3}{\c_2}\lp f_{(zt)}^{(3,f)} \rp^{\frac{\c_3}{\c_2}-1} \frac{d f_{(zt)}^{(3,f)}}{d\gamma_{sq}} \rp \rp.
 \end{equation}
From (\ref{eq:3negprac24a2}), we also have
\begin{eqnarray}\label{eq:3levder23}
 \frac{df_{(zt)}^{(3,f)}}{d\gamma_{sq}}& = &  \frac{df_{(zd)}^{(3,f)}}{d\gamma_{sq}}+ \frac{df_{(zu)}^{(3,f)}}{d\gamma_{sq}}
 =\frac{df_{(zd)}^{(3,f)}}{d\gamma_{sq}},
   \end{eqnarray}
where we utilized
\begin{eqnarray}\label{eq:3levder24}
 \frac{df_{(zu)}^{(2,f)}}{d\gamma_{sq}} =\frac{e^{-\frac{\tilde{h}^2}{2}}}{\sqrt{2\pi}}\frac{d\tilde{h}}{d\gamma_{sq}}=0.
   \end{eqnarray}
    After observing
\begin{eqnarray}\label{eq:3levder27}
 \frac{d\tilde{h}}{d\gamma_{sq}} =\frac{d\tilde{C}}{d\gamma_{sq}} =0 \quad \mbox{and}\quad
  \frac{d\tilde{B}}{d\gamma_{sq}} =-\frac{\c_2}{4\gamma_{sq}^2},
   \end{eqnarray}
 we can further write
\begin{eqnarray}\label{eq:3levder28}
 \frac{df_{(zd)}^{(2,f)}}{d\gamma_{sq}} =f_{(d\gamma_{sq})}^{(1)}+f_{(d\gamma_{sq})}^{(2)}+f_{(d\gamma_{sq})}^{(3)},
   \end{eqnarray}
 where
\begin{eqnarray}\label{eq:3levder29}
f_{(d\gamma_{sq})}^{(1)}=\lp \frac{\c_2\tilde{C}^2}{4\gamma_{sq}^2(2(1-\p_2)\tilde{B} + 1)}-\frac{\c_2(1-\p_2)\tilde{B}\tilde{C}^2}{2\gamma_{sq}^2(2(1-\p_2)\tilde{B} + 1).^2}\rp
e^{-\frac{\tilde{B}\tilde{C}^2}{2(1-\p_2)\tilde{B} + 1}}
\frac{\erfc\lp\frac{\tilde{h}}{\sqrt{4(1-\p_2)\tilde{B} + 2}}\rp}{2\sqrt{2(1-\p_2)\tilde{B} + 1}},
   \end{eqnarray}
   and
\begin{equation}\label{eq:3levder30}
f_{(d\gamma_{sq})}^{(2)}=\frac{e^{-\frac{\tilde{B}\tilde{C}^2}{2(1-\p_2)\tilde{B} + 1}}}{2\sqrt{2(1-\p_2)\tilde{B} + 1}}
\lp -\frac{2}{\sqrt{\pi}}\lp
 \frac{\c_2(1-\p_2)\tilde{h}}{2\gamma_{sq}^2\sqrt{4(1-\p_2)\tilde{B} + 2}^3}\rp
e^{-\lp\frac{\tilde{h}}{\sqrt{4(1-\p_2)\tilde{B} + 2}}\rp^2}\rp,
   \end{equation}
and
\begin{eqnarray}\label{eq:3levder31}
f_{(d\gamma_{sq})}^{(3)}=\frac{\c_2(1-\p_2)e^{-\frac{\tilde{B}\tilde{C}^2}{2(1-\p_2)\tilde{B} + 1}}}{8\gamma_{sq}^2\sqrt{2(1-\p_2)\tilde{B} + 1}^3}
\erfc\lp\frac{\tilde{h}}{\sqrt{4(1-\p_2)\tilde{B} + 2}}\rp,
   \end{eqnarray}
Together, (\ref{eq:3negprac24a2}), (\ref{eq:3levder22}), (\ref{eq:3levder23}), and (\ref{eq:3levder28})-(\ref{eq:3levder31}) provide all necessary ingredients to determine $\gamma_{sq}$--derivative. One then proceeds by solving the following system
\begin{align}\label{eq:3levder32}
  \frac{d\bar{\psi}_{rd}(\p,\q,\c,\gamma_{sq},\gamma_{sq}^{(p)},1,1,-1) }{d\q_3}
 & = \frac{d\bar{\psi}_{rd}(\p,\q,\c,\gamma_{sq},\gamma_{sq}^{(p)},1,1,-1) }{d\q_2}=  0\nonumber \\
 \frac{d\bar{\psi}_{rd}(\p,\q,\c,\gamma_{sq},\gamma_{sq}^{(p)},1,1,-1) }{d\p_3}
 & =   \frac{d\bar{\psi}_{rd}(\p,\q,\c,\gamma_{sq},\gamma_{sq}^{(p)},1,1,-1) }{d\p_2}=  0\nonumber \\
 \frac{d\bar{\psi}_{rd}(\p,\q,\c,\gamma_{sq},\gamma_{sq}^{(p)},1,1,-1) }{d\c_3}
 & =   \frac{d\bar{\psi}_{rd}(\p,\q,\c,\gamma_{sq},\gamma_{sq}^{(p)},1,1,-1) }{d\c_2}
= 0\nonumber \\
 \frac{d\bar{\psi}_{rd}(\p,\q,\c,\gamma_{sq},\gamma_{sq}^{(p)},1,1,-1) }{d\gamma_{sq}^{(p)}}
 & =
 \frac{d\bar{\psi}_{rd}(\p,\q,\c,\gamma_{sq},\gamma_{sq}^{(p)},1,1,-1) }{d\gamma_{sq}}
  =   0.
      \end{align}
After denoting by $\hat{\q}_2,\hat{\p}_2,\hat{\c}_2,\hat{\gamma}_{sq}^{(p)},\hat{\gamma}_{sq}$  the obtained solution, one further utilizes
\begin{align}\label{eq:3levder33}
 - f_{sq}^{(3)}(\infty)=\bar{\psi}_{rd}(\hat{\p},\hat{\q},\hat{\c},\hat{\gamma}_{sq},\hat{\gamma}_{sq}^{(p)},1,1,-1)   & =
  0,
  \end{align}
to determine the critical $\alpha_c(\kappa)$, for any given $\kappa$. For example, specializing  to $\kappa=-1.5$, we find
\begin{equation}\label{eq:3levder34}
\hspace{-2.5in}(\mbox{\bl{\textbf{\emph{full} third level:}}}) \qquad \qquad  a_c^{(3,f)}(-1.5)
\approx  \bl{\mathbf{36.40}}.
  \end{equation}

%%%%%%%%%%%%%%%%%%%%%%%%%%%%%%%%%%%%%%%%%%%%%%%%%%%%%%%%%%%%%%%%%%%%%%%%%%%%%%%%%%%%%%%%%%%%%%%%%
%%%%%%%%%%%%%%%%%%%%%%%%%%%%%%%%%%%%%%%%%%%%%%%%%%%%%%%%%%%%%%%%%%%%%%%%%%%%%%%%%%%%%%%%%%%%%%%%%
%%%%%%%%%%%%%%%%%%%%%%%%%%%%%%%%%%%%%%%%%%%%%%%%%%%%%%%%%%%%%%%%%%%%%%%%%%%%%%%%%%%%%%%%%%%%%%%%%
%%%%%%%%%%%%%%%%%%%%%%%%%%%%%%%%%%%%%%%%%%%%%%%%%%%%%%%%%%%%%%%%%%%%%%%%%%%%%%%%%%%%%%%%%%%%%%%%%
%%%%%%%%%%%%%%%%%%%%%%%%%%%%%%%%%%%%%%%%%%%%%%%%%%%%%%%%%%%%%%%%%%%%%%%%%%%%%%%%%%%%%%%%%%%%%%%%%
\subsubsection{Explicit generic closed form parametric relations}
\label{sec:userel}
%%%%%%%%%%%%%%%%%%%%%%%%%%%%%%%%%%%%%%%%%%%%%%%%%%%%%%%%%%%%%%%%%%%%%%%%%%%%%%%%%%%%%%%%%%%%%%%%%
%%%%%%%%%%%%%%%%%%%%%%%%%%%%%%%%%%%%%%%%%%%%%%%%%%%%%%%%%%%%%%%%%%%%%%%%%%%%%%%%%%%%%%%%%%%%%%%%%
%%%%%%%%%%%%%%%%%%%%%%%%%%%%%%%%%%%%%%%%%%%%%%%%%%%%%%%%%%%%%%%%%%%%%%%%%%%%%%%%%%%%%%%%%%%%%%%%%
%%%%%%%%%%%%%%%%%%%%%%%%%%%%%%%%%%%%%%%%%%%%%%%%%%%%%%%%%%%%%%%%%%%%%%%%%%%%%%%%%%%%%%%%%%%%%%%%%
%%%%%%%%%%%%%%%%%%%%%%%%%%%%%%%%%%%%%%%%%%%%%%%%%%%%%%%%%%%%%%%%%%%%%%%%%%%%%%%%%%%%%%%%%%%%%%%%%

Solving the above system is doable in principle. However, in general it is not an easy task. It often requires a substantial effort to conduct all the required numerical work. Rather surprisingly and despite heavy analytical machinery, it turns out that the key lifting parameters are generically connected to each other. Moreover, we below uncover that the parametric interconnections can be described via \emph{remarkably simple and elegant} closed form expressions. Besides their analytical importance, the relations that we provide below are practically extremely useful and make the underlying numerical work immeasurably simpler and smoother.

We first observe that from (\ref{eq:3levder1}) one can obtain
\begin{eqnarray}\label{eq:userel1}
 \p_3=\frac{\q_3}{(2\gamma_{sq}^{(p)}-\c_2(1-\q_2)-\c_3(\q_2-\q_3))^2}.
      \end{eqnarray}
In a similar fashion, from (\ref{eq:3levder1a1}), we find
\begin{eqnarray}\label{eq:userel2}
  \p_2 =
 \frac{\q_2-\q_3}{(2\gamma_{sq}^{(p)}-\c_2(1-\q_2))(2\gamma_{sq}^{(p)}-\c_2(1-\q_2)-\c_3(\q_2-\q_3))}
 +  \frac{\q_3}{(2\gamma_{sq}^{(p)}-\c_2(1-\q_2)-\c_3(\q_2-\q_3))^2}.
       \end{eqnarray}
A combination of (\ref{eq:userel1}) and (\ref{eq:userel2}) then gives
\begin{eqnarray}\label{eq:userel3}
  2\gamma_{sq}^{(p)}-\c_2(1-\q_2)=
 \frac{\q_2-\q_3}{(\p_2-\p_3)(2\gamma_{sq}^{(p)}-\c_2(1-\q_2)-\c_3(\q_2-\q_3))}
 = \frac{\q_2-\q_3}{\p_2-\p_3}\sqrt{\frac{\p_3}{\q_3}}.
       \end{eqnarray}
One then also observes
\begin{eqnarray}\label{eq:userel4}
\c_3(\q_2-\q_3)=  2\gamma_{sq}^{(p)}-\c_2(1-\q_2)-(2\gamma_{sq}^{(p)}-\c_2(1-\q_2)-\c_3(\q_2-\q_3)).
        \end{eqnarray}
A combination of (\ref{eq:userel1}), (\ref{eq:userel3}), and (\ref{eq:userel4}) then gives
\begin{eqnarray}\label{eq:userel5}
\c_3(\q_2-\q_3)= \frac{\q_2-\q_3}{\p_2-\p_3}\sqrt{\frac{\p_3}{\q_3}} -\sqrt{\frac{\q_3}{\p_3}},
       \end{eqnarray}
and
\begin{eqnarray}\label{eq:userel6}
\c_3= \frac{1}{\p_2-\p_3}\sqrt{\frac{\p_3}{\q_3}} -\frac{1}{\q_2-\q_3}\sqrt{\frac{\q_3}{\p_3}}.
       \end{eqnarray}
From (\ref{eq:3levder21a0}), we also have
\begin{eqnarray}\label{eq:userel7}
 1 =\frac{1-\q_2}{2\gamma_{sq}^{(p)}(2\gamma_{sq}^{(p)}-\c_2(1-\q_2))}
+\frac{\q_2-\q_3}{(2\gamma_{sq}^{(p)}-\c_2(1-\q_2))(2\gamma_{sq}^{(p)}-\c_2(1-\q_2)-\c_3(\q_2-\q_3))}
+\p_3.
\end{eqnarray}
A combination of (\ref{eq:userel2}) and (\ref{eq:userel7}) further gives
\begin{eqnarray}\label{eq:userel8}
 1 =\frac{1-\q_2}{2\gamma_{sq}^{(p)}(2\gamma_{sq}^{(p)}-\c_2(1-\q_2))}
+(\p_2-\p_3)
+\p_3 = \frac{1-\q_2}{2\gamma_{sq}^{(p)}(2\gamma_{sq}^{(p)}-\c_2(1-\q_2))}
+\p_2.
\end{eqnarray}
From (\ref{eq:userel3}) and (\ref{eq:userel8}), we then find
\begin{eqnarray}\label{eq:userel9}
 \gamma_{sq}^{(p)} =\frac{1-\q_2}{2(1-\p_2)(2\gamma_{sq}^{(p)}-\c_2(1-\q_2))}=\frac{1}{2}\frac{1-\q_2}{1-\p_2}
 \frac{\p_2-\p_3}{\q_2-\q_3}\sqrt{\frac{\q_3}{\p_3}}.
\end{eqnarray}
Moreover, from
(\ref{eq:userel3}) and (\ref{eq:userel9}), we also have
\begin{eqnarray}\label{eq:userel10}
 \c_2(1-\q_2)=2\gamma_{sq}^{(p)}- \frac{\q_2-\q_3}{\p_2-\p_3}\sqrt{\frac{\p_3}{\q_3}}.
\end{eqnarray}
Combining (\ref{eq:userel9}) and (\ref{eq:userel10}), one then easily also has
\begin{eqnarray}\label{eq:userel11}
 \c_2=\frac{2\gamma_{sq}^{(p)}}{1-\q_2}- \frac{1}{1-\q_2}\frac{\q_2-\q_3}{\p_2-\p_3}\sqrt{\frac{\p_3}{\q_3}}
 =\frac{1}{1-\p_2}
 \frac{\p_2-\p_3}{\q_2-\q_3}\sqrt{\frac{\q_3}{\p_3}}- \frac{1}{1-\q_2}\frac{\q_2-\q_3}{\p_2-\p_3}\sqrt{\frac{\p_3}{\q_3}}.
\end{eqnarray}
We found all the above relations (and in particular those given in (\ref{eq:userel6}), (\ref{eq:userel9}), and (\ref{eq:userel11})) as extremely useful for the numerical work. Moreover, following the above procedure one also obtains for any $r$ the analogue versions of (\ref{eq:userel6}), (\ref{eq:userel9}), and (\ref{eq:userel11})

\begin{eqnarray}\label{eq:userel13}
 \gamma_{sq}^{(p)}&  = & \frac{1}{2}\frac{\q_1-\q_2}{\p_{1}-\p_2}  \prod_{k=2:2:r-1} \frac{ \p_{k}-\p_{k+1} }{ \q_{k}-\q_{k+1}} \prod_{k=2:2:r-2}  \frac{ \q_{k+1}-\q_{k+2} }{ \p_{k+1}-\p_{k+2}}
 \sqrt{\lp \frac{\q_r}{\p_r}\rp^{(-1)^{r+1}}}.
\end{eqnarray}
and for $i\in\{2,3,\dots,r\}$
\begin{eqnarray}\label{eq:userel14}
 \c_i &  = &
\frac{1}{\p_{i-1}-\p_i}  \prod_{k=i:2:r-1} \frac{ \p_{k}-\p_{k+1} }{ \q_{k}-\q_{k+1}} \prod_{k=i:2:r-2}  \frac{ \q_{k+1}-\q_{k+2} }{ \p_{k+1}-\p_{k+2}}
  \sqrt{\lp \frac{\q_r}{\p_r}\rp^{(-1)^i}}
\nonumber \\
& &  -
\frac{1}{\q_{i-1}-\q_i}  \prod_{k=i:2:r-1} \frac{ \q_{k}-\q_{k+1} }{ \p_{k}-\p_{k+1}} \prod_{k=i:2:r-2}  \frac{ \p_{k+1}-\p_{k+2} }{ \q_{k+1}-\q_{k+2}}
  \sqrt{\lp \frac{\p_r}{\q_r}\rp^{(-1)^i}}.
\end{eqnarray}

We summarize the above in the following lemma.

\begin{theorem} \label{thm:thmclfr} Assume the setup of Theorem \ref{thme:negthmprac1}. Let the ``non-fixed'' parts of $\hat{\p}_k$, $\hat{\q}_k$, and $\hat{\c}_k$ ($k\in\{2,3,\dots,r\}$) be the solutions of the system in (\ref{eq:negthmprac1eq1}). The following holds.

\noindent For $r=1$:
\begin{eqnarray}
 \hat{\gamma}_{sq}^{(p)} &  =  &  \frac{1}{2}.
   \label{eq:thmclfreq1}
 \end{eqnarray}

\noindent For $r=2$:
\begin{eqnarray}
 \hat{\gamma}_{sq}^{(p)} &  =  &  \frac{1}{2}\frac{1-\hat{\q}_2}{1-\hat{\p}_2}
 \sqrt{\frac{\hat{\p}_2}{\hat{\q}_2}} \nonumber \\
  \hat{\c}_2 & = &  \frac{1}{1-\hat{\p}_2}
 \sqrt{\frac{\hat{\p}_2}{\hat{\q}_2}}- \frac{1}{1-\hat{\q}_2} \sqrt{\frac{\hat{\q}_2}{\hat{\p}_2}}.
   \label{eq:thmclfreq2}
 \end{eqnarray}

\noindent For $r=3$:
\begin{eqnarray}
 \hat{\gamma}_{sq}^{(p)} &  =  &  \frac{1}{2}\frac{1-\hat{\q}_2}{1-\hat{\p}_2}
 \frac{\hat{\p}_2-\hat{\p}_3}{\hat{\q}_2-\hat{\q}_3}\sqrt{\frac{\hat{\q}_3}{\hat{\p}_3}} \nonumber \\
\hat{\c}_3 & = & \frac{1}{\hat{\p}_2-\hat{\p}_3}\sqrt{\frac{\hat{\p}_3}{\hat{\q}_3}} -\frac{1}{\hat{\q}_2-\hat{\q}_3}\sqrt{\frac{\hat{\q}_3}{\hat{\p}_3}} \nonumber \\
 \hat{\c}_2 & = &  \frac{1}{1-\hat{\p}_2}
 \frac{\hat{\p}_2-\hat{\p}_3}{\hat{\q}_2-\hat{\q}_3}\sqrt{\frac{\hat{\q}_3}{\hat{\p}_3}}- \frac{1}{1-\hat{\q}_2}\frac{\hat{\q}_2-\hat{\q}_3}{\hat{\p}_2-\hat{\p}_3}\sqrt{\frac{\hat{\p}_3}{\hat{\q}_3}}.
   \label{eq:thmclfreq3}
 \end{eqnarray}

\noindent For general $r$:
\begin{eqnarray}
 \hat{\gamma}_{sq}^{(p)}&  = & \frac{1}{2}\frac{\hat{\q}_1-\hat{\q}_2}{\hat{\p}_{1}-\hat{\p}_2}  \prod_{k=2:2:r-1} \frac{ \hat{\p}_{k}-\hat{\p}_{k+1} }{ \hat{\q}_{k}-\hat{\q}_{k+1}} \prod_{k=2:2:r-2}  \frac{\hat{\q}_{k+1}-\hat{\q}_{k+2} }{ \hat{\p}_{k+1}-\hat{\p}_{k+2}}
 \sqrt{\lp \frac{\hat{\q}_r}{\hat{\p}_r}\rp^{(-1)^{r+1}}} \nonumber \\
 \hat{\c}_i &  = &
\frac{1}{\hat{\p}_{i-1}-\hat{\p}_i}  \prod_{k=i:2:r-1} \frac{ \hat{\p}_{k}-\hat{\p}_{k+1} }{ \hat{\q}_{k}-\hat{\q}_{k+1}} \prod_{k=i:2:r-2}  \frac{ \hat{\q}_{k+1}-\hat{\q}_{k+2} }{ \hat{\p}_{k+1}-\hat{\p}_{k+2}}
  \sqrt{\lp \frac{\hat{\q}_r}{\hat{\p}_r}\rp^{(-1)^i}}
\nonumber \\
& &  -
\frac{1}{\hat{\q}_{i-1}-\hat{\q}_i}  \prod_{k=i:2:r-1} \frac{ \hat{\q}_{k}-\hat{\q}_{k+1} }{ \hat{\p}_{k}-\hat{\p}_{k+1}} \prod_{k=i:2:r-2}  \frac{ \hat{\p}_{k+1}-\hat{\p}_{k+2} }{ \hat{\q}_{k+1}-\hat{\q}_{k+2}}
  \sqrt{\lp \frac{\hat{\p}_r}{\hat{\q}_r}\rp^{(-1)^i}}, \quad \mbox{with}\quad i\in\{2,3,\dots,r\}. \nonumber \\
   \label{eq:thmclfreq4}
 \end{eqnarray}
\end{theorem}
\begin{proof}
  The $r=1$ case follows immediately from (\ref{eq:negprac19}), the $r=2$ from (\ref{eq:helprel4}) and (\ref{eq:helprel7}), the $r=3$ from (\ref{eq:userel6}), (\ref{eq:userel9}), and (\ref{eq:userel11}), whereas the general case follows by repeating the above procedure for an arbitrary $r$.
\end{proof}

%%%%%%%%%%%%%%%%%%%%%%%%%%%%%%%%%%%%%%%%%%%%%%%%%%%%%%%%%%%%%%%%%%%%%%%%%%%%%%%%%%%%%%%%%%%%%%%%%
%%%%%%%%%%%%%%%%%%%%%%%%%%%%%%%%%%%%%%%%%%%%%%%%%%%%%%%%%%%%%%%%%%%%%%%%%%%%%%%%%%%%%%%%%%%%%%%%%
%%%%%%%%%%%%%%%%%%%%%%%%%%%%%%%%%%%%%%%%%%%%%%%%%%%%%%%%%%%%%%%%%%%%%%%%%%%%%%%%%%%%%%%%%%%%%%%%%
%%%%%%%%%%%%%%%%%%%%%%%%%%%%%%%%%%%%%%%%%%%%%%%%%%%%%%%%%%%%%%%%%%%%%%%%%%%%%%%%%%%%%%%%%%%%%%%%%
%%%%%%%%%%%%%%%%%%%%%%%%%%%%%%%%%%%%%%%%%%%%%%%%%%%%%%%%%%%%%%%%%%%%%%%%%%%%%%%%%%%%%%%%%%%%%%%%%
\subsubsection{Concrete parameter values}
\label{sec:concval}
%%%%%%%%%%%%%%%%%%%%%%%%%%%%%%%%%%%%%%%%%%%%%%%%%%%%%%%%%%%%%%%%%%%%%%%%%%%%%%%%%%%%%%%%%%%%%%%%%
%%%%%%%%%%%%%%%%%%%%%%%%%%%%%%%%%%%%%%%%%%%%%%%%%%%%%%%%%%%%%%%%%%%%%%%%%%%%%%%%%%%%%%%%%%%%%%%%%
%%%%%%%%%%%%%%%%%%%%%%%%%%%%%%%%%%%%%%%%%%%%%%%%%%%%%%%%%%%%%%%%%%%%%%%%%%%%%%%%%%%%%%%%%%%%%%%%%
%%%%%%%%%%%%%%%%%%%%%%%%%%%%%%%%%%%%%%%%%%%%%%%%%%%%%%%%%%%%%%%%%%%%%%%%%%%%%%%%%%%%%%%%%%%%%%%%%
%%%%%%%%%%%%%%%%%%%%%%%%%%%%%%%%%%%%%%%%%%%%%%%%%%%%%%%%%%%%%%%%%%%%%%%%%%%%%%%%%%%%%%%%%%%%%%%%%

In  Table \ref{tab:tab3lev1},  $a_c^{(3,f)}(-1.5)$, obtained in (\ref{eq:3levder34}), is complemented with the concrete values of all the relevant quantities related to the third \emph{full} (3-sfl RDT)  level of lifting. To enable a systematic view of the lifting progress,  we, in parallel, show the results from Table \ref{tab:tab1} that contain the same quantities for the first \emph{full} (1-sfl RDT), the second \emph{partial} (2-spf RDT), and the second \emph{full} (2-sff RDT) level.
\begin{table}[h]
\caption{$r$-sfl RDT parameters; \emph{negative} spherical perceptron capacity;  $\hat{\c}_1\rightarrow 1$; $\kappa=-1.5$; $n,\beta\rightarrow\infty$}\vspace{.1in}
%\begin{adjustwidth}{-1.4cm}{}
\centering
\def\arraystretch{1.2}
{\small
\begin{tabular}{||l||c|c||c|c|c||c|c|c||c|c||c||}\hline\hline
 \hspace{-0in}$r$-sfl RDT                                             & $\hat{\gamma}_{sq}$    & $\hat{\gamma}_{sq}^{(p)}$    &  $\hat{\p}_3$ & $\hat{\p}_2$ & $\hat{\p}_1$     & $\hat{\q}_3$  &$\hat{\q}_2$  & $\hat{\q}_1$ &  $\hat{\c}_3$ &  $\hat{\c}_2$    & $\alpha_c^{(r)}(\kappa)$  \\ \hline\hline
$1$-sfl RDT                                      & $0.5$ & $0.5$ &  $0$ &  $0$  & $\rightarrow 1$  &  $0$ & $0$ & $\rightarrow 1$
&  $0$ &  $\rightarrow 0$  & \bl{$\mathbf{43.77}$} \\ \hline\hline
 $2$-spl RDT                                      & $0.1737$ & $1.4397$ &  $0$ &  $0$ & $\rightarrow 1$ &  $0$ &  $0$ & $\rightarrow 1$ &  $0$ &   $2.5320$   & \bl{$\mathbf{37.36}$} \\ \hline
  $2$-sfl RDT                                      & $0.1324$  & $1.8884$ &  $0$ & $0.4747$ & $\rightarrow 1$ &  $0$ &  $0.0981$ & $\rightarrow 1$
&  $0$ &  $3.6835$   & \bl{$\mathbf{36.57}$}  \\ \hline\hline
  $3$-sfl RDT                                      & $0.0647$  & $3.8759$ &  $0.4075$ & $0.9693$ & $\rightarrow 1$ &  $ 0.0743 $ &  $0.5384$ & $\rightarrow 1$
&  $3.25$ &  $12.6$   & \bl{$\mathbf{36.40}$}  \\ \hline\hline
  \end{tabular}
}
%\end{adjustwidth}
\label{tab:tab3lev1}
\end{table}
 In Table \ref{tab:tab3lev2}, we show the key second level of lifting parameters over a range of $\kappa$. The progression of the capacity as the level of lifting increases is shown in Table \ref{tab:tab3lev3}. The systematic showing of the progression in Table \ref{tab:tab3lev3} (as well as in in Table \ref{tab:tab3lev1}) allows one to also note, that the first rows in these tables relate to the results that can be obtained through the \emph{plain} RDT (see, e.g., \cite{StojnicGardGen13}), whereas their second rows relate to the results that can be obtained through the \emph{partially} lifted RDT of \cite{StojnicGardSphNeg13}.

\begin{table}[h]
\caption{$3$-sfl RDT parameters; \emph{negative} spherical perceptron capacity $\alpha=\alpha_c^{(3,f)}(\kappa)$}\vspace{.1in}
%\begin{adjustwidth}{-1.4cm}{}
\centering
\def\arraystretch{1.2}
\begin{tabular}{||l||c|c|c|c||}\hline\hline
 \hspace{-0in}$\kappa$                                             & $\mathbf{-2.0}$    & $\mathbf{-1.5}$ & $\mathbf{-1.0}$   & $\mathbf{-0.5}$   \\ \hline\hline
$\hat{\gamma}_{sq}$                                       & $ 0.0493$ &
$0.0647$
& $0.0835$ & $0.0886$
 \\ \hline
$\hat{\gamma}_{sq}^{(p)}$                                       & $5.0735$ &
   $3.8759$
 & $3.0237$ & $2.8159$
 \\ \hline \hline
$\hat{\p}_3$                                       & $0.2304$ &
$0.4075$
& $0.6252$ & $0.8165$
 \\ \hline
$\hat{\p}_2$                                       & $0.9821$ &
 $0.9693$
& $0.9691$ & $0.9821$
 \\ \hline \hline
$\hat{\q}_3$                                       & $0.0172$ &
 $ 0.0743 $
& $0.2500$ & $0.5681$
 \\ \hline
$\hat{\q}_2$                                       & $0.5392$ &
$0.5384$
 & $0.6536$ & $0.8179$
 \\ \hline \hline
$\hat{\c}_3$                                      & $4.35  $  &
   $3.25$
& $ 3.03  $  & $ 3.90  $   \\ \hline
$\hat{\c}_2$                                       & $16.4$ &
  $12.6$
& $12.1$ & $21.0$
 \\ \hline\hline
 $\alpha$                                      & $\bl{\mathbf{124.8}}  $ & $ \bl{\mathbf{ 36.40 }} $ & $  \bl{\mathbf{12.29}} $ & $  \bl{\mathbf{ 4.698 }} $ \\ \hline\hline
 \end{tabular}
%\end{adjustwidth}
\label{tab:tab3lev2}
\end{table}

\begin{table}[h]
\caption{\emph{Negative} spherical  perceptron capacity $\alpha_c(\kappa)$ --- progression of $r$-sfl RDT mechanism}\vspace{.1in}
%\begin{adjustwidth}{-1.4cm}{}
\centering
\def\arraystretch{1.2}
\begin{tabular}{||l||c|c|c|c||}\hline\hline
 \hspace{-0in}$\kappa$                                               & $\mathbf{-2.0}$ & $\mathbf{-1.5}$ & $\mathbf{-1}$        & $\mathbf{-0.5}$    \\ \hline\hline
$\alpha_c^{(1,f)}(\kappa)$                    & $  173.4 $  & $   43.77  $  & $ 13.27 $  & $   4.770 $
 \\ \hline \hline
$\alpha_c^{(2,p)}(\kappa)$                                      & $  126.2 $  & $   37.36  $  & $ 12.78 $  & $   4.770 $ \\  \hline
 $\alpha_c^{(2,f)}(\kappa)$                                      & $ 125.4  $ & $  36.57 $ & $   12.32 $ & $  4.701 $   \\ \hline\hline
 $\alpha_c^{(3,f)}(\kappa)$                                      & $ \bl{\mathbf{ 124.8 }} $ & $  \bl{\mathbf{36.40}} $ & $  \bl{\mathbf{ 12.29 }} $ & $ \bl{\mathbf{ 4.698 }}$   \\ \hline\hline
 \end{tabular}
%\end{adjustwidth}
\label{tab:tab3lev3}
\end{table}

The obtained results are also visualized in Figures \ref{fig:fig1} and \ref{fig:fig2}. In Figure \ref{fig:fig1} a small $\kappa$ range is shown resulting in not so large scaled capacities (of order of a few tens). In these regimes the differences between various levels of lifting are more pronounced. However, as the figure clearly shows, the convergence is rather remarkably fast. When capacities get larger the relative differences become even smaller. This is clear from Figure \ref{fig:fig2}, where one can not make much of a difference between say 2-spl RDT on the one side  and 2-sfl and 3-sfl RDT on the other side. In other words, in the large $\alpha_c(\kappa)$ regimes, the lifted curves are visually indistinguishable which reconfirms the fact that 2-spl RDT  results of \cite{StojnicGardSphNeg13} are up to the leading order terms optimal (this was also shown in \cite{ZhouXiMon21}).
\begin{figure}[h]
%\begin{minipage}[b]{.5\linewidth}
\centering
\centerline{\includegraphics[width=1\linewidth]{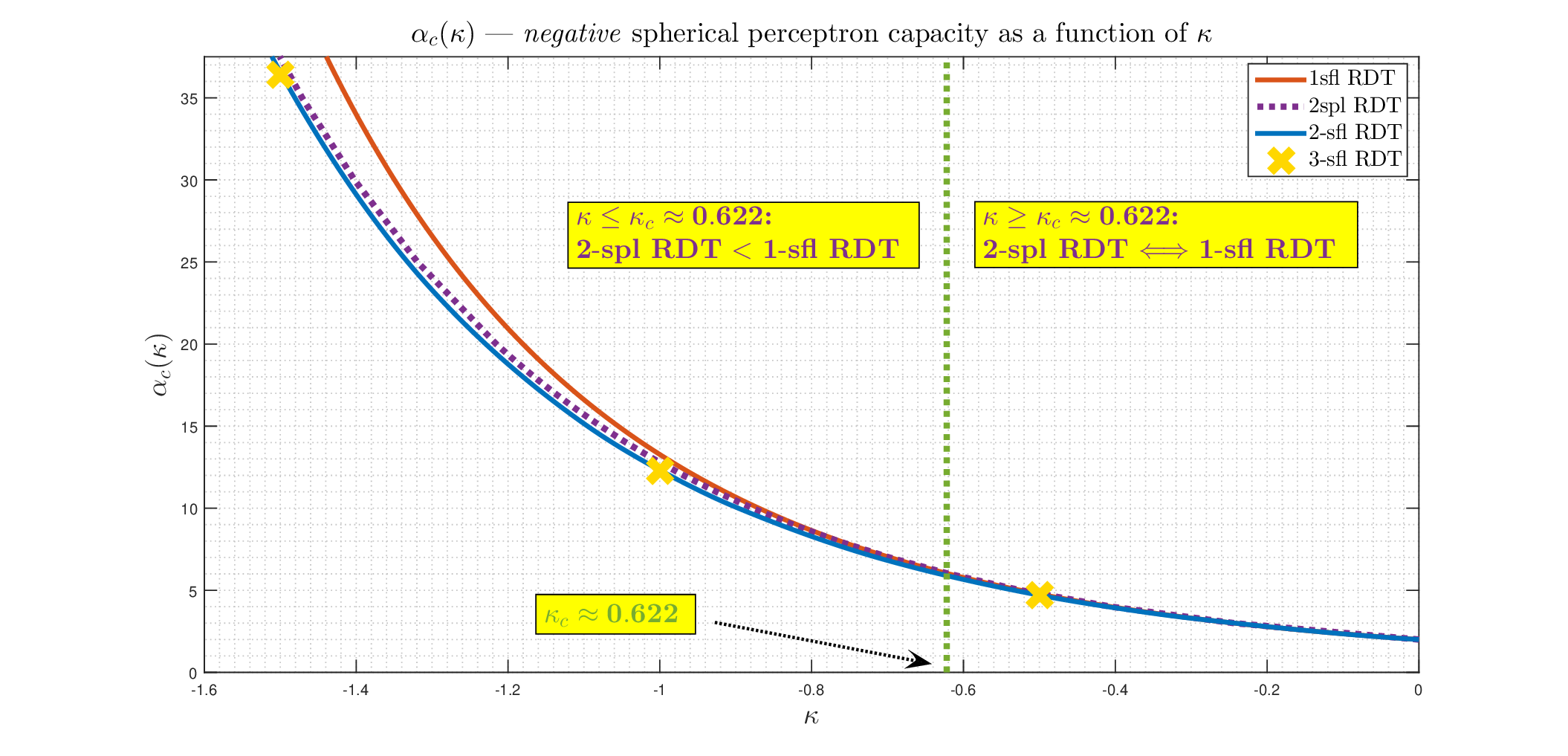}}
%\end{minipage}
%\begin{minipage}[b]{.5\linewidth}
%\centering
%\centerline{\epsfig{figure=finprerral08.eps,width=9cm,height=6.5cm}}
%\end{minipage}
\caption{\emph{Negative} spherical perceptron capacity as a function of $\kappa$}
\label{fig:fig1}
\end{figure}

\begin{figure}[h]
%\begin{minipage}[b]{.5\linewidth}
\centering
\centerline{\includegraphics[width=1\linewidth]{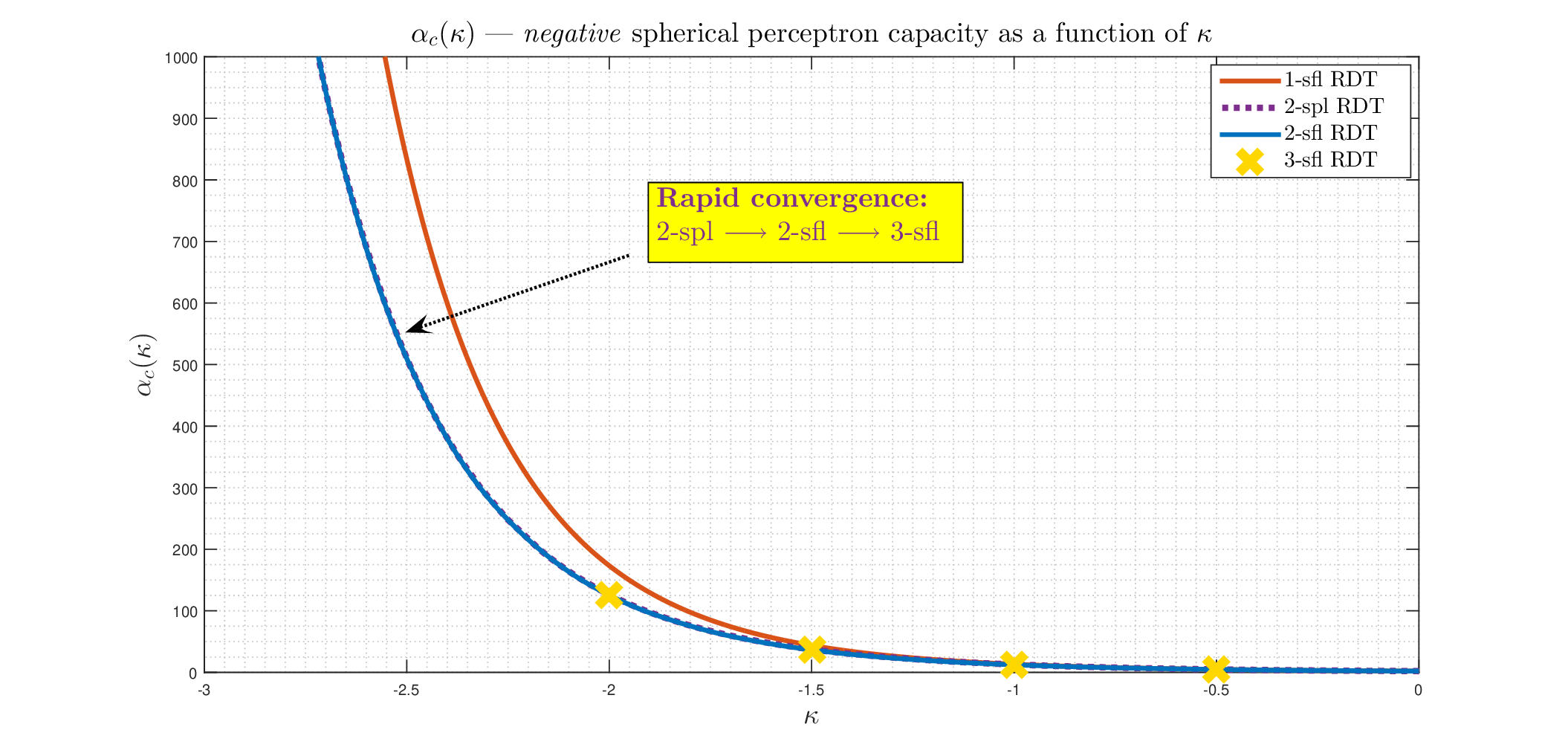}}
%\end{minipage}
%\begin{minipage}[b]{.5\linewidth}
%\centering
%\centerline{\epsfig{figure=finprerral08.eps,width=9cm,height=6.5cm}}
%\end{minipage}
\caption{\emph{Negative} spherical perceptron capacity as a function of $\kappa$; larger $\alpha_c(\kappa)$}
\label{fig:fig2}
\end{figure}

%%%%%%%%%%%%%%%%%%%%%%%%%%%%%%%%%%%%%%%%%%%%%%%%%%%%%%%%%%%%%%%%%%%%%%%%%%%%%%%%%%%%%%%%%%%%%%%%%
%%%%%%%%%%%%%%%%%%%%%%%%%%%%%%%%%%%%%%%%%%%%%%%%%%%%%%%%%%%%%%%%%%%%%%%%%%%%%%%%%%%%%%%%%%%%%%%%%
%%%%%%%%%%%%%%%%%%%%%%%%%%%%%%%%%%%%%%%%%%%%%%%%%%%%%%%%%%%%%%%%%%%%%%%%%%%%%%%%%%%%%%%%%%%%%%%%%
%%%%%%%%%%%%%%%%%%%%%%%%%%%%%%%%%%%%%%%%%%%%%%%%%%%%%%%%%%%%%%%%%%%%%%%%%%%%%%%%%%%%%%%%%%%%%%%%%
%%%%%%%%%%%%%%%%%%%%%%%%%%%%%%%%%%%%%%%%%%%%%%%%%%%%%%%%%%%%%%%%%%%%%%%%%%%%%%%%%%%%%%%%%%%%%%%%%
\subsubsection{Modulo-$\m$ sfl RDT}
\label{sec:posmodm}
%%%%%%%%%%%%%%%%%%%%%%%%%%%%%%%%%%%%%%%%%%%%%%%%%%%%%%%%%%%%%%%%%%%%%%%%%%%%%%%%%%%%%%%%%%%%%%%%%
%%%%%%%%%%%%%%%%%%%%%%%%%%%%%%%%%%%%%%%%%%%%%%%%%%%%%%%%%%%%%%%%%%%%%%%%%%%%%%%%%%%%%%%%%%%%%%%%%
%%%%%%%%%%%%%%%%%%%%%%%%%%%%%%%%%%%%%%%%%%%%%%%%%%%%%%%%%%%%%%%%%%%%%%%%%%%%%%%%%%%%%%%%%%%%%%%%%
%%%%%%%%%%%%%%%%%%%%%%%%%%%%%%%%%%%%%%%%%%%%%%%%%%%%%%%%%%%%%%%%%%%%%%%%%%%%%%%%%%%%%%%%%%%%%%%%%
%%%%%%%%%%%%%%%%%%%%%%%%%%%%%%%%%%%%%%%%%%%%%%%%%%%%%%%%%%%%%%%%%%%%%%%%%%%%%%%%%%%%%%%%%%%%%%%%%

Everything presented above can be repeated relying on the so-called modulo-$m$ sfl RDT frame of \cite{Stojnicsflgscompyx23}. Instead of Theorem \ref{thme:negthmprac1}, one then basically has the following theorem.
 \begin{theorem}
  \label{thme:negthmprac2}
  Assume the setup of Theorem \ref{thme:negthmprac1} and instead of the complete, assume the modulo-$\m$ sfl RDT setup of \cite{Stojnicsflgscompyx23}.
   Let the ``fixed'' parts of $\hat{\p}$, $\hat{\q}$, and $\hat{\c}$ satisfy $\hat{\p}_1\rightarrow 1$, $\hat{\q}_1\rightarrow 1$, $\hat{\c}_1\rightarrow 1$, $\hat{\p}_{r+1}=\hat{\q}_{r+1}=\hat{\c}_{r+1}=0$, and let the ``non-fixed'' parts of $\hat{\p}_k$, and $\hat{\q}_k$ ($k\in\{2,3,\dots,r\}$) be the solutions of the following system of equations
  \begin{eqnarray}\label{eq:negthmprac2eq1}
   \frac{d \bar{\psi}_{rd}(\p,\q,\c,\gamma_{sq},\gamma_{sq}^{(p)},1,1,-1)}{d\p} =  0 \nonumber \\
   \frac{d \bar{\psi}_{rd}(\p,\q,\c,\gamma_{sq},\gamma_{sq}^{(p)},1,1,-1)}{d\q} =  0 \nonumber \\
    \frac{d \bar{\psi}_{rd}(\p,\q,\c,\gamma_{sq},\gamma_{sq}^{(p)},1,1,-1)}{d\gamma_{sq}} =  0\nonumber \\
    \frac{d \bar{\psi}_{rd}(\p,\q,\c,\gamma_{sq},\gamma_{sq}^{(p)},1,1,-1)}{d\gamma_{sq}^{(p)}} =  0.
 \end{eqnarray}
 Consequently, let
\begin{eqnarray}\label{eq:negthmprac2eq2}
c_k(\hat{\p},\hat{\q})  & = & \sqrt{\hat{\q}_{k-1}-\hat{\q}_k} \nonumber \\
b_k(\hat{\p},\hat{\q})  & = & \sqrt{\hat{\p}_{k-1}-\hat{\p}_k}.
 \end{eqnarray}
 Then
 \begin{eqnarray}
-f_{sq}(\infty)
& \leq  &   \max_{\c} \frac{1}{2}    \sum_{k=2}^{r+1}\Bigg(\Bigg.
   \hat{\p}_{k-1}\hat{\q}_{k-1}
   -\hat{\p}_{k}\hat{\q}_{k}
  \Bigg.\Bigg)
\c_k \nonumber \\
&&
  -\hat{\gamma}_{sq}^{(p)}- \varphi(D_1^{(per)}(c_k(\hat{\p},\hat{\q})),\c) + \hat{\gamma}_{sq} - \alpha\varphi(-D_1^{(sph)}(b_k(\hat{\p},\hat{\q})),\c)
  =-f_{sq,\m}(\infty).
  \label{eq:negthmprac2eq3}
\end{eqnarray}
\end{theorem}
\begin{proof}
Follows from the previous discussion, Theorems \ref{thm:thmsflrdt1} and \ref{thme:negthmprac1}, Corollary \ref{cor:cor1}, and the sfl RDT machinery presented in \cite{Stojnicnflgscompyx23,Stojnicsflgscompyx23,Stojnicflrdt23}.
\end{proof}
We conducted the numerical evaluations using the modulo-$\m$ results of the above theorem without finding any scenario where the inequality in (\ref{eq:negthmprac2eq3}) is not tight. In other words, we have found that $f_{sq}^{(r)}(\infty)=f_{sq,\m}^{(r)}(\infty)$. This indicates that the \emph{stationarity} over $\c$ is actually of the \emph{maximization} type.

%%%%%%%%%%%%%%%%%%%%%%%%%%%%%%%%%%%%%%%%%%%%%%%%%%%%%%%%%%%%%%%%%%%%%%%%%%%%%%%%
%%%%%%%%%%%%%%%%%%%%%%%%%%%%%%%%%%%%%%%%%%%%%%%%%%%%%%%%%%%%%%%%%%%%%%%%%%%%%%%%
%%%%%%%%%%%%%%%%%%%%%%%%%%%%%%%%%%%%%%%%%%%%%%%%%%%%%%%%%%%%%%%%%%%%%%%%%%%%%%%%
%%%%%%%%%%%%%%%%%%%%%%%%%%%%%%%%%%%%%%%%%%%%%%%%%%%%%%%%%%%%%%%%%%%%%%%%%%%%%%%%
%%%%%%%%%%%%%%%%%%%%%%%%%%%%%%%%%%%%%%%%%%%%%%%%%%%%%%%%%%%%%%%%%%%%%%%%%%%%%%%%
\section{Conclusion}
\label{sec:conc}
%%%%%%%%%%%%%%%%%%%%%%%%%%%%%%%%%%%%%%%%%%%%%%%%%%%%%%%%%%%%%%%%%%%%%%%%%%%%%%%%
%%%%%%%%%%%%%%%%%%%%%%%%%%%%%%%%%%%%%%%%%%%%%%%%%%%%%%%%%%%%%%%%%%%%%%%%%%%%%%%%
%%%%%%%%%%%%%%%%%%%%%%%%%%%%%%%%%%%%%%%%%%%%%%%%%%%%%%%%%%%%%%%%%%%%%%%%%%%%%%%%
%%%%%%%%%%%%%%%%%%%%%%%%%%%%%%%%%%%%%%%%%%%%%%%%%%%%%%%%%%%%%%%%%%%%%%%%%%%%%%%%
%%%%%%%%%%%%%%%%%%%%%%%%%%%%%%%%%%%%%%%%%%%%%%%%%%%%%%%%%%%%%%%%%%%%%%%%%%%%%%%%

We studied the statistical capacity of the negative spherical  perceptrons (i.e., the classical spherical perceptron with \emph{negative} thresholds $\kappa$). Differently from their positive thresholds counterparts, these problems belong to the class of hard random structures where standard analytical approaches are powerless when it comes to approaching the \emph{exact} capacity characterizations. The random duality (RDT)  based results  \cite{StojnicGardGen13} provided solid generic upper bounds that were substantially improved via the \emph{partially} lifted RDT in \cite{StojnicGardSphNeg13}. A recent breakthroughs in studying bilinearly indexed random processes  \cite{Stojnicsflgscompyx23,Stojnicnflgscompyx23}, enabled \cite{Stojnicflrdt23} to create a \emph{fully lifted} random duality theory (fl RDT) counterpart to the RDT from \cite{StojnicCSetam09,StojnicICASSP10var,StojnicRegRndDlt10,StojnicGardGen13,StojnicICASSP09}.

After recognizing the connection between the statistical perceptrons, general \emph{random feasibility problems} (rfps), and the bilinearly indexed (bli), we utilized the fl RDT and its a particular \emph{stationarized} variant (called sfl RDT) to establish a general framework for studying the negative spherical perceptrons. The practical usability of the entire fl RDT machinery relies on a successful conducting of heavy underlying numerical evaluations. We first presented a large amount of analytical simplifications that resulted in uncovering remarkable closed form interconnections among the key lifting parameters. In addition to providing a direct view into the structure of the parametric relations, they also greatly helped with the numerical work. In particular, we obtained concrete numerical results and uncovered a remarkably rapid convergence of the whole fl RDT mechanism. Over a wide range of thresholds $\kappa$ (allowing \emph{scaled} capacities of a few thousands), we observed that the third (second non-trivial) level of stationarized full lifting suffices to achieve relative improvements no better than $\sim 0.1\%$. To ensure that the lifting progress is systematically presented and that the rapid convergence is clearly visible, we started with the very first level and then incrementally increased the level of lifting. Such a systematic procedure also allowed us to deduce as special cases the earlier results obtained through the \emph{plain} RDT in \cite{StojnicGardGen13} and the \emph{partial} RDT in \cite{StojnicGardSphNeg13}.

The methodology is very generic and  various extensions and generalizations are possible. These include many related to both general \emph{random feasibility problems} (rfps) and  particular random perceptrons. A lengthy list of random structures discussed in \cite{Stojnicnflgscompyx23,Stojnicsflgscompyx23,Stojnicflrdt23,Stojnicflrdt23,Stojnichopflrdt23} is an example of a collection of such problems
that can be handled through the methods presented here. As the technical details are problem specific,  we discuss them in separate papers.

As  \cite{Stojnicflrdt23,Stojnichopflrdt23} emphasized, the sfl RDT considerations do not require the standard Gaussianity  assumption of the random primals. The Lindeberg variant of the central limit theorem (see, e.g.,  \cite{Lindeberg22}) can be utilized to quickly extend the sfl RDT results to a wide range of different statistics.  \cite{Chatterjee06}'s utilization of the Lindenberg approach is, for example, particularly elegant.

%\newpage1
%\setcounter{page}{1}
\begin{singlespace}
\bibliographystyle{plain}
\bibliography{nflgscompyxRefs}
\end{singlespace}

\end{document}